\documentclass[twoside,11pt]{article}

\usepackage{blindtext}

\usepackage{subcaption}
\usepackage{caption}
\usepackage[abbrvbib, preprint]{jmlr2e}
\usepackage{natbib}
\usepackage{microtype}
\usepackage{xcolor}
\usepackage{url}
\usepackage{booktabs}
\usepackage{multirow}
\usepackage{array}
\usepackage{threeparttable}
\usepackage{amsmath,amssymb,amsfonts,mathtools,bm,latexsym}
\usepackage{bbm}
\usepackage{dsfont}
\usepackage{graphicx}
\usepackage[export]{adjustbox}
\usepackage{float}
\usepackage[ruled,vlined,linesnumbered]{algorithm2e}
\usepackage{algorithmic}
\usepackage{hyperref}
\usepackage[all]{hypcap}
\usepackage{cleveref}
\usepackage{makecell}
\usepackage{adjustbox}
\usepackage{tabularx}
\usepackage{array}

\newcolumntype{C}[1]{>{\centering\arraybackslash}p{#1}}

\newtheorem{theorem}{Theorem}

\newtheorem{lemma}{Lemma}

\newtheorem{assumption}{Assumption}

\newtheorem{remark}{Remark}

\newcommand{\E}{\mathbb{E}}

\newcommand{\Prb}{\mathbb{P}}
\newcommand{\Var}{\mathbb{V}}
\newcommand{\Dc}{\mathcal{D}}
\newcommand{\State}{\mathcal{S}}
\newcommand{\Action}{\mathcal{A}}
\newcommand{\Pest}{\hat{P}}
\newcommand{\Vest}{\hat{V}}

\newcommand{\src}{\mathrm{src}}
\newcommand{\tar}{\mathrm{tar}}
\newcommand{\TV}{\mathrm{TV}}
\newcommand{\KL}{\mathrm{KL}}
\newcommand{\Ber}{\operatorname{Ber}}

\newcommand{\pos}[1]{\left[#1\right]_+}
\newcommand{\kl}{\mathrm{kl}}
\newcommand{\wt}{\widetilde}

\newcommand{\clip}{\operatorname{clip}}

\newcommand{\on}{\mathrm{on}}
\newcommand{\hyb}{\mathrm{hyb}}

\newcommand{\bud}{\mathrm{bud}}
\newcommand{\loth}{\mathrm{l.o.t.}}

\newcommand{\1}{\mathbf{1}}

\newcommand{\one}{\mathbbm{1}}
\newcommand{\Reg}{\mathrm{Regret}}
\newcommand{\Sset}{\mathcal S}
\newcommand{\Aset}{\mathcal A}

\newcommand{\instdep}{\makecell[l]{\textbf{Instance-}\\\textbf{dependent}}}
\newcommand{\instind}{\makecell[l]{\textbf{Instance-}\\\textbf{independent}}}

\usepackage{lastpage}
\jmlrheading{23}{2022}{1-\pageref{LastPage}}{1/21; Revised 5/22}{9/22}{21-0000}
{Zheshun Wu, Renjie Zheng, Jinhang Zuo, Zenglin Xu, and Fang Kong}

\ShortHeadings{Hybrid RL under Shifted Transition Dynamics}{Wu, Zheng, Zuo, Xu, and Kong}
\firstpageno{1}

\begin{document}

\title{A Unified Algorithmic Framework for Hybrid Reinforcement Learning in Tabular MDPs with Shifted Transition Dynamics}

\author{\name Zheshun Wu \email wuzhsh23@gmail.com \\
       \addr Southern University of Science and Technology\\
       Harbin Institute of Technology, Shenzhen\\
    Shenzhen, Guangdong, China
       \AND
       \name Renjie Zheng \email zhengrj2024@mail.sustech.edu.cn \\
       \addr Southern University of Science and Technology\\
           Shenzhen, Guangdong, China
       \AND
       \name Jinhang Zuo \email jinhangzuo@gmail.com \\
       \addr City University of Hong Kong\\
         Hong Kong SAR, China
       \AND
       \name Zenglin Xu \email zenglin@gmail.com \\
       \addr Fudan University\\
       Shanghai Academy of AI for Science\\
           Shanghai, China
       \AND
       \name Fang Kong\thanks{Corresponding author: Fang Kong.}\email kongf@sustech.edu.cn \\
       \addr Southern University of Science and Technology\\
        Shenzhen 518055, Guangdong, China
     }

\editor{My Editor}

\maketitle
\begin{abstract}%
This paper investigates a hybrid reinforcement learning setting in tabular Markov Decision Processes (MDPs), where an agent aims to learn an optimal policy by combining online interactions with a target environment and offline data from a source environment.
  A central challenge is that offline data may be collected from outdated environments with shifted transition dynamics, making naive integration of historical data ineffective. To address this, we propose a unified algorithmic framework featuring two algorithms: MIN-UCB-VI for regret minimization and MAX-LCB-VI for best policy identification. Both algorithms leverage fine-grained bias information to more effectively exploit offline data under general transition shifts. We provide theoretical guarantees for our framework, including both instance-dependent and independent upper bounds on regret and sub-optimality gap. Furthermore, we establish matching lower bounds to demonstrate the optimality of our approach and validate our theoretical findings through extensive experiments.
\end{abstract}


\section{Introduction}

Reinforcement learning (RL) is a powerful paradigm for formulating a learning agent interacting with an uncertain environment and has achieved great success in many intelligence applications like RLHF~\citep{DBLP:conf/nips/RafailovSMMEF23,bai2022training}, robot learning~\citep{kober2013reinforcement}, game playing~\citep{silver2016mastering,silver2017mastering}, and more. Many applications often require training the RL agent with millions of samples to learn an effective policy, which is costly and time-consuming in practice~\citep{li2023understandingcomplexitygainssingletask}. This has motivated a line of recent efforts to study the sample efficiency of RL algorithms~\citep{azar2017minimax,NEURIPS2021_e61eaa38,wagenmaker2023leveraging}. One promising direction is to improve sample efficiency through hybrid learning~\citep{DBLP:conf/icml/CheungL24}. Hybrid RL frameworks enable an agent to learn an effective policy by leveraging both offline and online data~\citep{qu2025hybrid,huang2025augmenting}. Specifically, the agent conducts online exploration to refine its policy while simultaneously exploiting offline experience to enhance sample efficiency.
 

A large body of empirical work has investigated hybrid RL and demonstrated strong empirical performance~\citep{serrano2023similarity,ammar2015autonomous,you2022cross,kalashnikov2018scalable}. In parallel, several studies have sought to provide theoretical explanations for the advantages of hybrid learning. However, existing hybrid RL algorithms with theoretical guarantees typically rely on the assumption that offline and online data are drawn from the same environment, i.e., they share identical transition dynamics~\citep{huang2025augmenting,NEURIPS2021_e61eaa38,wagenmaker2023leveraging,li2023reward,song2022hybrid}.
In practice, this assumption is often violated. Offline data may originate from environments with shifted transition dynamics, for example due to data collection from imperfect simulators, noisy datasets, or outdated environments. As a result, it is difficult to guarantee that offline experience is fully consistent with the online environment. While some empirical studies~\citep{ball2023efficient} indicate that leveraging data with shifted dynamics can still improve learning efficiency, existing theoretical results under such transition shifts rely on strong structural assumptions and remain limited in scope.
 

Recently, \citet{qu2025hybrid} propose algorithms with theoretical guarantees for hybrid RL under shifted transition dynamics. Their methodology relies on two restrictive assumptions: (i) specific region reachability conditions must hold, and (ii) the transition dynamics of the source and target environments are separable.
The region reachability condition requires the agent to reliably access all regions exhibiting transition shifts through online exploration, which may be infeasible in large-scale or safety-constrained environments. The separability assumption presumes that the transition dynamics of the offline and online environments are either identical or differ by a sufficiently large margin. In realistic settings, however, discrepancies between offline and online dynamics can be arbitrary and continuous, exhibiting neither perfect alignment nor clearly separable differences. Such nuanced and heterogeneous shifts violate the separability condition, thereby limiting the applicability of the proposed guarantees.

Based on the above observations, we argue that despite recent theoretical progress on hybrid RL, a fundamental challenge remains:
\begin{center}
\textit{Can we develop a unified algorithmic framework that provably exploits offline data to improve hybrid RL under general transition shifts?}
\end{center}

This paper tries to answer this question by studying the hybrid RL problem in the tabular MDP setting with the main contributions summarized as follows.

\begin{itemize}
    \item  We develop a unified algorithmic framework for hybrid tabular MDPs with shifted transition dynamics. This framework includes both a MIN-UCB-VI algorithm for regret minimization and a MAX-LCB-VI algorithm for best policy identification, both of which effectively leverage fine-grained bias information between the source and target environments to  enhance online learning efficiency. We also integrate a model selection scheme into our proposed MIN-UCB-VI algorithm to address settings in which the agent does not have prior knowledge of the environment bias. 
    \item  We provide rigorous theoretical guarantees for our algorithms by deriving both instance-dependent and independent upper bounds on regret and sub-optimality gap. Furthermore, we establish matching lower bounds to demonstrate the optimality of our approach.
\item We show that cumulative bias over steps in regret bounds adversely impacts online learning performance due to the inherent step-wise dependency structure  in MDPs. Besides, we introduce a novel concentrability metric that jointly characterizes environment bias and policy mismatch, providing a more appropriate measure of the learned policy’s quality under shifted transition dynamics.
\end{itemize}

\begin{table}[t] \small \centering \setlength{\tabcolsep}{2pt} 
\begin{adjustbox}{max width=\textwidth}
\begin{tabular}{lcc} \toprule & Upper bound & Lower bound \\ \midrule \multicolumn{3}{c}{\textbf{Regret bounds}} \\ \midrule \instdep & $\Tilde{\mathcal{O}}\left(\sum_{s,a} \max\left\{ \frac{H^2\Var^*(s,a)}{\Delta(s,a)}-\mathrm{sav}(s,a),\Delta(s,a) \right\}\right)$ & $\Omega\left(\sum_{s,a}\frac{H^2L}{\Delta(s,a)}-N^\src_{s,a}\Delta(s,a)\Big(1-\frac{\Tilde{\nu}(s,a)}{\Delta(s,a)}\Big)_+^2\right)$ \\ \addlinespace[0.55em] \instind & $\Tilde{\mathcal{O}}\left(H\min\left\{\sqrt{SAK}, K\Tilde{\nu}_{\max}+K\sqrt{\frac{H}{\tau_*}}\right\}\right)$ & $\Omega\left(H\min\left\{\sqrt{SAK},KH\nu_{\max}+K\sqrt{\frac{H}{\tau_*}} \right\} \right)$ \\[0.8em] \midrule \multicolumn{3}{c}{\textbf{Sub-optimality gap bounds}} \\ \midrule \instdep & $K
\geq\Omega\left(
\frac{H^3SAL^2}{\epsilon\Delta_{\min}}
-
\frac{AN^{\rm src}}{C_\nu^*}
\left(
1-\frac{ H^2\nu_{\max}}{\Delta_{\min}}
\right)_+
\right),$ & $ K\leq \mathcal{O}\left( \frac{SHA}{\epsilon\Delta_{\min}}-\frac{A N^{\src}}{C_\nu^*}
\left(
1-\frac{\nu_{\max}}{\Delta_{\min}}
\right)_+\right)$ \\ \addlinespace[0.55em] \instind & $\Tilde{\mathcal{O}}\left(H\min\left\{ \sqrt{\frac{SA}{K}}, H\nu_{\max}+\sqrt{\frac{S}{N^{\src}/C^{*}_\nu+K/A} }\right\}\right)$ & $\Omega\left(H\min\left\{
\sqrt{\frac{SM}{K}},\,
H\nu_{\max}
+\sqrt{\frac{S}{N^{\src}/C_\nu^*+K/M}}
\right\}\right)$ \\ \bottomrule \end{tabular}\end{adjustbox} \vspace{0.1in} \caption{Regret bounds and sub-optimality gap bounds. Let $K$ be the number of online episodes and $N^\src$ be the number of offline trajectories. $\nu(s,a)$ bounds the TV distance between the target and source transitions. Let $\Tilde{\nu}(s,a):=\max\{H^2\nu(s,a),H^3\nu^2(s,a)\}$ and $\Tilde{\nu}_{\max}=\max_{s,a}\Tilde{\nu}(s,a)$ to reflect the effect of cumulative bias. $\mathrm{sav}_{s,a}:=N^\src_{s,a}\Delta(s,a)\max\left\{1-\frac{H\Tilde{\nu}(s,a)}{\Delta(s,a)},0\right\}^2$ denotes the saving term.  $C^*_\nu$ denotes the concentrability tailored to hybrid RL with shifted transitions. We define $M:=\min\{A,C^*_{\nu}\}$, $\Delta_{\min}=\min_{s,a:\Delta(s,a)>0}\Delta(s,a)$ and $L=\log(SAKH/\delta)$.} \label{tab:combined-bounds} 
\end{table}

\section{Related Work}
\label{sec:related_work}

\subsection{Online and Offline RL} There is a large body of work on provably sample-efficient online RL, particularly in the tabular setting with finite states and actions, where exploration strategies yield polynomial sample complexity guarantees~\citep{azar2017minimax,zhang2021reinforcement,chen2025sharp,simchowitz2019non}. When state and action spaces are large or continuous, function approximation with structural assumptions is typically used to retain sample efficiency~\citep{wang2020reinforcement,pmlr-v125-jin20a}. In offline (batch) RL, where learning relies solely on a fixed dataset, sample efficiency is more challenging due to distributional shift between the behavior policy and the policies being evaluated~\citep{levine2020offline}. Classical results require strong coverage or concentrability assumptions on the dataset to control this shift~\citep{duan2020minimax}. To relax these requirements, recent theoretical work widely adopts the pessimism principle during policy evaluation and optimization, which yields finite-sample guarantees under weaker partial coverage assumptions  without necessitating full exploration of the state–action space~\citep{rashidinejad2021bridging,shi2022pessimistic,jin2021pessimism}.

\subsection{Hybrid RL} The study of leveraging offline datasets to improve the efficiency of online exploration has  been extensively studied in RL, often under the umbrella of hybrid RL.  Several recent works have studied hybrid RL in finite-horizon tabular MDPs using both offline and online data~\citep{NEURIPS2021_e61eaa38}. Beyond the tabular setting,~\citet{wagenmaker2023leveraging} consider linear MDPs and develop novel hybrid RL algorithms. There are also several  approaches that address this problem within the general function approximation framework. For example, \citet{DBLP:conf/iclr/0001ZSBK023} introduce hybrid Q-learning, and \citet{DBLP:conf/rlc/TanX24} develop an algorithm based on Global Optimism with Local Fitting (GOLF) that accommodates general Q-function approximation.  However, a common assumption in these works is that the offline and online Markov Decision Processes (MDPs) are identical, i.e., they share the same transition dynamics and reward functions. Such an assumption may be overly restrictive in many practical applications where the environment evolves or shifts over time. To address discrepancies between offline and online environments, a recent study by \citet{qu2025hybrid} formulates this as a hybrid transfer RL problem and proposes the HySRL method to mitigate distribution shift.  Nonetheless, their approach relies on assumptions such as region reachability, and separability which includes access to a known lower bound on the distribution shift between the offline and online environments. Such assumptions may be unrealistic in practical applications. 

\subsection{Hybrid Bandits} The hybrid bandits problem, often known as the warm-start bandits setting, is intrinsically linked to hybrid RL by treating bandits as a simplified, one-step RL task focused on leveraging offline data to accelerate online exploration. Within the classic multi-armed bandit (MAB) regime, research has evolved from integrating historical data to refine confidence bounds—as seen in the HUCB algorithm \citep{shivaswamy2012multi} and warm-started Thompson Sampling \citep{DBLP:conf/icml/HaoJLRW23}—to addressing the restrictive homogeneity assumption that offline and online reward distributions are identical. Recognizing that this assumption rarely holds in practice, \citet{DBLP:conf/icml/CheungL24} introduced the MIN-UCB algorithm to maintain near-optimal regret even under distribution shifts. While this progress extends to the hybrid linear bandit setting through adaptations of LinUCB and Thompson Sampling \citep{DBLP:journals/kais/OetomoPBR23}, a significant portion of current literature still relies on the assumption of unbiased reward distributions, highlighting a continuing gap between theoretical models and real-world application.

\section{Preliminaries}
\subsection{Notations} In this paper, we  use $[N]$ to denote the set $\{1,2,\ldots,N\}$ for $N\in\mathbb{Z}_{+}$. For an event $\mathcal{E}$, we use $\mathbbm{1}[\mathcal{E}]$ to denote the indicator function.  We use standard $\mathcal{O}(\cdot)$ and $\Omega(\cdot)$ notation, and use the tilde notation $\Tilde{\mathcal{O}}(\cdot)$ to hide additional log factors. We denote the cardinality of a set $\mathcal{X}$ by $|\mathcal{X}|$. For a function $f$ defined on $\mathcal{S}$, we define its expectation under the probability measure $P$ as $P\cdot f:=\mathbb{E}_{s\sim P}[f(s)]$ and its variance as $\mathbb{V}_{s\sim p}(f):=P\cdot (f-P\cdot f)^2$.

\subsection{Episodic Tabular Markov Decision Processes} In this paper, we focus on episodic Markov decision processes (MDPs) with time-homogeneous (stage-independent) transitions, described by a tuple $M=(\mathcal{S},\mathcal{A},H,P,r)$. $\mathcal{S}$ is the finite state space with cardinality $S$. $\mathcal{A}$ is the finite action space with cardinality $A$. $H$ is the horizon length. $P: \mathcal{S}\times \mathcal{A}\rightarrow \Delta(\mathcal{S})$ denotes the  probability transition kernel of the MDP, and $r: \mathcal{S}\times \mathcal{A}\rightarrow [0,1]$ is the deterministic known reward function. Besides, we assume that each episode of the MDP begins from an initial state sampled from a distribution $\rho\in\Delta(\mathcal{S})$ and $s_1\sim \rho$.

A Markov policy $\pi=\{\pi(\cdot|s,h)\}_{h\in[H],s\in\mathcal{S}}$ includes distributions $\pi(\cdot|s,h)\in \Delta_{\mathcal{A}}$. In specific,  $\pi(\cdot|s,h)$ is the probability of picking action $a$ in state $s$ at time step $h$. We denote the set of all Markov policies by $\Pi$. For a given transition $P$, the policy $\pi$ induces a trajectory $\{s_1,a_1,r_1,s_2,a_2,r_2,\ldots,s_H,a_H,r_H\}$, where $a_h\sim\pi(\cdot|s_h,h),r_h=r(s_h,a_h),s_{h+1}\sim P(\cdot|s_h,a_h)$. For the transition $P$ and policy $\pi$, let $V_h^{P,\pi}:\mathcal{S}\rightarrow \mathbb{R}$ and $Q^{P,\pi}_h:\mathcal{S}\times \mathcal{A}\rightarrow \mathbb{R}$ denote its value functions and Q-functions at each step $h\in[H]$, that is,
\begin{align}
 V_h^{P,\pi}(s)&:=\mathbb{E}\left[ \sum_{h'=h}^H r(s_{h'},a_{h'})\bigg\vert s_h=s\right],\nonumber\\
    Q_h^{P,\pi}(s,a)&:=\mathbb{E}\left[ \sum_{h'=h}^H r(s_{h'},a_{h'})\bigg\vert s_h=s,a_h=a\right],\nonumber
\end{align}
where the expectation  is over the initial state distribution $\rho$, the transition kernel $P$ and the policy $\pi$. We further define $V_h^{P,\pi}=V^{P,\pi}_h(\rho):=\mathbb{E}_{s\sim \rho}[V_h^{P,\pi}(s)]$.

Then Bellman equation establishes the following identities for $P$ and $\pi$ : $Q_h^{P,\pi}(s,a)=r(s,a)+P_{s,a}\cdot V^{P,\pi}_{h+1}$ and $V_h^{P,\pi}(s)=\sum_{a\in\mathcal{A}}\pi_h(a|s)Q^{P,\pi}_h(s,a)$. We use $\pi^*:=\arg\max_{\pi} V_1^{P,\pi}$ to denote any optimal policy, and  $V_{h}^{*}:=V_h^{P,\pi^*}$ and $Q_{h}^{*}:=Q_h^{P,\pi^*}$ to  denote the value function and Q function of $\pi^*$ at $h\in[H]$. We emphasize that unless otherwise stated, all value functions are evaluated in the target MDP.

 Additionally, we introduce the  occupancy distributions associated with the policy $\pi$ and transition $P$   at step $h$: $ d^{\pi, P}_h(s):=P(s_h=s|P,\pi)$, $ d^{\pi, P}_h(s,a):=P(s_h=s,a_h=a|P,\pi)$, which are conditioned on the event that all actions are selected by $\pi$ and transitions occur based on $P$.


\subsection{Hybrid RL in Tabular MDPs}

For  Hybrid RL in Tabular MDPs, the  agent interacts with the target MDP $\mathcal{M}_{\mathrm{tar}}=(\mathcal{S},\mathcal{A},H,P^{\mathrm{tar}},r)$ for $K$ episodes. In an episode, at each step $h\in[H]$, the agent observes a state $s_h\in\mathcal{S}$, takes an action $a_h\in \mathcal{A}$, receives a reward $r(s_h,a_h)$ and transitions to a new state $s_{h+1}$ based on the underlying transition probability $P^{\mathrm{tar}}(\cdot|s_h,a_h)$.

In addition, the agent is granted access to an offline/batch dataset $\mathcal{D}_{\mathrm{src}}$, which comprises a collection of $N^{\mathrm{src}}$ i.i.d. sample trajectories pre-collected from a source MDP $\mathcal{M}_{\mathrm{src}}=(\mathcal{S},\mathcal{A},H,P^{\mathrm{src}},r)$ using a behavior policy $\pi^b$. In specific, the $n$-th sample trajectory includes a sequence $\{s^n_1,a^n_1,r^n_1,s^n_2,a^n_2,r^n_2,\ldots,s^n_H,a^n_H,r^n_H\}$, which is generated by the MDP $\mathcal{M}_{\mathrm{src}}$ and the behavior policy $\pi^b$ in the below manner: $a_h^n \sim \pi^b(\cdot|s_h^n)$ and $s^n_{h+1}\sim P^{\mathrm{src}}(\cdot|s_h^n,a_h^n)$. 

We highlight that the source and target MDPs share the same state space, action space, horizon, and reward function, while their transition kernels may differ.

Two classical reinforcement learning objectives are considered in this paper, as defined below.

\subsection{Regret Minimization}
For this learning objective, the agent
 aims to minimize the cumulative regret during the online interactions for $K$ episodes, and it chooses a policy $\pi^k$ at the $k$-th episode. The cumulative regret is defined as
 \begin{align}
    \mathrm{Regret}(K):=\sum_{k=1}^K \left(V_1^{P_{\mathrm{tar}},*}-V_1^{P_{\mathrm{tar}},\pi^k}\right).\nonumber
\end{align}

\subsection{Best Policy Identification} For this learning objective, the goal of the agent is to learn an $\epsilon$-optimal policy $\hat{\pi}$ for $\mathcal{M}_{\mathrm{tar}}$ from both $\mathcal{D}_{\mathrm{src}}$ and $\mathcal{M}_{\mathrm{tar}}$. In specific, the agent learns to find a policy $\hat{\pi}$ for $\mathcal{M}_{\mathrm{tar}}$ satisfying that:
\begin{align}
    V_1^{P_{\mathrm{tar}},*}-  V_1^{P_{\mathrm{tar}},\hat{\pi}}\leq \epsilon.\nonumber
\end{align}

\subsection{Auxiliary Input: Valid Bias Bound} In this paper,  we presume that the learner is endowed with prior knowledge of   the bias level between $P^{\mathrm{tar}}$ and $P^{\mathrm{src}}$. In particular, we say $\nu=\{\nu(s,a)\}$ is a valid \emph{bias bound}, if $ \forall (s,a) \in \mathcal{S}\times \mathcal{A}$, the $\ell_1$ distance satisfies
\begin{align} 
    \left\Vert P^{\mathrm{tar}}(\cdot|s,a)- P^{\mathrm{src}}(\cdot|s,a)\right\Vert_{1}     \leq \nu(s,a).\nonumber
\end{align}

The knowledge on an upper bound like $\nu(s,a)$ defined above is in line with  the assumptions considered in prior studies on learning under distribution shift on multi-armed bandit~\citep{DBLP:conf/icml/CheungL24,pmlr-v286-yang25b}. In settings where the offline and online environments are similar, we may set $\nu(s,a)$ to be small. As shown in~\citet{DBLP:conf/icml/CheungL24}, in the presence of biased offline
data, no hybrid algorithm in MAB can be guaranteed to outperform a pure online method unless some
prior knowledge about the bias is available. Given that MAB is a special case of episodic tabular MDP (when horizon numbers $H=1$), we claim that such prior bias knowledge is necessary for our setting.

Below we introduce a novel concentrability,  defined by extending the common single-policy concentrability used in offline RL~\citep{NEURIPS2021_e61eaa38,li2024settling}, and it further captures the effect of transition shift.
\begin{assumption}[Single-policy concentrability under bias MDPs] For MDPs $\mathcal{M}_{\mathrm{tar}}$ and $\mathcal{M}_{\mathrm{src}}$, the reference policy $\pi^b$ satisfies that
\begin{align}
    \max_{s,a,h} \frac{d_h^{\pi^{*},P_{\mathrm{tar}}}(s,a)}{d_h^{\pi^b,P_{\mathrm{src}}}(s,a)} \leq C^*_{\nu}.\nonumber
\end{align}

\begin{remark}
    Notice that $C^*_{\nu}$ measures the level of both the  transition shift $\nu$, and policy mismatch between $\pi^*$ and $\pi^b$. In words, this metric reflects: 1) the environment heterogeneity between $\mathcal{M}_{\mathrm{tar}}$ interacting with the agent in the online stage, and $\mathcal{M}_{\mathrm{src}}$  which generated the batch dataset; 2) the dissimilarity between the optimal policy $\pi^{*}$ for $\mathcal{M}_{\mathrm{tar}}$  and the behavior policy $\pi^b$ used for collected batch dataset. Hence, $C^*_{\nu}=1$   holds when $P^{\mathrm{tar}}=P^{\mathrm{src}}$ and $\pi^*=\pi^b$.
\end{remark}

\end{assumption}

\section{Algorithm and Theoretical Results}

In this section, we first introduce the procedure of our proposed unified algorithmic framework for hybrid tabular MDPs under shifted transition. We then delve into two specific learning tasks and introduce our devised novel value iteration methods tailored for them. We summarize our theoretical results in Table~\ref{tab:combined-bounds}.

\subsection{A Unified   Hybrid Tabular MDP Framework}

In this subsection, we introduce our proposed unified hybrid tabular MDP framework. The procedure of this framework is illustrated in Algorithm~\ref{alg:unified}.

\begin{algorithm}[htbp]
    \caption{Hybrid Tabular MDP Framework}
    \begin{algorithmic}[1]
        \STATE {\bf Input:} Offline data set $\Dc_{\mathrm{src}}$, number of online episodes $K$, valid bias bound $\nu$.
        \FOR{$k=1,\ldots, K$}
        \STATE $\pi_k \gets \mathtt{MIN-UCB-VI}(\Dc_{\mathrm{src}}\cup\Dc_{k-1}, \nu)$.
            \STATE Collect trajectory $\tau_k$. 
            \STATE Update $\Dc_k=\Dc_{k-1}\cup\{\tau_k\}$.
        \ENDFOR
        \IF{Best Policy Identification:}
        \STATE {\bf Output:} $\hat{\pi}\gets\mathtt{MAX-LCB-VI}(\Dc_{\mathrm{src}}\cup\Dc_{K}, \nu)$. 
        \ENDIF
    \end{algorithmic}
        \label{alg:unified}
\end{algorithm}

In particular,  at each online episode $k\in[K]$, the learner maintains an online data set $\mathcal{D}_{k-1}$, which stores all trajectories collected during the online learning so far. Compared with pure online RL only utilizing $\mathcal{D}_{k-1}$ to learn the online policy, we instead conduct hybrid RL based on the augmented data set $\mathcal{D}_{k-1}\cup \mathcal{D}_{\mathrm{src}}$.

\subsection{MIN-UCB-VI for Regret Minimization}

In this subsection, we introduce our proposed MIN-UCB-VI method in Algorithm~\ref{alg:min-ucbvi} which is devised for tackling regret minimization. At episode $k\in[K]$, MIN-UCB-VI follows the similar idea of~\citet{DBLP:conf/icml/CheungL24} to maintain two optimistic estimates $\{Q^k_h(s,a),Q^{k,\hyb}_h(s,a)\}$ of the Q-function $Q^*_{h}(s,a)$  and chooses the smaller one $\min\{Q^k_h(s,a),Q^{k,\hyb}_h(s,a)\}$ as the UCB of $Q^{*}_h(s,a)$.  

Below we provide the detailed introduction for our MIN-UCB-VI algorithm. We first denote the number of times that the agent plays action $a$ in state $s$  before episode $k$ during online exploration as $N^k_{s,a}$. $N^k_{s,a,s'}$ is the number of times that agent does so and observes the next state $s'$. For the offline data,  we use $N^\src_{s,a}$ and $N^\src_{s,a,s'}$ to denote the corresponding counts of state–action pairs and transitions observed in the offline data, respectively.

Before running the online learning algorithm, the agent first uses the offline data to compute the empirical transition probability $\hat{P}^{\src}(s'|s,a)=\frac{N^\src_{s,a,s'}}{\max\{N^\src_{s,a,}1\}}$. For each episode $k\in[K]$ during online exploration, the agent first computes the empirical transition  probabilities $\hat{P}^k(s'|s,a)=\frac{N^k_{s,a,s'}}{\max\{N^k_{s,a,},1\}},$ using purely online data, and $\hat{P}^{k,\hyb}(s'|s,a)=\frac{N^\src_{s,a} \hat{P}^\src(s'|s,a)+N^k_{s,a}\hat{P}^k(s'|s,a)}{N^k_{s,a}+N^\src_{s,a}}$ using both offline and online data. 

Then the agent conducts the optimistic value iteration beginning from stage $H$ backward, where  two optimistic estimates of $Q^*_h$ are constructed as follows: $Q^k_h(s,a)=\min\{H, r(s,a)+\hat{P}_{s,a}^k\cdot V^k_{h+1}+b^k_h(s,a)\}$ and $Q^{k,\hyb}_h(s,a)=\min\{H, r(s,a)+\hat{P}_{s,a}^{k,\hyb}\cdot V^k_{h+1}+b^{k,\hyb}_h(s,a)\}$. $b^k_h(s,a)$ and $b^{k,\hyb}_h(s,a)$ are bonus terms used to encourage exploration under environment uncertainty. 

In this paper, we use the Monotonic Value Propagation  (MVP) algorithm proposed in~\citet{zhang2021reinforcement} to construct these two bonus terms. In specific, the purely online bonus constructed in a Bernstein-style is computed as
\begin{align}
b_h^k(s,a)&=c_1\sqrt{\frac{\mathbb{V}_{s'\sim\hat{P}^k_{s,a}}(V^k_{h+1}(s'))L}{N^k_{s,a}}}+\frac{c_2HL}{N^k_{s,a}},\nonumber
\end{align}
where $L:=\log(SAKH/\delta)$ and $\delta \in (0,1]$. $c_1$ and $c_2$ are  constants selected carefully to guarantee the optimism~\citep{zhang2021reinforcement}.

Similarly, the bonus corresponding to both offline and online data is computed by
\begin{align}
     b_h^{k,\hyb}(s,a)&=c_1\sqrt{\frac{\mathbb{V}_{s'\sim\hat{P}^{k,\hyb}_{s,a}}(V^k_{h+1}(s'))L}{N^{k}_{s,a}+N^\src_{s,a}}}\nonumber +\frac{c_2HL}{N^k_{s,a}+N^\src_{s,a}}+\frac{H\nu(s,a)N^\src_{s,a}}{N^k_{s,a}+N^\src_{s,a}}.\nonumber
\end{align}

The value function estimate $V^k_h$ is computed as usual value iteration based on the smaller optimistic estimates, i.e., $V^k_h(s)=\max_{a\in\mathcal{A}}\min\{Q^{k,\on}_h(s,a),Q^{k,\hyb}_h(s,a)\}$. After that, the agent can interact with the environment by the following way: at each stage $h$ and episode $k$, the agent observes state $s_h^k$ and chooses the action based on the smaller optimistic estimate, i.e., $a_h^k \in \arg\max_{a\in\mathcal{A}} \min\{Q^{k,\on}(s,a),Q^{k,\hyb}(s,a)\}$.

Below we provide both instance-independent and instance-dependent regret upper bounds for our proposed MIN-UCB-VI algorithm.

\RestyleAlgo{ruled}
\begin{algorithm}[htbp]{
\caption{MIN-UCB-VI.}\label{alg:min-ucbvi}
\textbf{Input}: Valid bias bound $\nu$, Offline data set $\Dc_{\mathrm{src}}$. \\ 
\textbf{Initialize}: $\hat{P}^{\mathrm{src}}(s'|s,a)= \frac{N^{\mathrm{src}}(s,a,s')}{\max\{1,N^{\mathrm{src}}(s,a)\}}$ 

    \vspace{.1cm}

    \textcolor{black!40!white}{\texttt{// Update the two optimistic estimates for episode} $k$}
    
 Estimate $\Pest^k (s' | s, a) = \frac{N^k(s, a, s')}{\max\{1,N^k(s, a)\}}$ \label{alg:ucbvi:pest}

 Define $\Pest^{k,\hyb} (s' | s, a):=\frac{\Pest^kN^k_{s,a} +\Pest^{\src}N^{\src}_{s,a}}{N^k_{s,a}+N^{\src}_{s,a}}$
    
    
    \For{$h = \{ H, H-1, \ldots, 1 \}$}{
    
        \For{$s \in \mathcal{S},a \in \mathcal{A}$}{
                $Q^{k,\on}_h (s,a) =H \wedge(r(s,a) + \Pest^k_{s,a}\cdot V^k_{h+1}+ b^k_{h} (s, a)) $ \label{alg:ucbvi:q_k}

                 $Q^{k,\hyb}_h (s,a) = H \wedge(r(s,a) + \Pest^{k,\hyb}_{s,a}\cdot V^k_{h+1}+ b^{k,\hyb}_{h} (s, a))$ \label{alg:ucbvi:q_combine}

                 $Q^{k}_h(s,a)=\min\{Q^{k,\on}_h(s, a),Q^{k,\hyb}_h(s, a)\}$

            $V^k_h(s) = \max_{a \in \mathcal{A}} Q^{k}_h(s,a)$ \label{alg:ucbvi:vvals}
                        
        }
    }

    \textcolor{black!40!white}{\texttt{// Interact with the environment for episode} $k$}
    
    
    \For{$h \in [H]$}{

takes action by:
\\$a^k_{h} \in \arg\max_{a \in \mathcal{A}} \min\{Q^{k,\on}_h(s_h^k, a),Q^{k,\hyb}_h(s_h^k, a)\}$ \label{alg:ucbvi:play}

        Environment returns the next state $s^k_{h+1}$ \label{alg:ucbvi:observe}
        
        Update $\Dc_{k} \gets \Dc_{k-1}\cup  \{s^k_h,a^k_h,s^k_{h+1}\}$  \label{alg:ucbvi:updatefirst}




    }
}
\end{algorithm}

\subsubsection{Instance-independent Regret Upper Bounds}

\begin{theorem}\label{thm:regret_independent}
   Let $\tilde{\nu}(s,a):=40\max \{H^2\nu(s,a),H^3\nu^2(s,a)\}$. Let $\delta \in (0,1) $, the regret of \emph{MIN-UCB-VI} is bounded w.p. at least $1-\delta$, by 
    \begin{align}
\mathrm{Regret}(K)\nonumber&\leq \Tilde{\mathcal{O}}\left( H\min \left\{\sqrt{SAK},
K\Tilde{\nu}_{\max}+K\sqrt{\frac{H}{\tau_*}} \right\}\right),\label{eq:regret_independent}
    \end{align}
    where $\Tilde{\nu}_{\max}=\max_{s,a}\tilde{\nu}(s,a)$ and $(\tau_*,n^*)$ is an optimum solution of the following linear program:
     \begin{equation}
        \begin{aligned}
             \text{(LP): }\max_{\tau,n} \quad & \tau \\
            \text{s.t.} \quad & \tau \leq N^{\src}_{s,a} + n_{s,a} \quad \forall (s,a )\in \State \times \Action,\\
            & \sum_{s,a}n_{s,a} =KH,\\
            & \tau\geq 0, n_{s,a}\geq 0\quad\;\; \forall (s,a )\in \State \times \Action.
        \end{aligned}
        \label{eq:indpt-LP}
    \end{equation}

    \begin{remark}
        The proof for Theorem~\ref{thm:regret_independent} is in Appendix~\ref{sec:regret_independent_proof}. In the minimum in this regret bound, the first term is the same as the regret of  vanilla UCB-VI. The second term is the regret caused by using the fused data. The minimum corresponds to the fact that MIN-UCB-VI adapts to the better of the two UCBs. For $\tau_*$, we know that $\tau_*\geq \frac{KH}{SA}$ for any $N^\src_{s,a}$. Note that we have $\sqrt{SAKH}\geq K\sqrt{\frac{H}{\tau_*}}$, meaning that $\mathrm{Regret}(K)=o(H\sqrt{SAKHL})$ when $\Tilde{\nu}_{\max}=o(\sqrt{SA/K})$. Besides, in the case when $N^\src_{s,a}=m$ for all $s,a$, we know that $\tau_*=\frac{KH}{SA}+m$ is the optimum of (LP) in~\eqref{eq:indpt-LP}, indicating that $K\sqrt{H/\tau_*}\leq \sqrt{SAKH}$. $\Tilde{\nu}_{\max}$ captures the cumulative effect of bias due to the inherent stage-dependency of MDP.
    \end{remark}

\end{theorem}

\subsubsection{Instance-dependent Regret Upper Bounds}

\begin{theorem} \label{thm:regret_dependent} We define $\Delta(s,a):=\min_{h: \Delta_h(s,a)>0}\Delta_h(s,a)$ and $\Var^*(s,a):=\max_{h\in[H]}\Var_h^*(s,a)$. The instance-dependent regret upper bound of \emph{MIN-UCB-VI} is
    \begin{align}
        \mathrm{Regret}(K)\leq \Tilde{\mathcal{O}}\left(   \sum_{s,a} \max\left\{ \frac{H^2\Var^*(s,a)}{\Delta(s,a)}-\mathrm{sav}(s,a),\Delta(s,a)\right\}\right), \label{eq:regret_dependent}
    \end{align}
    where $\mathrm{sav}(s,a):=N^{src}(s,a)\Delta(s,a)\max\{1-\frac{H\Tilde{\nu}(s,a)}{\Delta(s,a)},0\}^2$.
    \begin{remark}
        The proof for Theorem~\ref{thm:regret_dependent} is provided in Appendix~\ref{sec:regret_dependent_proof}. Comparing with the regret bound $\Tilde{\mathcal{O}}(\sum_{s,a}\frac{H^2\Var^*(s,a)}{\Delta(s,a)})$ of vanilla UCB-VI, MIN-UCB-VI provides an improvement on the regret by the saving term $\mathrm{sav}(s,a)$. Note that $\mathrm{sav}(s,a)\geq 0$ always. The case of $H\Tilde{\nu}(s,a)\leq \Delta(s,a)$ means that the environment $\mathcal{M}^\tar$ and $\mathcal{M}^\src$ on area  $(s,a)$ are sufficiently similar. Hence the offline data can be leveraged by MIN-UCB-VI to enhance online learning efficiency. In contrast, when $H\Tilde{\nu}(s,a)\geq \Delta(s,a)$, offline data on $(s,a)$ will be ignored and MIN-UCB-VI still matches the performance guarantee of the vanilla UCB-VI. Additionally, when $H\Tilde{\nu}(s,a)<\Delta(s,a)$, we observe that $\mathrm{sav}(s,a)$ is increasing in $N^\src(s,a)$, which reflects that more offline data benefit the online learning process when the target and source environments are sufficiently similar. We also remark that the current proof prioritizes the offline-to-online effect and uses a simple
local bonus-threshold argument. The horizon dependence could be sharpened by applying more fine-grained clipped regret
decomposition of~\citet{simchowitz2019non}.
    \end{remark}
\end{theorem}

\subsection{MAX-LCB-VI for Best Policy Identification}

In this subsection, we introduce our presented MAX-LCB-VI in Algorithm~\ref{alg:max-lcbvi}, which is tailored for the task of best policy identification. Different from the optimism principle applied in MIN-UCB-VI to encourage exploration, the key idea of MAX-LCB-VI is to leverage the pessimism principle to identify the optimal policy. However, we note that the vanilla idea of selecting the smaller optimistic estimates in~\citet{DBLP:conf/icml/CheungL24} can  not be directly applied in this setting since here we construct the pessimistic estimates of $Q^*_h$ here. We thus devise a novel hybrid algorithm tailored for this setting, introduced below.

\RestyleAlgo{ruled}
\LinesNumbered
\begin{algorithm}[htbp]{
\caption{MAX-LCB-VI.}\label{alg:max-lcbvi}

\textbf{Input}: Valid bias bound $\nu$, Offline data set $\Dc_{\mathrm{src}}$, online data set $\Dc_K$ . \\ 
\textbf{Initialize}: $\hat{P}^{\mathrm{src}}(s'|s,a)= \frac{N^{\mathrm{src}}(s,a,s')}{\max\{1,N^{\mathrm{src}}_{s,a}\}}$, $\hat{P}^{K}(s'|s,a)=\frac{N^{K}(s,a,s')}{\max\{1,N^{K}_{s,a}\}}$. 

Define $\hat{P}^{K,\hyb}(s'|s,a):=\frac{\Pest_{s,a}^K N^K_{s,a} +\Pest_{s,a}^{\src}N^{\src}_{s,a}}{N^K_{s,a}+N^{\src}_{s,a}}$

    \For{$h = \{ H, H-1, \ldots, 1 \}$}{
    
        \For{$s \in \mathcal{S}$}{

            \For{$a \in \mathcal{A}$}{
                $\hat{Q}^{K,\on}_h (s,a) = 0\lor (r(s,a) + \Pest^K_{s,a}\cdot \Vest_{h+1}- b^K_{h} (s, a))$ \label{alg:lcbvi:q_k}

                 $\hat{Q}^{K,\hyb}_h (s,a) =  0\lor (r(s,a) +\hat{P}^{K,\hyb}_{s,a}\cdot \Vest_{h+1}- b^{K,\hyb}_{h} (s, a))$ \label{alg:lcbvi:q_combine}
            }
    
            $\Vest_h(s) = \max_{a \in \mathcal{A}} \max\{\hat{Q}^{K,\on}_h(s, a),\hat{Q}^{K,\hyb}_h(s, a)\}$ \label{alg:Lcbvi:vvals}

            $\hat{\pi}_h(s) \in \arg\max_{a \in \mathcal{A}}  \max\{\hat{Q}^{K,\on}_h(s, a),\hat{Q}^{K,\hyb}_h(s, a)\} $
                        
        }
    }

}

\end{algorithm}

For the task of best policy identification, the agent needs to  first compute two empirical transition probabilities $\hat{P}^{\src}_{s,a}$ and $ \hat{P}^{K}_{s,a}$ based on offline and online data, respectively. Then a fused empirical transition probability $\hat{P}^{K,\hyb}_{s,a}$ is computed. With these estimated transition probabilities, the MAX-LCB-VI algorithm maintains the value function estimate $\hat{V}_h$ and two Q-function estimates $\hat{Q}^{K}_h, \hat{Q}^{K,\hyb}_h$, working backward from $h=H$ to $h=1$. The construction of Q-function estimates in this subsection follows the pessimism principle. In specific, the agent updates them by $\hat{Q}^K_h(s,a)=\max\{r(s,a)+\hat{P}_{s,a}^K\cdot\hat{V}_{h+1}-b^K_h(s,a),0\}$ and $\hat{Q}^{K,\hyb}_h(s,a)=\max\{r(s,a)+\hat{P}_{s,a}^{K,\hyb}\cdot\hat{V}_{h+1}-b^{K,\hyb}_h(s,a),0\}$. 

The key idea of MAX-LCB-VI is to compute $\hat{V}_h$ by selecting the bigger one of two pessimistic Q-function estimates, i.e., $\hat{V}_{h}(s)=\max_{a}\max\{\hat{Q}^K_h(s,a),\hat{Q}^{K,\hyb}_h(s,a)\}$. Similarly, the policy $\hat{\pi}$ is selected greedily in accordance to the bigger Q-function estimates: $\hat{\pi}(s)\in\arg\max_a \max\{\hat{Q}^K_h(s,a),\hat{Q}^{K,\hyb}_h(s,a)\}$.

In this paper, we follow the manner considered in~\citet{NEURIPS2021_e61eaa38,li2024settling} to compute two exploration bonuses in a Bernstein style.  The bonus terms based on purely online data and fused data are computed by
\begin{align}
b_h^K(s,a)&=c_1\sqrt{\frac{\mathbb{V}_{s'\sim\hat{P}^K_{s,a}}(\hat{V}(s'))L}{N^K_{s,a}}}+\frac{c_2HL}{N^K_{s,a}},\nonumber\\
 b_h^{K,\hyb}(s,a)&=c_1\sqrt{\frac{\mathbb{V}_{s'\sim\hat{P}^{K,\hyb}_{s,a}}(\hat{V}_{h+1}(s'))L}{N^{K}_{s,a}+N^\src_{s,a}}}+\frac{c_2HL}{N^K_{s,a}+N^\src_{s,a}}+\frac{H\nu(s,a)N^\src_{s,a}}{N^K_{s,a}+N^\src_{s,a}}.\nonumber
\end{align}
where $L:=\log(SAKH/\delta)$.

    

Besides, we utilize the two-fold subsampling techniques proposed in~\citet{li2024settling} in our proposed MAX-LCB-VI algorithm to tackle  the technical difficulty that sample transitions in collected data cannot be regarded as independently generated. Below we provide both instance-independent and dependent theoretical guarantees for our MAX-LCB-VI.

\begin{algorithm}[htp]
    \DontPrintSemicolon
    \SetKwData{Left}{left}\SetKwData{This}{this}\SetKwData{Up}{up}
    \SetKwFunction{Union}{Union}\SetKwFunction{FindCompress}{FindCompress}
    \SetKwInOut{Input}{input}\SetKwInOut{Output}{output}

    \Input{Dataset $\mathcal{D}$, reward function $r$}
    \BlankLine
    Split $\mathcal{D}$ into 2 halves containing same number of sample trajectories, $\mathcal{D}^{\mathrm{main}}$ and $\mathcal{D}^{\mathrm{aux}}$\; $\mathcal{D}_0=\{\}$\;

    \For{$(h,s)\in[H]\times \mathcal{S}$}{
        $N^{\mathrm{trim}}_h(s)\leftarrow \max\{0, N^{\mathrm{aux}}_h(s) - 10\sqrt{N_h^{\mathrm{aux}}(s)\log\frac{HS}{\delta}}\}$\;
        Randomly subsample $\min\{N^{\mathrm{trim}}_h(s), N^{\mathrm{main}}_h(s)\}$ samples of transition from $(h,s)$ from $\mathcal{D}^{\mathrm{main}}$ to add to $\mathcal{D}_0$
    }
    $\hat{\pi}\leftarrow \text{MAX-LCB-VI}(\mathcal{D}_0, r)$\;
    \Output{policy $\hat{\pi}$}
\caption{Two-fold Subsampling for MAX-LCB-VI}
\label{alg:LCB with subsampling}
\end{algorithm}

\subsubsection{Instance-Independent Sub-optimality Upper Bound}

\begin{theorem}\label{thm:bpi_independent} Let $\hat{\pi}$ be the output policy of MAX-LCB-VI. Suppose $\pi^*$ is an optimal policy. Then with probability $1-\delta$, the sub-optimality gap of $\hat{\pi}$ is
\begin{align}
     &V_1^{P_{\tar},*}-V_1^{P_{\tar},\hat{\pi}}\leq \Tilde{\mathcal{O}}\left(  H\min\left\{ \sqrt{\frac{SA}{K}}, H\nu_{\max}+\sqrt{\frac{S}{N^{\src}/C^{*}_\nu+K/A}} \right\}\right).\nonumber
\end{align}
\begin{remark}
   The proof for Theorem~\ref{thm:bpi_independent} is provided in Appendix~\ref{sec:bpi_independent_proof}. Below, we provide performance analysis for MAX-LCB-VI w.r.t different qualities of  $\pi^b$ and bias $\nu$. When the behavior policy satisfies  $\pi^b=\pi^*$ and the environment bias $\nu(s,a)=0, \forall s,a$, i.e., $P^\tar_{s,a}=P^\src_{s,a},\forall s,a$, the occupancy measure $d^{\pi^b,P^\src}_h(s,a)$ satisfies $d^{\pi^b,P^\src}_h(s,a)=d^{\pi^*,P^\tar}_h(s,a),\forall s,a,h$ and thus the concentrability $C^*_\nu=1$. The sub-optimality gap of $\hat{\pi}$ becomes $\Tilde{\mathcal{O}}(H\sqrt{\frac{S}{N^\src+K/A}})$, implies the strict improvement of offline data on MAX-LCB-VI and increasing offline data size $N^\src_{s,a}$ will benefit the performance of learned policy $\hat{\pi}$. When the environment bias satisfies $\nu(s,a)=0, \forall s,a$, the second term of this upper bound recovers $H\sqrt{\frac{S}{N^\src/C^*+K/A}}$, which matches with the performance of studies for hybrid RL~\citep{huang2025augmenting}. 
\end{remark} 
\end{theorem}

\subsubsection{Instance-dependent Sub-optimality Upper Bound}


\begin{theorem}\label{thm:bpi_dependent}
MAX-LCB-VI returns an
$\epsilon$-optimal policy with probability at least $1-\delta$ whenever
\begin{align}
K
\geq
\Omega\left(A\left[
\frac{H^3SL^2}{\epsilon\Delta_{\min}}
-
\frac{N^{\rm src}}{C_\nu^*}
\left(
1-\frac{ H^2\nu_{\max}}{\Delta_{\min}}
\right)_+
\right]_+\right),
\label{eq:sample-complexity}
\end{align}
where $\Delta_{\min}:=
\min_{\substack{h,s,a:\Delta_h(s,a)>0}}
\Delta_h(s,a)$.
  \begin{remark}
      The proof for Theorem~\ref{thm:bpi_dependent} is provided in Appendix~\ref{sec:bpi_dependent_proof}. This theorem shows that to learn an $\epsilon$-optimal policy, MAX-LCB-VI needs the number  of episodes $K$ during online exploration to be greater than $\Omega(\frac{SAH^3L^2}{\epsilon\Delta_{\min}})$, where the improvement is from the offline data. When $\nu(s,a)=0,\forall (s,a)$ and $\pi^b=\pi^*$, the number  of episodes $K$ should satisfy $K\geq \Omega(\frac{SAH^3L^2}{\epsilon\Delta_{\min}}-\frac{AN^\src}{C^*_\nu})$ to output an optimal policy. For this case, we can clearly see that the increasing of $N^\src$ can improve the performance of MAX-LCB-VI.
  \end{remark}
    
\end{theorem}

\subsection{Model Selection under Unknown Distribution Bias}
In this subsection, we consider the setting where the agent does not have prior knowledge of bias level $\{\nu\}_{s,a}$.  Here we employ a classical model selection paradigm, widely used in online learning, to adaptively estimate the unknown distribution bias $\nu=\{\nu(s,a)\}$. In specific, we apply a model selection method called EXP3.P~\citep{auer2002nonstochastic}, to maintain multiple experts with different estimates of the unknown bias.


Below, we present the detailed algorithm for model selection used to handle unknown bias in regret minimization.
When the bias level $\nu(s,a)$ is unknown, the bonus in MIN-UCB-VI can not be directly constructed since the enlarged term $H\nu(s,a)N^\src_{s,a}/(N^\src_{s,a}+N^k_{s,a})$ depends explicitly on $\nu(s,a)$. This term plays a crucial role in compensating for the transition shift $(P^\tar-P^\src) \cdot V^*_h $ when establishing optimism guarantees for MIN-UCB-VI. Thus, we can regard $(P^\tar-P^\src) \cdot V^*_h $ as an unknown but fixed quantity that must be estimated during learning. Since we know that $(P^\tar-P^\src) \cdot V^*_h$ can be upper bounded by $\Vert P^\tar-P^\src \Vert_1 \Vert V^*_h \Vert_{\infty}\leq 2H$, we can use the peeling technique commonly used in online learning to adaptively estimate this term. In specific, we can maintain $\log (2H)$ experts for each bonus $b^{k,\hyb}_{h}(s,a), \forall h,s,a$. Hence we should maintain a total of $SAH\log 2H$ experts. Then we use the model selection introduced below to adaptively estimate the unknown parameter needed by MIN-UCB-VI.

Below we present the protocol for model selection used in MIN-UCB-VI. This high-level overview outlines the workflow for integrating model selection into MIN-UCB-VI. Further technical details can be found in~\citet{pacchiano2020model}.

\begin{algorithm}[htbp]
    \caption{Model selection for MIN-UCB-VI under unknown bias}
    \label{alg:hybrid}
    \begin{algorithmic}[1]
        \STATE {\bf Input:} Base algorithm $\{B_j\}_{j=1}^{\log (2H)}$ for each $(s,a,h)$, learning rate $\eta$.
        \STATE { Initialize all base algorithms}
        \FOR{$k=1,\ldots, K$}
        \FOR{$(h,s,a) \in [H]\times \State \times \Action$}
        \STATE Sample $j^k_{s,a,h}\sim q^k_{s,a,h}$
        \STATE Run MIN-UCB-VI based on the $b^{j^k,\hyb}_h(s,a)$ output by expert $j^k_{s,a,h}$.
        \STATE Receive feedback and updates expert $j^k_{s,a,h}$.
        \STATE Update the sampling distribution $q^k_{s,a,h}$ of experts.
            \ENDFOR
        \ENDFOR
    \end{algorithmic}
    \label{alg:model_selection}
\end{algorithm}

From Algorithm~\ref{alg:model_selection}, we see that for each $(s,a,h)$ and each expert $j\in\log(2H)$, the expert maintains its own estimate of the quantity $H\nu(s,a)$ used in the bonus term $b^{k,\hyb}_h(s,a)$. The core idea of model selection is to allocate more samples to an expert when the meta-algorithm achieves better performance using that expert's estimate of $H\nu(s,a)$, thereby automatically adapting to the unknown bias.


 Below we introduce a key lemma that relates the regret bound of the meta-algorithm to that of the base algorithm.
 
\begin{lemma}[Theorem 5.3 in~\citet{pacchiano2020model}]
    If base algorithm instance has regret upper bound $\mathrm{Regret}(K)\leq K^{\alpha} c(\delta)$ with probability $1-\delta$ for $\alpha\in[\frac{1}{2},1)$ and the function $c$, then the regret of meta-algorithm EXP3.P is $\mathrm{Regret}(K)\leq  \Tilde{\mathcal{O}}(K^{\alpha} c(\delta)^{\frac{1}{\alpha}})$.
\end{lemma}

Based on this lemma, we can derive the regret bound of the meta algorithm as follows.
\begin{theorem}
The regret of MIN-UCB-VI with model selection for unknown bias setting is 
    \begin{align}
       \mathrm{Regret}(K)&\leq \Tilde{\mathcal{O}}\left( \min\left\{ H^2SA\sqrt{K},  K^{\frac{3}{2}}\left(\Tilde{\nu}_{\max}^2+\frac{H^3}{\tau^2_*}\right)\right\}\right).\nonumber
    \end{align}
\end{theorem}

\begin{remark}
  We observe that when $\nu_{\max}=o(\frac{H\sqrt{SA}}{K})$ and $N^{\src}\geq \Omega(\frac{H\sqrt{K}}{\sqrt[4]{SA}})$, the regret of model selection improves over the $\tilde{\mathcal{O}}(H\sqrt{SAKH})$ rate of pure online learning up to logarithmic factors. Conversely, when the cumulative bias $\tilde{\nu}_{\max}^2$ dominates, the bound degrades to $\tilde{\mathcal{O}}(H^2SA\sqrt{K})$, which is a factor of $H\sqrt{SA}$ worse than the $\tilde{\mathcal{O}}(H\sqrt{SAKH})$ guarantee of the vanilla UCB-VI. In the special case of multi-armed bandits ($S=H=1$), this approach incurs an extra $\sqrt{A}$ factor compared to UCB, which is tolerable.
\end{remark}

\section{Lower Bound}
In this section, we provide the instance-independent and dependent lower bounds 
for two considered learning objectives   to show the optimality of our proposed algorithms.

\subsection{Regret Minimization}
In this subsection, we adopt the techniques from~\citet{DBLP:conf/icml/CheungL24} and~\citet{domingues2021episodic} in our considered settings to derive the regret lower bound. 

\begin{theorem}[Instance-Independent Regret Lower Bound] \label{thm:regret_independ_lower} Let $S\geq6$, $A\geq2$, $H\geq8$, and $K\geq1$ be integers. Let $\nu_{\max}\geq 0$ and $\{N^\src_{s,a}\}_{s,a}$ be fixed. Set $\nu(s,a)=\nu_{\max}$ for all $(s,a)\in \State \times \Action$. There exists a target MDP $\mathcal{M}_{\tar}$ , such that
\begin{align}
 \mathbb{E}[\mathrm{Regret}(K)]&\geq \Omega\left( H\min\left\{\sqrt{SAK},KH\nu_{\max}+K\sqrt{\frac{H}{\tau^*}} \right\} \right),\nonumber
\end{align}
where  $\tau^*$  the optimum of the LP.
    
\begin{remark}
  The proof for  Theorem~\ref{thm:regret_independ_lower} is provided in Appendix~\ref{sec:regret_independent_lower_proof}. In the case of $P^\tar_{s,a}=P^\src_{s,a}$, for all $(s,a)$, MIN-UCB-VI achieves a regret of $\Tilde{\mathcal{O}}(KH\sqrt{\frac{H}{\tau^*}})$ and our derived lower bound is $\Omega(KH\sqrt{H/\tau^*})$, indicating the near-optimality of MIN-UCB-VI when there exists no environment heterogeneity.
\end{remark}
\end{theorem}

\begin{theorem}[Instance-Dependent Regret Lower Bound] \label{thm:regret_depend_lower} Let $\nu_{\max}\geq 0$ and $\{N^\src_{s,a}\}_{s,a}$ be fixed but arbitrary. There exists a target MDP $\mathcal{M}_{\tar}$ , such that
\begin{align}
 \mathbb{E}[\mathrm{Regret}(K)]&\geq \Omega\left(\sum_{s,a}\frac{H^2\log K}{\Delta(s,a)}-N^\src_{s,a}\Delta(s,a)\max\left\{1-\frac{\Tilde{\nu}(s,a)}{\Delta(s,a)},0 \right\}^2\right),\nonumber
\end{align}
for all sufficiently large $K$.
\begin{remark}
    Theorem~\ref{thm:regret_depend_lower} is proved in Appendix~\ref{sec:regret_dependent_lower_proof}. We observe that this lower bound and~\eqref{eq:regret_dependent} are nearly matching up to $\mathrm{poly}(H)$ factor, which can be improved by the conditional variance technique proposed in~\citet{chen2025sharp}. 
\end{remark}
\end{theorem}

\subsection{Best Policy Identification}

In this subsection, we follow the techniques from~\citet{NEURIPS2021_e61eaa38} to derive the sub-optimality gap lower bound by constructing hard MDP instances.  The key idea is to introduce a behavior policy defined over the first $\min\{C^*_\nu, A\}$ actions, and  define an alternative  MDP using $\nu_{\max}$ to capture the intrinsic difficulty of learning under shifted dynamics.

\begin{theorem}[Instance-Independent Sub-optimality gap Lower Bound] \label{thm:bpi_independent_lower} Suppose \(S\geq 8\), \(A\geq 2\), \(H\geq 8\), \(K\geq 1\),
\(N^{\src}\geq 0\), \(C_\nu^*\geq 2\), and \(\nu_{\max}\geq 0\). We define $M:=\min\{\lfloor C_\nu^*\rfloor,A\}$.
 There exists a class of source--target
MDP pairs 
such that every algorithm which observes \(N^{\src}\) source trajectories,
interacts with the target MDP for \(K\) episodes, and outputs
\(\widehat\pi\) satisfies
\begin{align}
\E\!\left[
V_1^{*,\tar}(s_0)-V_1^{\widehat\pi,\tar}(s_0)
\right]
\geq
\Omega\left(\min\left\{
H\sqrt{\frac{SM}{K}},\,
H^2\nu_{\max}
+H\sqrt{\frac{S}{K/M+N^{\src}/C_\nu^*}}
\right\}\right).
\nonumber
\end{align}

\begin{remark}
    Theorem~\ref{thm:bpi_independent_lower} is proved in Appendix~\ref{sec:bpi_independent_lower_proof}. Recall that when $\nu(s,a)=0$, for all $(s,a)$, MAX-LCB-VI achieves a sub-optimality gap of $\Tilde{\mathcal{O}}(H\min\{\sqrt{\frac{SA}{K}},\sqrt{\frac{S}{N^\src/C^*+K/A}}\})$ which nearly matches with this lower bound  when \(C_\nu^*\geq A\), \(M=A\).
\end{remark}
\end{theorem}

\begin{theorem}[Instance-Dependent Sub-optimality gap Lower Bound] \label{thm:bpi_dependent_lower} Let $S\geq 8$, $A\geq2$, $H\geq5$, $C_\nu^*\geq2$.
For every $\Delta_{\min}\in(0,1/4]$,
$\nu_{\max}\geq0$, and $0<\epsilon
\leq
\frac{(H-2)\Delta_{\min}}{8}$,  such that the following holds for every learning algorithm.
when $C_\nu^*\geq A$, 
\begin{align}
K\leq \mathcal{O}\left( \frac{SHA}{\epsilon\Delta_{\min}}-\frac{A N^{\src}}{C_\nu^*}
\left(
1-\frac{\nu_{\max}}{\Delta_{\min}}
\right)_+\right)
\label{eq:SA-lower-condition}
\end{align}
implies at least one instance in the family satisfies $\Prb\!\left(
V_1^*(s_1)-V_1^{\widehat\pi}(s_1)
\geq\epsilon
\right)
\geq\frac17$.

\begin{remark}
  The proof for  Theorem~\ref{thm:bpi_dependent_lower} is deferred in Appendix~\ref{sec:bpi_dependent_lower_proof}. This lower bound also matches with the upper bound  up to $H$ factor.
\end{remark}
    
\end{theorem}

\section{Experiments}

In this section, we conduct simulation experiments to evaluate the empirical performance of our algorithms. All experiments are performed in the GridWorld environment considered in~\citet{qu2025hybrid}. The environment has $S=16$ states, $A=4$ actions, and horizon $H=20$. The source and target environments share the same state space, action space, and reward function, but differ in their transition dynamics. Specifically, when the success probability is $p$, the agent moves in the intended direction with probability $p$; with probability $1-p$, it moves uniformly at random to one of the unintended directions. We control the bias level by varying the difference between the success probabilities of the source and target environments.

\begin{figure}[htbp]
    \centering
    \begin{subfigure}{0.35\textwidth}
        \centering
        \includegraphics[width=\linewidth]{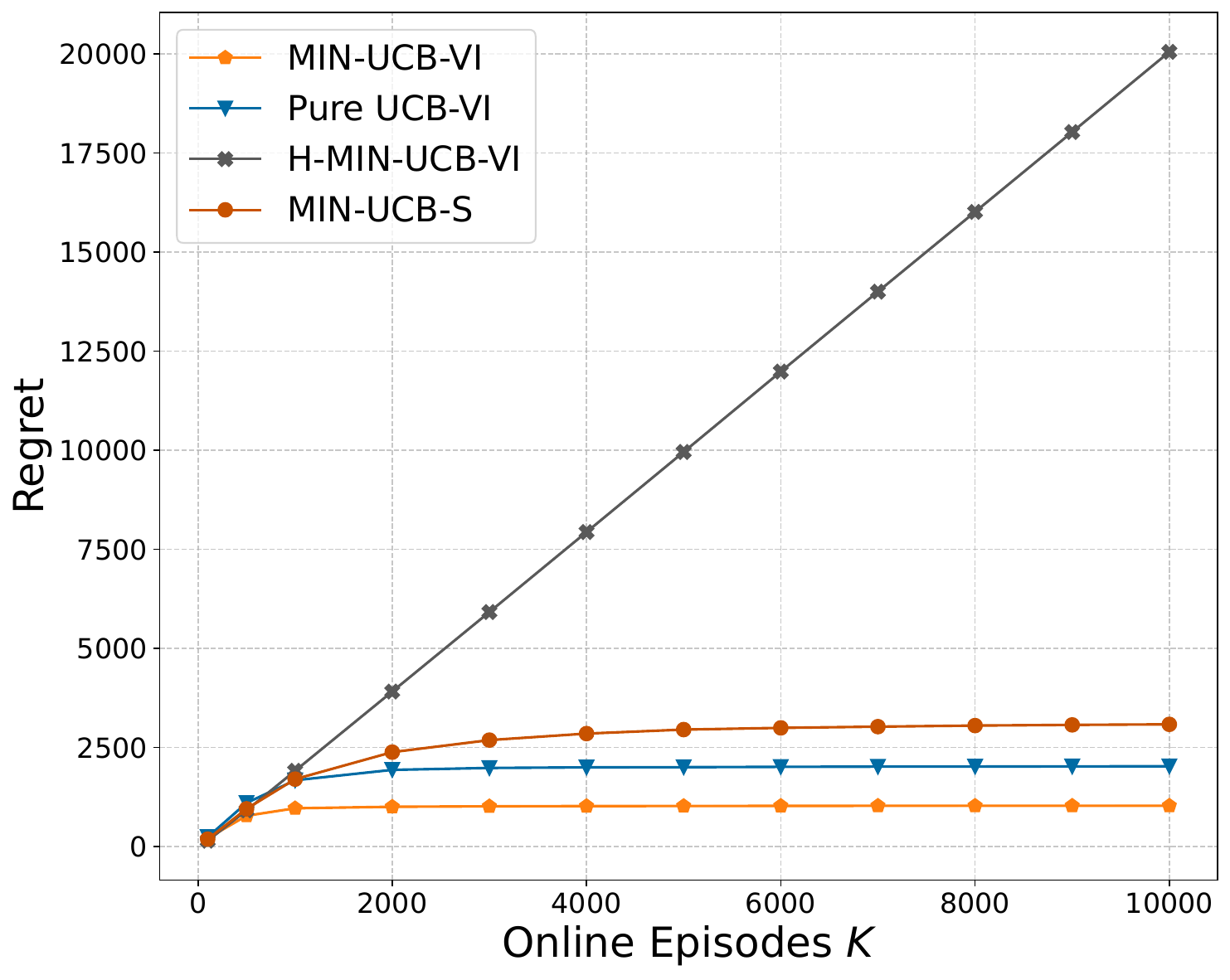}
        \caption{Bias level $0.2$}
        \label{fig:1a}
    \end{subfigure}
    \begin{subfigure}{0.35\textwidth}
        \centering
        \includegraphics[width=\linewidth]{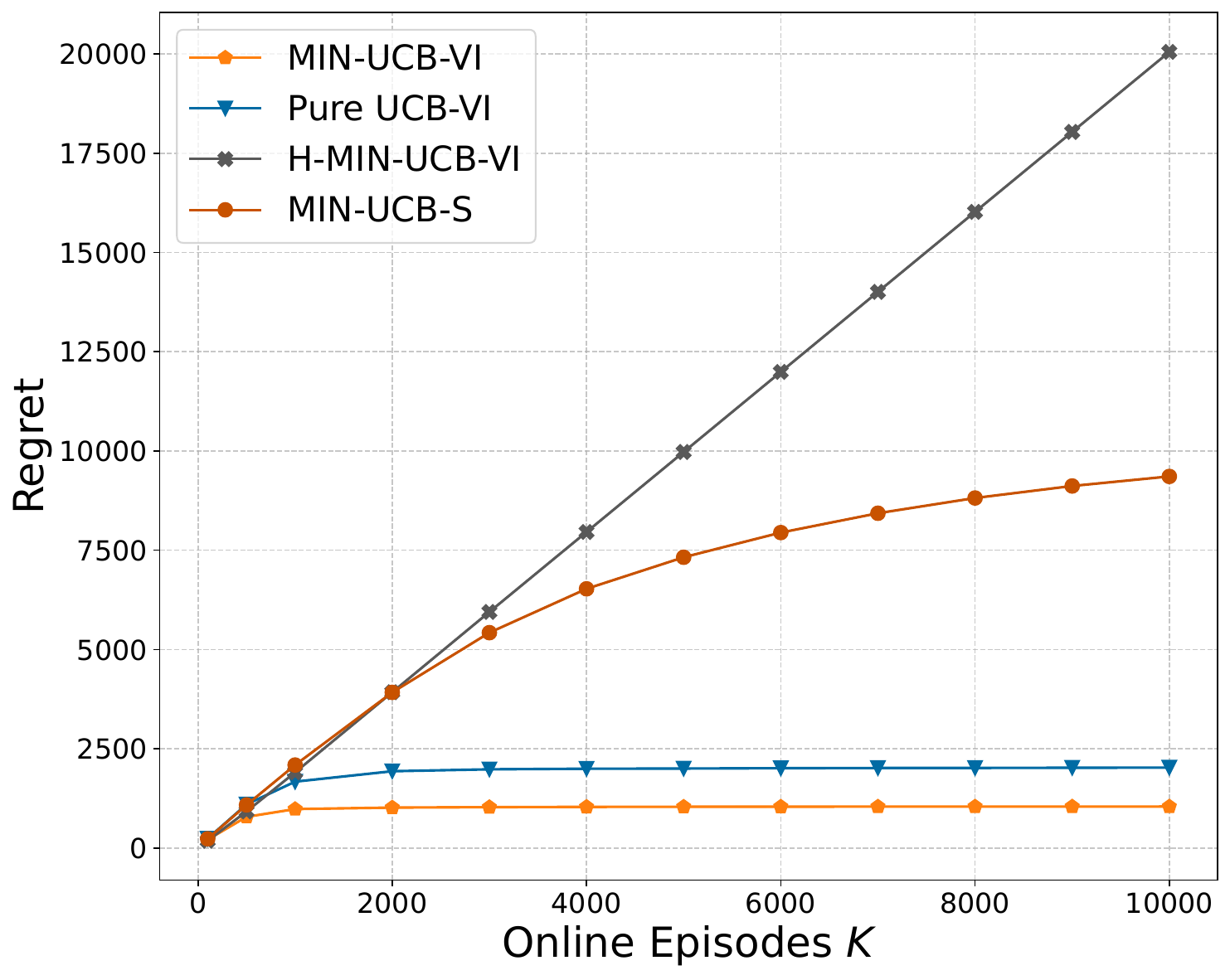}
        \caption{Bias level $0.7$}
        \label{fig:1b}
    \end{subfigure}

    \begin{subfigure}{0.35\textwidth}
        \centering
        \includegraphics[width=\linewidth]{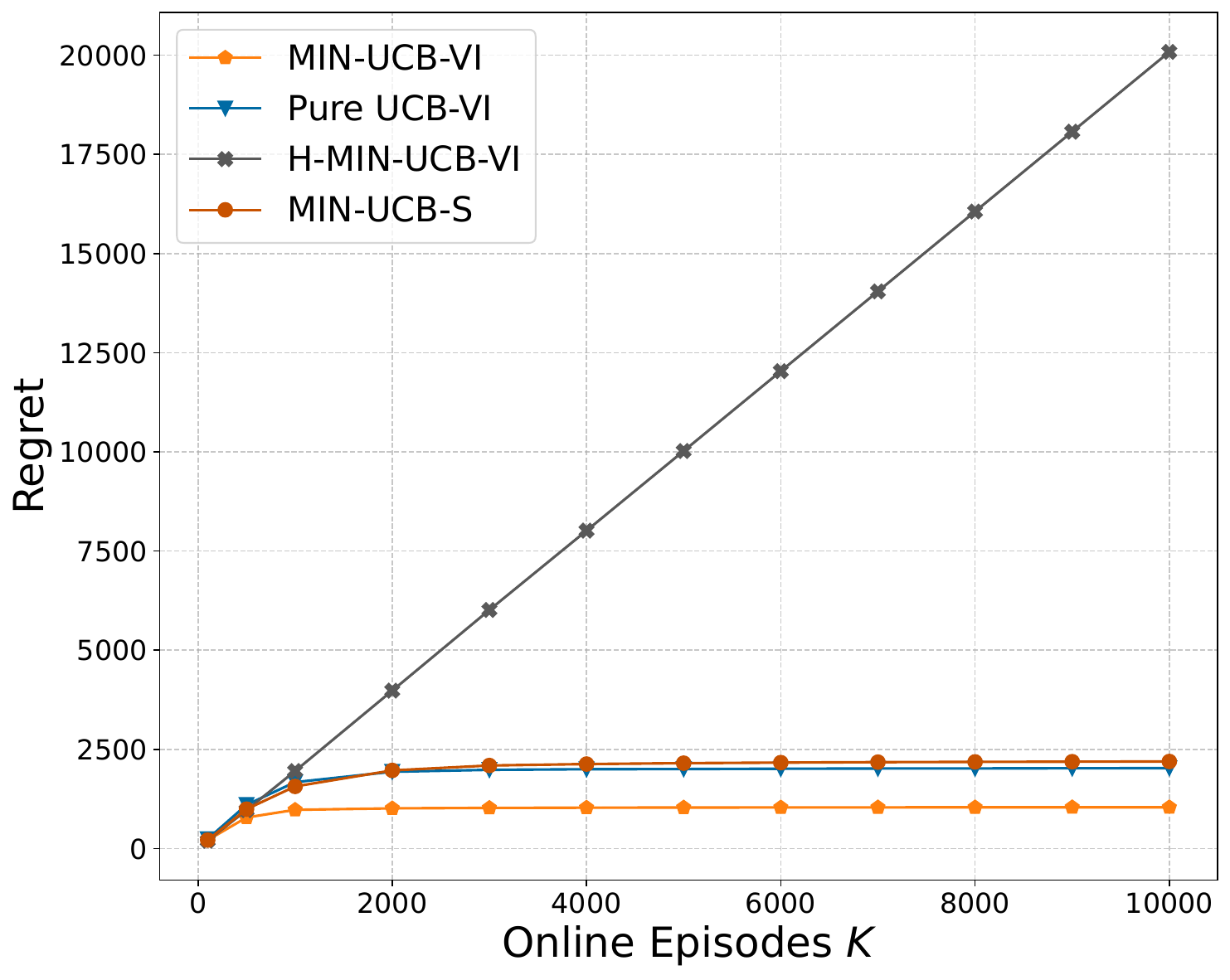}
        \caption{$N^{\mathrm{src}}=100$}
        \label{fig:1c}
    \end{subfigure}
    \begin{subfigure}{0.35\textwidth}
        \centering
        \includegraphics[width=\linewidth]{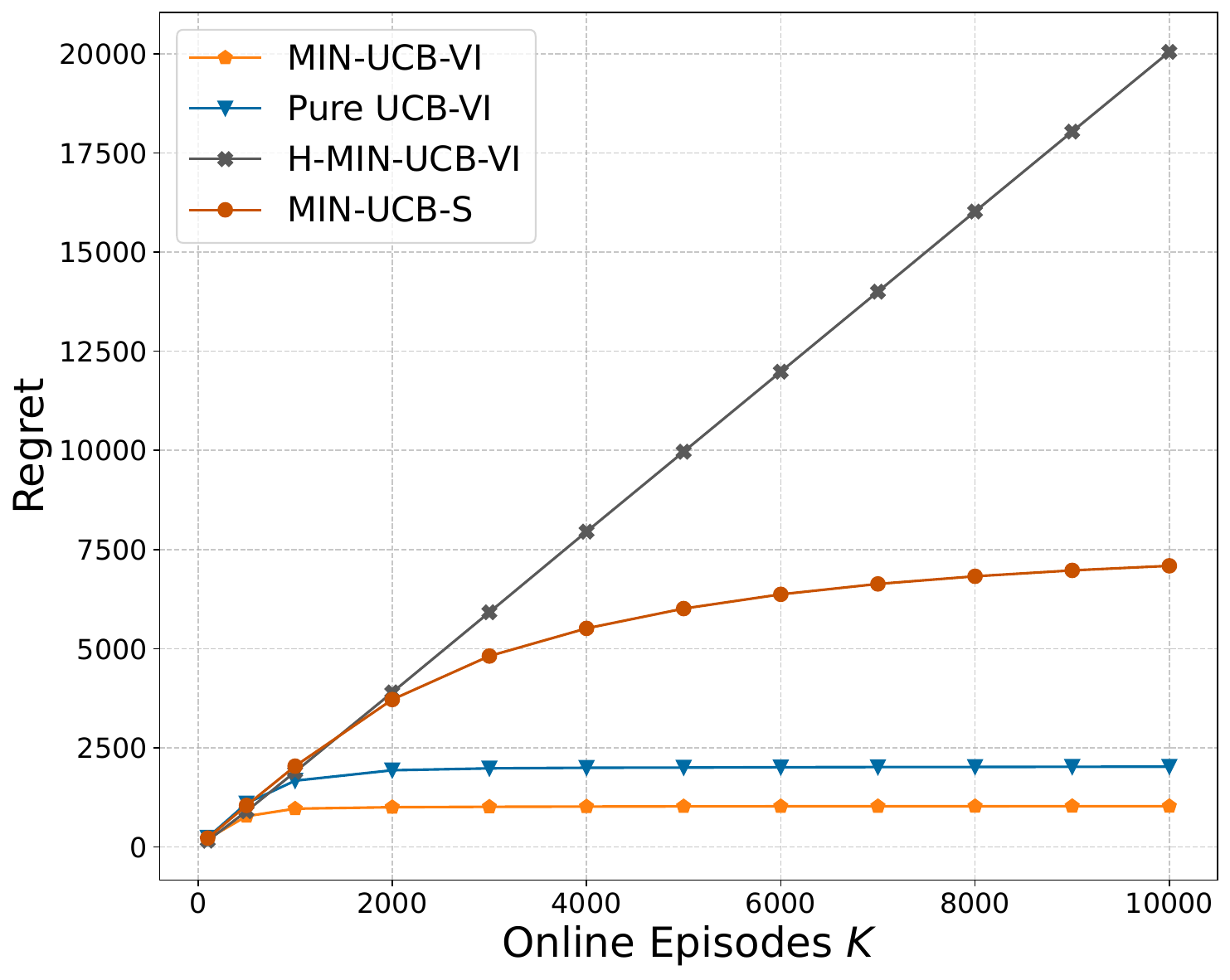}
        \caption{$N^{\mathrm{src}}=1000$}
        \label{fig:1d}
    \end{subfigure}
    \caption{Cumulative regret of MIN-UCB-VI and baselines under different bias levels and offline sample sizes.}
    \label{fig:1}
\end{figure}

All results are reported as means with standard error bars over multiple independent runs. For all algorithms, we set the confidence parameter to $\delta=0.1$ and the bonus scale to $c=0.01$.

\begin{figure}[htbp]
    \centering
    \begin{subfigure}{0.35\textwidth}
        \centering
        \includegraphics[width=\linewidth]{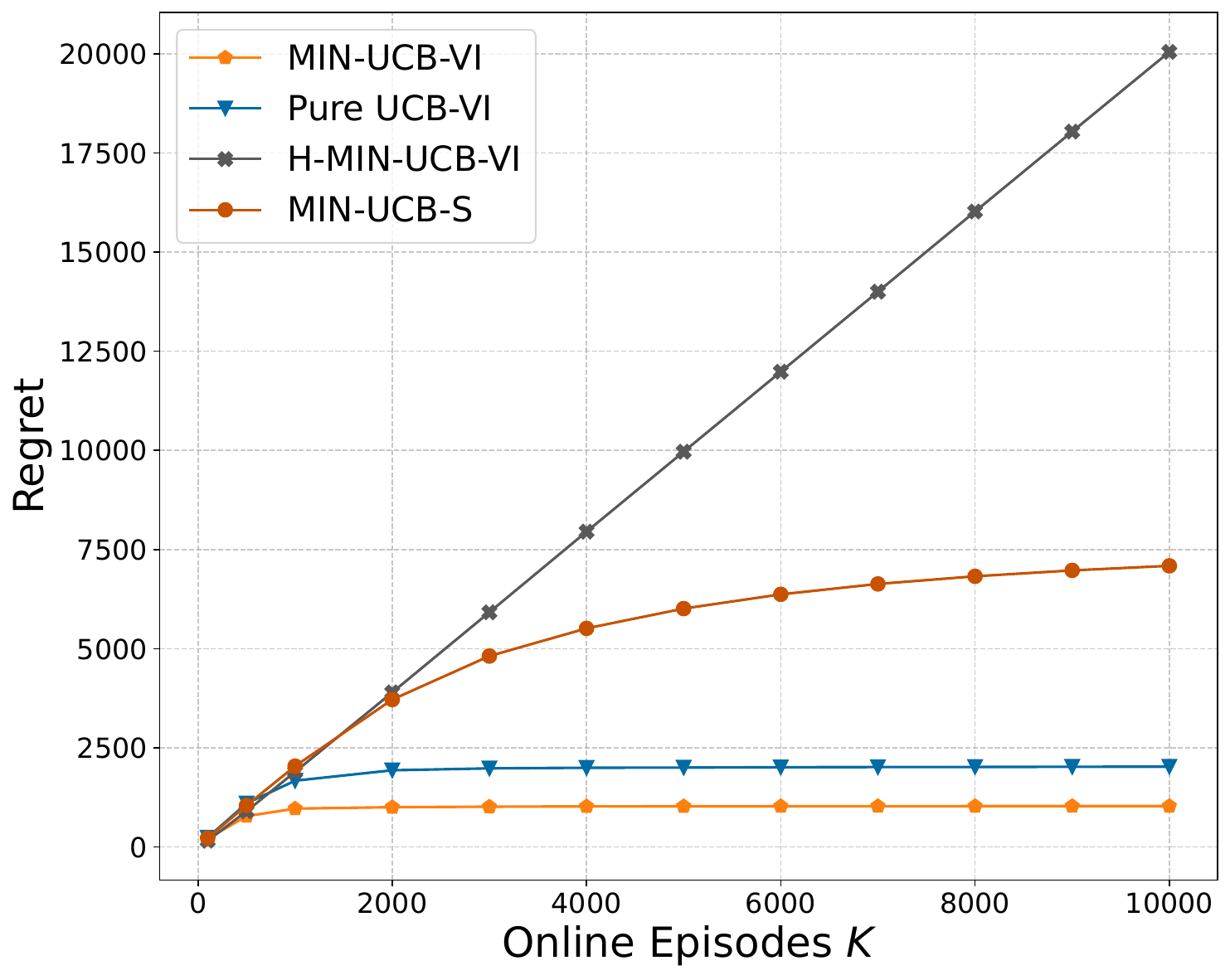}
        \caption{Bias level $0.5$}
        \label{fig:3a}
    \end{subfigure}
    \begin{subfigure}{0.35\textwidth}
        \centering
        \includegraphics[width=\linewidth]{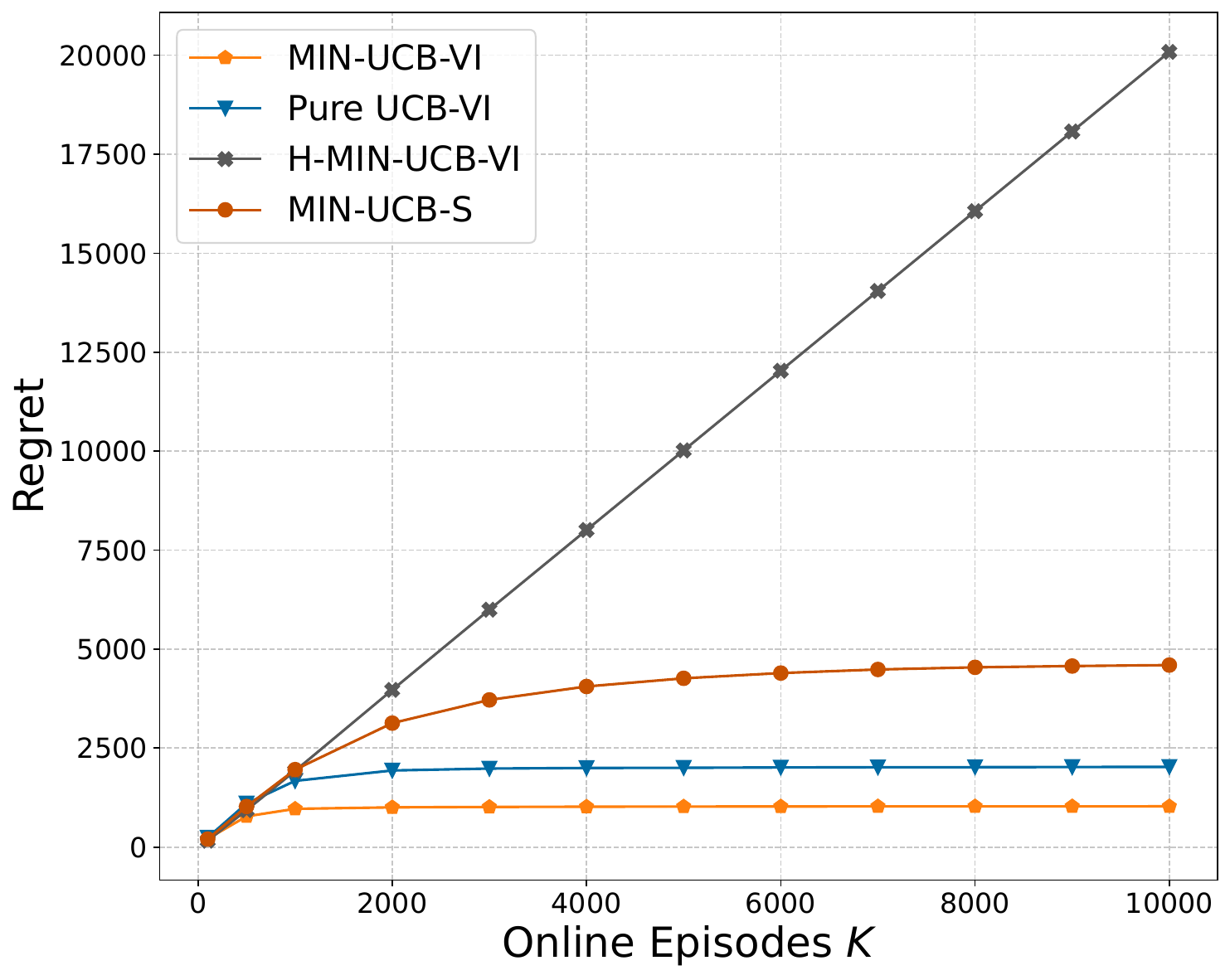}
        \caption{$N^{\mathrm{src}}=500$}
        \label{fig:3b}
    \end{subfigure}
    \caption{Additional cumulative-regret experiments for MIN-UCB-VI.}
    \label{fig:3}
\end{figure}

\subsection{Regret Minimization}
We first evaluate MIN-UCB-VI for regret minimization. We compare it with three baselines: (i) Vanilla UCB-VI with the MVP bonus, which is a purely online baseline~\citep{zhang2021reinforcement}; (ii) UCB-VI-S, which omits the comparison operation and directly uses the source-based optimistic estimate, i.e., $V^k_h(s)=\max_{a\in\mathcal A}Q^{k,\mathrm{src}}_h(s,a)$; and (iii) H-UCB-VI, an MDP extension of H-UCB~\citep{shivaswamy2012multi}, which uses offline data without enlarging the confidence radius to account for source-target bias. For all offline-data-based methods, the offline dataset is collected in the source environment using a common random policy.

\begin{figure}[htbp]
    \centering
    \begin{subfigure}{0.35\textwidth}
        \centering
        \includegraphics[width=\linewidth]{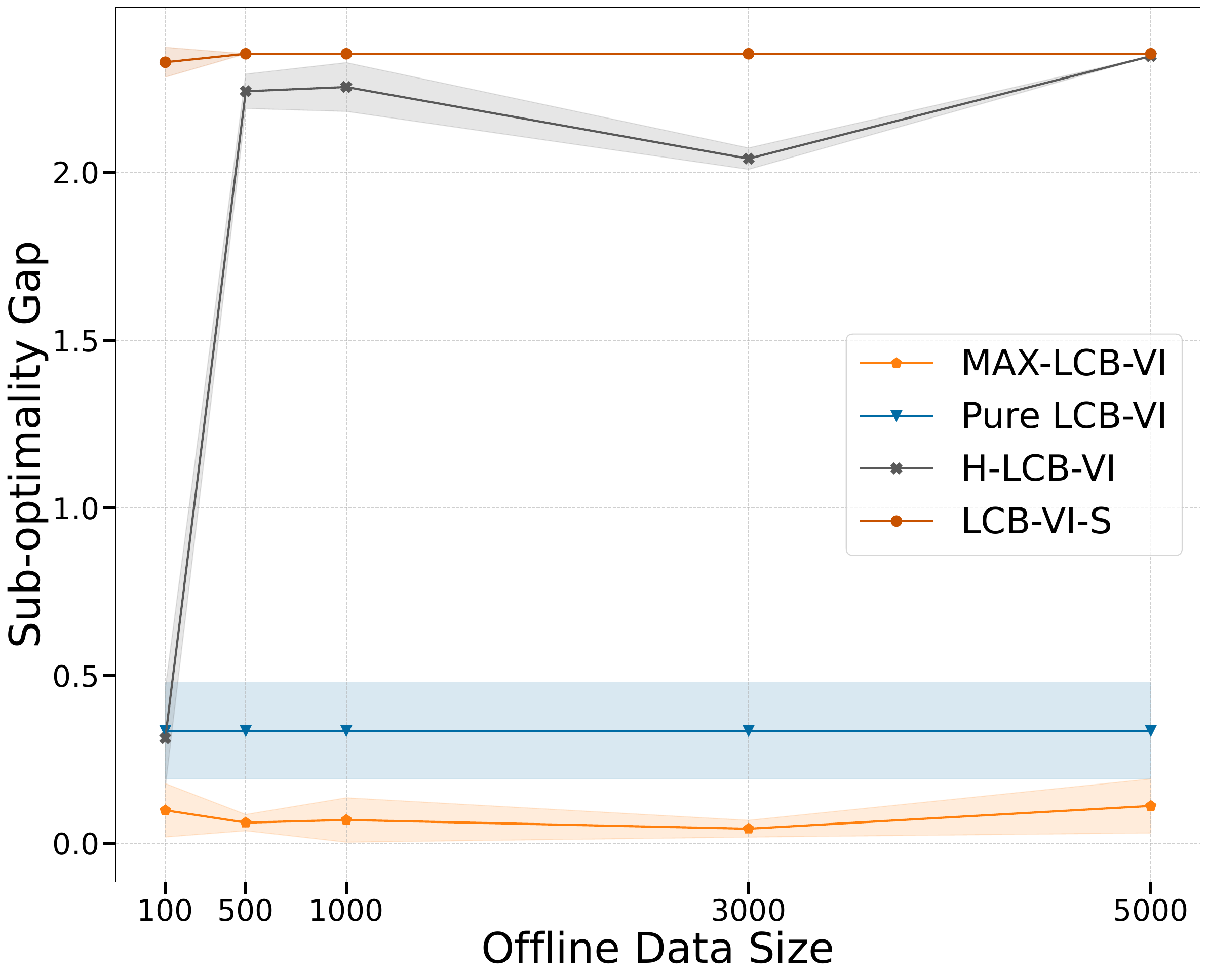}
        \caption{Bias level $0.5$}
        \label{fig:2a}
    \end{subfigure}
    \begin{subfigure}{0.35\textwidth}
        \centering
        \includegraphics[width=\linewidth]{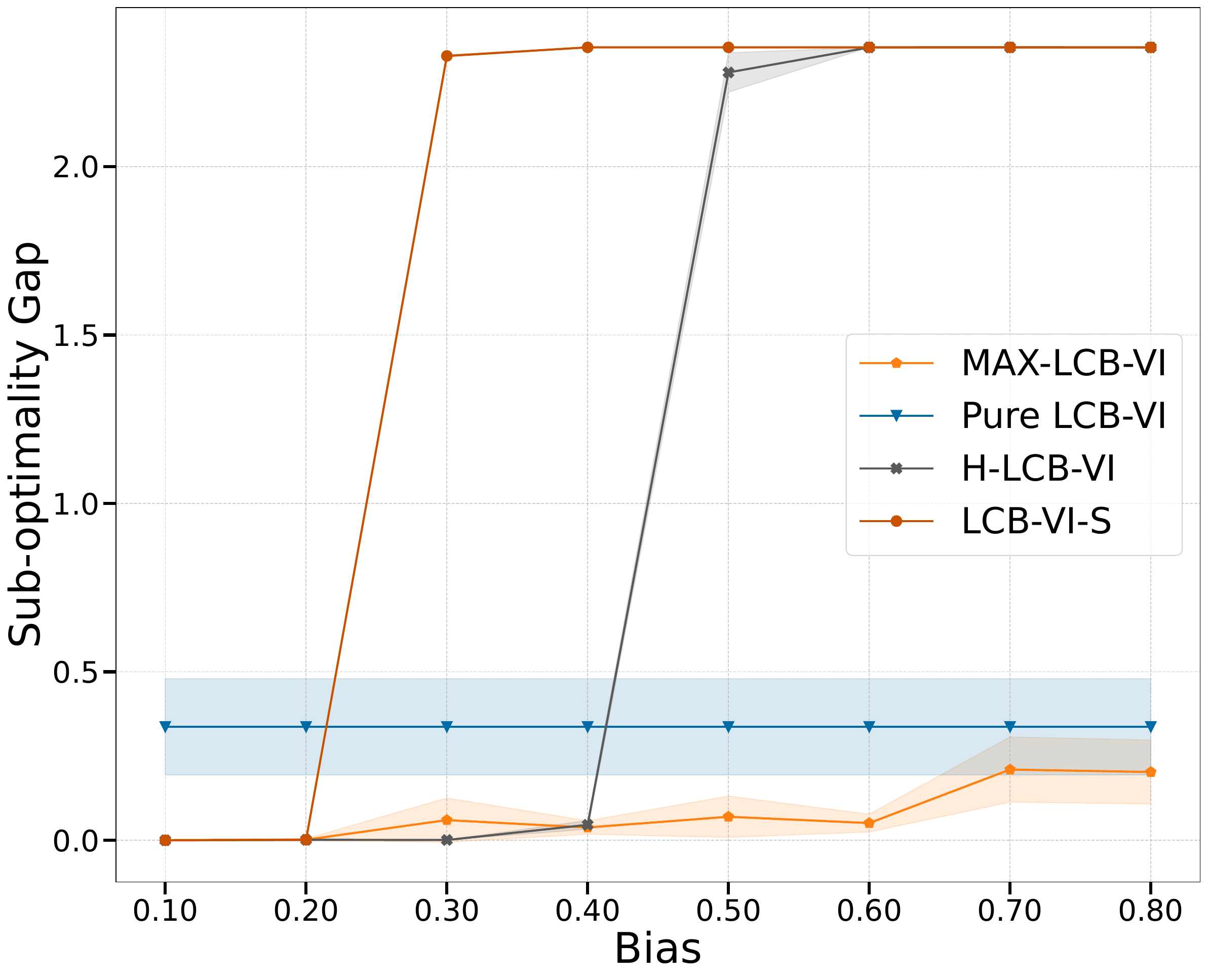}
        \caption{$N^{\mathrm{src}}=1000$}
        \label{fig:2b}
    \end{subfigure}

    \begin{subfigure}{0.35\textwidth}
        \centering
        \includegraphics[width=\linewidth]{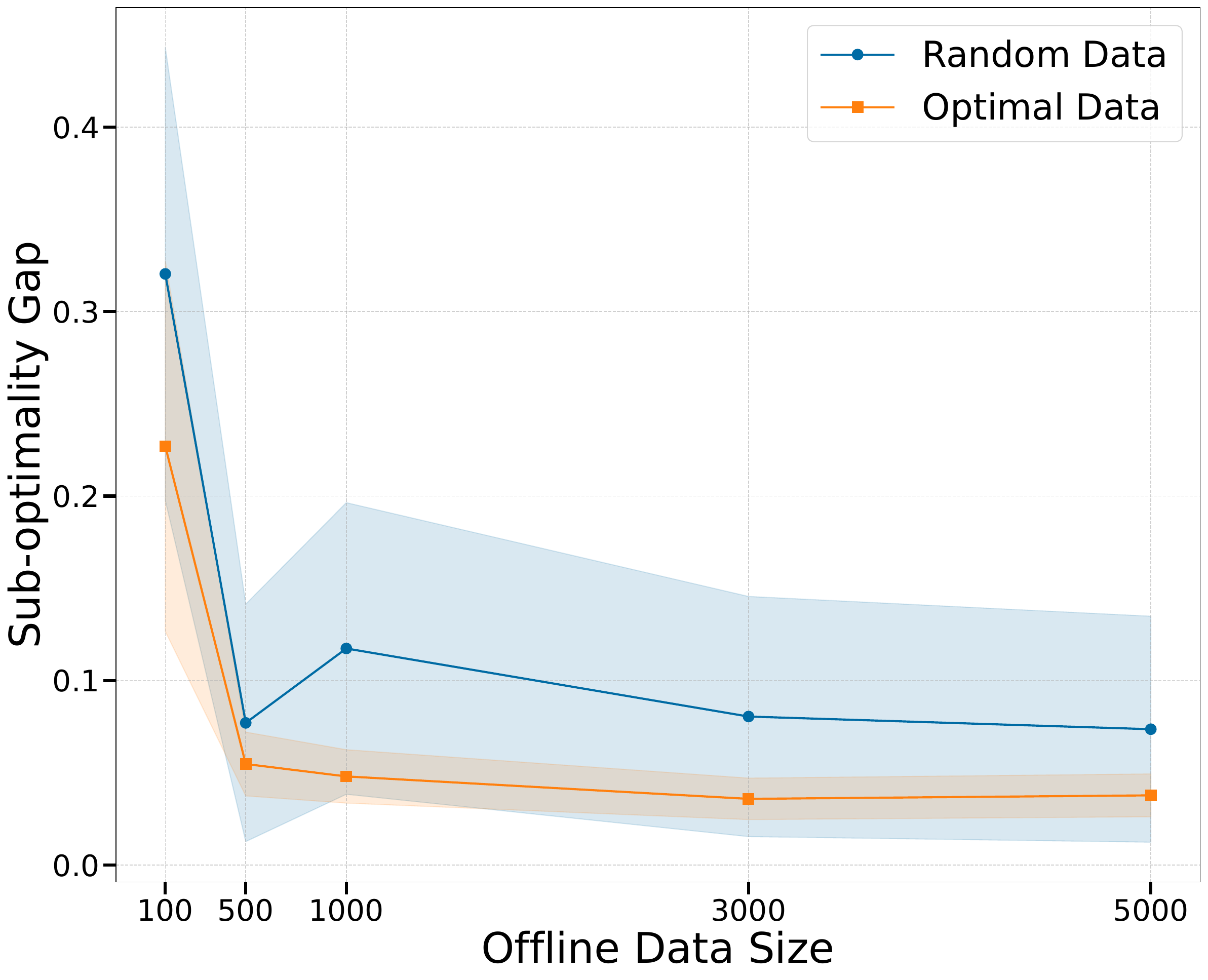}
        \caption{Bias level $0.2$}
        \label{fig:2c}
    \end{subfigure}
    \begin{subfigure}{0.35\textwidth}
        \centering
        \includegraphics[width=\linewidth]{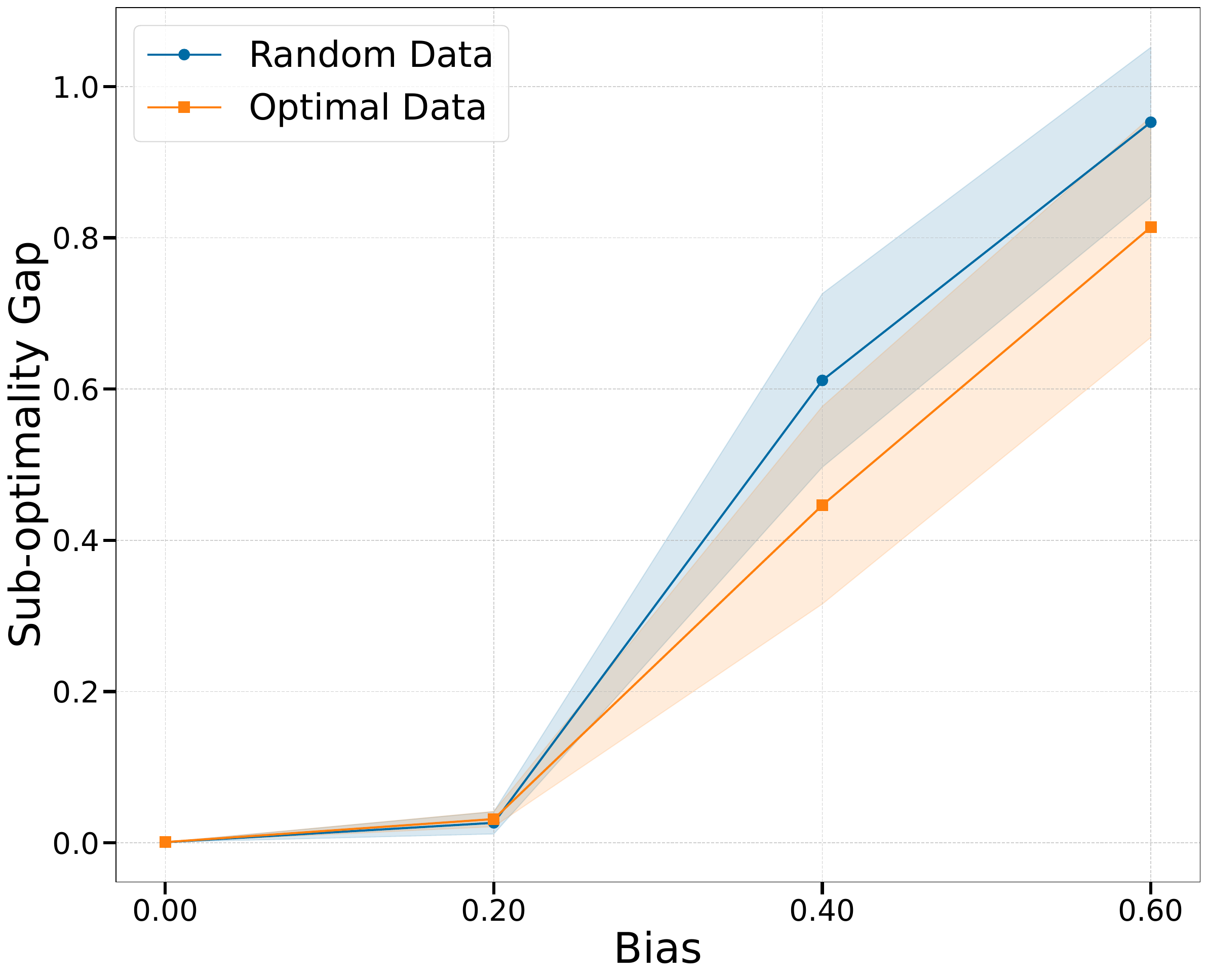}
        \caption{$N^{\mathrm{src}}=1000$}
        \label{fig:2d}
    \end{subfigure}
    \caption{Suboptimality gap of MAX-LCB-VI and baselines under different bias levels, offline sample sizes, and offline data-collection policies.}
    \label{fig:2}
\end{figure}

In Figures~\ref{fig:1} and~\ref{fig:3}, the target success probability is fixed at $0.95$. We vary the source success probability among $0.75$, $0.45$, and $0.25$, corresponding to bias levels $0.2$, $0.5$, and $0.7$, respectively. Figure~\ref{fig:1} shows that MIN-UCB-VI consistently achieves lower cumulative regret than Vanilla UCB-VI across different bias levels and offline sample sizes. Figures~\ref{fig:1a} and~\ref{fig:1b} compare performance under different bias levels, while Figures~\ref{fig:1c} and~\ref{fig:1d} compare performance under different amounts of offline data. The results show that UCB-VI-S can suffer from negative transfer when the bias is large, whereas MIN-UCB-VI remains robust by adaptively comparing source-based and target-based estimates. Figure~\ref{fig:3} provides two additional cumulative-regret experiments, further confirming the robustness of MIN-UCB-VI.

\subsection{Best Policy Identification}
We next evaluate MAX-LCB-VI for best policy identification under the same GridWorld environment. We compare MAX-LCB-VI with three baselines: Vanilla LCB-VI~\citep{NEURIPS2021_e61eaa38}, LCB-VI-S, and H-LCB-VI. The performance metric is the suboptimality gap of the output policy.

For Figures~\ref{fig:2a} and~\ref{fig:2b}, the target success probability is fixed at $0.95$, while the source success probability ranges from $0.85$ to $0.15$, corresponding to bias levels from $0.1$ to $0.8$. The offline data is collected using a random policy. Figure~\ref{fig:2a} shows the effect of increasing the offline sample size when the bias level is fixed. MAX-LCB-VI effectively leverages larger offline datasets to reduce the suboptimality gap. In contrast, H-LCB-VI can degrade significantly when $N^{\mathrm{src}}=1000$, indicating negative transfer from biased offline data. Figure~\ref{fig:2b} further shows that as the bias level increases, LCB-VI-S and H-LCB-VI deteriorate rapidly, whereas MAX-LCB-VI remains stable and consistently outperforms Vanilla LCB-VI.

For Figures~\ref{fig:2c} and~\ref{fig:2d}, the target success probability is fixed at $0.8$. We vary the source success probability among $0.8$, $0.6$, $0.4$, and $0.2$, corresponding to bias levels $0$, $0.2$, $0.4$, and $0.6$, respectively. To evaluate the effect of offline data quality, we compare two data-collection policies: a random policy that selects actions uniformly at random, and an optimal policy $\pi^\star$ computed by value iteration in the target environment and then executed in the source environment. Figures~\ref{fig:2c} and~\ref{fig:2d} show that offline data collected by the optimal policy consistently leads to smaller suboptimality gaps than offline data collected by the random policy, across both varying sample sizes and varying bias levels.

\section{Conclusion}
This paper presents a unified algorithmic framework for hybrid RL in tabular MDPs with shifted transition dynamics, addressing both regret minimization and best policy identification. By explicitly incorporating prior bias information, our framework more effectively leverages shifted offline data and achieves improved learning efficiency. We derive both instance-dependent and instance-independent upper bounds on regret and suboptimality gap, together with matching lower bounds, demonstrating the optimality of our results. To handle the case where the bias magnitude is unknown, we further propose a model selection strategy and establish regret upper bounds for our algorithm. A promising future direction is to investigate how hybrid RL algorithms can retain performance guarantees that are no worse than purely online methods when the bias magnitude is unknown.

\acks{The corresponding author Fang Kong is supported by Guangdong Basic and Applied Basic Research Foundation (2025A1515011412) and National Natural Science Foundation of China (62506150). The authors declare no competing financial interests.}



\newpage

\appendix
\section{Notations}\label{sec:appendix_notations}

In this section, we introduce the main notations used in this paper. We list notations in~\Cref{tab:notations}.
\begin{table}[htbp]
    \centering
\setlength{\tabcolsep}{2pt}
\renewcommand{\arraystretch}{1.05}
\begin{tabularx}{\textwidth}
{|C{0.40\textwidth}|>{\raggedright\arraybackslash}X|}
\hline
         $\State,S=|\State|$ & State space and its size\\
         $\Action, A=|\Action|$ & Action space and its size\\
         $H, K$ & Horizon and Learning episodes \\
         $s, s'$ & States in $\State$ \\
         $a, a'$ & Actions in $\Action$ \\
         $h,h'$ & Horizon numbers \\
         $k$ & Index of learning episode \\
         $P^{\tar}_{s,a}$ & Transition probability of the target MDP  \\
         $P^{\src}_{s,a}$ & Transition probability of the source MDP  \\
         $ N^k_{s,a}/N^k(s,a)$ & Visitation count at $(s,a)$ before the $k$-th episode \\
          $ N^k_{s,a,h}/N^k_{h}(s,a)$ & Visitation count at $(s,a,h)$ before the $k$-th episode \\
           $P^{k,\hyb}_{s,a}$ & Virtual transition probability   \\
           $ N^\src_{s,a}/N^\src(s,a)$ & Visitation count at $(s,a)$ in the offline dataset\\
         $ N^\src_{s,a,h}/N^\src_{h}(s,a)$ & Visitation count at $(s,a,h)$ in the offline dataset\\
         $r(s,a)$ & Reward function\\
         $\pi^k$ & Policy at the $k$-th episode used in the UCB algorithm  \\
         $\hat{\pi}$ & Policy output by the LCB algorithm\\
         $\pi_h(s)$ & Action that policy $\pi$ takes at state $s$, step $h$ \\
         $V_h^\pi(s),V_h^*(s)$ & $V$-function of policy $\pi$ and of optimal policy, respectively \\
         $Q_h^\pi(s,a),Q_h^*(s,a)$ & $Q$-function of policy $\pi$ and of optimal policy, respectively \\
         $\Var_h^*/\mathbb{V}_{P^\tar}(V^*_h)$ & Variance of $V^*$ evaluated in $P^\tar$ at step $h$ \\
         $\mathbb{V}_{\hat{P}}(V_h)$ & Empirical variance of $V$ evaluated in $\hat{P}$ at step $h$ \\
         $\Delta_h(s,a)$ & Sub-optimality gap \\
                 $s_h^k,a_h^k,r^k$ & States, actions, and rewards observed in the $k$-th episode \\
        $V_h^k(s)$ & Optimistic estimate of $V_h$ before the $k$-th episode used in the UCB algorithm  \\
           $\hat{V}_h(s)$ & Pessimistic estimate of $V_h$ of the LCB algorithm\\
        $Q_h^{k,\on}(s,a), Q_h^{k,\hyb}(s,a)$ & Two optimistic estimates \\
        $\hat{P}_{s,a}^k, \hat{P}_{s,a}^{k,\hyb}$ & Two estimates of $P^\tar_{s,a}$ before the $k$-th episode \\
        $b_h^k(s,a),b_h^{k,\hyb}(s,a) $ & Bonus terms in the $k$-th episode \\
         $E_h^{k,\on}(s,a)$ & Surplus; $Q_h^{k,\on}(s,a)-r(s,a)-\mathbb E^{s'\sim P_{s,a}}[V_{h+1}^k(s')]$ \\
          $E_h^{k,\hyb}(s,a)$ & Surplus; $Q_h^{k,\hyb}(s,a)-r(s,a)-\mathbb E^{s'\sim P_{s,a}}[V_{h+1}^k(s')]$ \\
         \hline
  \end{tabularx}
\caption{Notations used throughout the paper.}
\label{tab:notations}
\end{table}


\section{Proof for Instance-Independent Regret Upper Bound} \label{sec:regret_independent_proof}

Below we provide some useful lemmas used for deriving the instance-independent regret upper bound.

\begin{lemma} [Lemma F.5 in~\citet{simchowitz2019non}] \label{lemma_variance} Let $X,Y$ be two random variables defined on the same probability space. Then
\begin{align}
    |\sqrt{\mathbb{V}[X]}-\sqrt{\mathbb{V}[Y]}|\leq \sqrt{\mathbb{E}(X-Y)^2}.\nonumber
\end{align}
\end{lemma}

\begin{lemma}[Good events~\citep{chen2025sharp}] \label{lemma:reg_indpendent_good} Let $L=\log(SAKH/\delta)$. With probability at least $1-10\delta$, the following inequalities hold for all $s,a,s',h,k$:
\begin{align}
    |\hat{P}^{k,\hyb}_{s,a}(s')-P^{k,\hyb}_{s,a}(s')|&\leq \sqrt{\frac{2\hat{P}^{k,\hyb}_{s,a}(s')L}{N^k_{s,a}+N^\src_{s,a}}}+\frac{L}{N^k_{s,a}+N^\src_{s,a}},\nonumber\\
 |\mathbb{E}_{s'\sim\hat{P}^{k,\hyb}_{s,a}}[V^*_{h+1}(s')]-\mathbb{E}_{s'\sim P^{k,\hyb}_{s,a}}[V^*_{h+1}(s')]|&\leq \sqrt{\frac{2\mathbb{V}_{s'\sim P^{k,\hyb}} V^*_{h+1}(s')L }{N^k_{s,a}+N^{\src}_{s,a}}}+\frac{HL}{N^k_{s,a}+N^{\src}_{s,a}},\nonumber\\
 \sqrt{\mathbb{V}_{s'\sim \hat{P}^{k,\hyb}_{s,a}}[V^*_{h+1}(s')]}-\sqrt{\mathbb{V}_{s'\sim P^{k,\hyb}_{s,a}}[V^*_{h+1}(s')]}&\leq H\sqrt{\frac{2L}{N^k_{s,a}+N^{\src}_{s,a}-1}},\nonumber
\end{align}
 where $P^{k,\hyb}_{s,a}:=\frac{P^{\tar}_{s,a}N^k_{s,a}+P^{\src}_{s,a}N^{\src}_{s,a}}{N^k_{s,a}+N^\src_{s,a}}$.
    
\end{lemma}

\begin{lemma}[Lemma 7 in~\citet{chen2025sharp}] \label{lemma:reg_indpendent_good_two} Let $V$ be function defined on $\mathcal{S}$. Conditioned on the success of Lemma~\ref{lemma:reg_indpendent_good},
\begin{align}
    |\mathbb{E}_{s'\sim\hat{P}^{k,\hyb}_{s,a}}[V(s')]-\mathbb{E}_{s'\sim P^{k,\hyb}_{s,a}}[V(s')]|\leq \sqrt{\frac{2S\mathbb{E}_{s'\sim P^{k,\hyb}_{s,a}}[V^2(s')]L}{N^k_{s,a}+N^\src_{s,a}}}+\frac{HSL}{N^k_{s,a}+N^\src_{s,a}},\nonumber
\end{align}
    where $P^{k,\hyb}_{s,a}:=\frac{P^{\tar}_{s,a}N^k_{s,a}+P^{\src}_{s,a}N^{\src}_{s,a}}{N^k_{s,a}+N^\src_{s,a}}$.
\end{lemma}

\begin{lemma}[Variance bounds~\citep{chen2025sharp,NEURIPS2021_e61eaa38}] \label{lemma:variance_bound} Let $\pi$ be  any fixed policy, and $\mathcal{M}$ be a MDP with its transition $P$. For any $h\in [H]$ and $s\in \mathcal{S}$, we have
    \begin{align}
        \mathbb{E}^{\pi,\mathcal{M}}\left[\sum_{h'=h}^H \mathbb{V}^*_h(s_{h'},a_{h'}) \Big|s_h=s\right]\leq H^2.\nonumber
    \end{align}
\end{lemma}

Before formally deriving the regret bound, we first provide the proof of $Q^{k,\hyb}_h(s,a)>Q^*_h(s,a)$ for all $(s,a,h,k)\in \State \times \Action \times [H] \times [K]$ if the good events mentioned above hold. The proof of the two principles, optimism and monotonicity largely follows~\citet{zhang2021reinforcement}, which rely on exploiting the properties of the  $f$ defined in the following lemma.
\begin{lemma}[Lemma 14 in~\citet{zhang2021reinforcement}]\label{lemma:f_properties}
Let $f: \Delta^{S} \times \mathbb{R}^S \times \mathbb{R} \times \mathbb{R} \rightarrow \mathbb{R}$ with $f(p,v,n,\iota) =pv+ \max\left\{\bar{c}_1\sqrt{\frac{ \mathbb{ V}(p,v) \iota }{n }} ,\bar{c}_2\frac{\iota}{n} \right\}$ with $\bar{c}_1= \frac{20}{3}$ and $\bar{c}_2 = \frac{400}{9}$.
Then  $f$ satisfies: 1) $f(p,v,n,\iota)$ is non-decreasing in $v(s)$  for all $p\in \Delta^{S}$,$\|v\|_{\infty}\leq 1$  and $n,\iota>0$;
2) $f(p,v,n,\iota)\geq pv +  2\sqrt{\frac{ \mathbb{ V}(p,v) \iota }{n }} +\frac{14\iota}{3n}$ for all $p,v$ and $n,\iota>0$. 

\end{lemma}

Based on it, we have
\begin{align}
&\quad Q^{k,\hyb}(s,a)\nonumber\\
&=\min\{ r(s,a) +\hat{P}_{s,a}^{k,\hyb}\cdot V_{h+1}^k +b_{h}^{k,\hyb}(s,a) ,H\}\nonumber\\
 &\geq \min\{ r(s,a) +\hat{P}_{s,a}^{k,\hyb}\cdot V_{h+1}^k +b_{h}^{k,\hyb}(s,a) ,Q^{*}(s,a)\}\nonumber\\
  &\geq \min\{ r(s,a) +\hat{P}_{s,a}^{k,\hyb}\cdot V_{h+1}^k+ c_1\sqrt{\frac{\mathbb{V}_{\hat{P}^{k,\hyb}}(V^k_{h+1})L}{N^k_{s,a}+N^{\src}_{s,a}}}+\frac{c_2HL}{N^k_{s,a}+N^{\src}_{s,a}} +\frac{H\nu(s,a)N^{\src}_{s,a}}{N^k_{s,a}+N^{\src}_{s,a}},Q^{*}(s,a)\}\label{eq_lemma1_1}\\
    &\geq \min\{ r(s,a) +\hat{P}_{s,a}^{k,\hyb}\cdot V_{h+1}^k+ c'_1\sqrt{\frac{\mathbb{V}_{\hat{P}^{k,\hyb}}(V^k_{h+1})L}{N^k_{s,a}+N^{\src}_{s,a}}} \lor \frac{c'_2HL}{N^k_{s,a}+N^{\src}_{s,a}} +\frac{H\nu(s,a)N^{\src}_{s,a}}{N^k_{s,a}+N^{\src}_{s,a}},Q^{*}(s,a)\}\label{eq_lemma1_2}\\
    &\geq \min\{ r(s,a) +\hat{P}_{s,a}^{k,\hyb}\cdot V_{h+1}^k+ c'_1\sqrt{\frac{\mathbb{V}_{\hat{P}^{k,\hyb}}(V^*_{h+1})L}{N^k_{s,a}+N^{\src}_{s,a}}} \lor \frac{c'_2HL}{N^k_{s,a}+N^{\src}_{s,a}} +\frac{H\nu(s,a)N^{\src}_{s,a}}{N^k_{s,a}+N^{\src}_{s,a}},Q^{*}(s,a)\}\label{eq_lemma1_3}\\
&\geq \min\{ r(s,a) +\hat{P}_{s,a}^{k,\hyb}\cdot V_{h+1}^*+2\sqrt{\frac{\mathbb{V}_{\hat{P}^{k,\hyb}}(V^*_{h+1})L}{N^k_{s,a}+N^{\src}_{s,a}}}+\frac{14HL}{3(N^k_{s,a}+N^{\src}_{s,a})} +\frac{H\nu(s,a)N^{\src}_{s,a}}{N^k_{s,a}+N^{\src}_{s,a}},Q^{*}(s,a)\}\label{eq_lemma1_4}\\
 &= \min\{ r(s,a) +P_{s,a}^{\tar}\cdot V_{h+1}^*+(\hat{P}_{s,a}^{k,\hyb}-P_{s,a}^{\tar})\cdot V_{h+1}^* +2\sqrt{\frac{\mathbb{V}_{\hat{P}^{k,\hyb}}(V^*_{h+1})L}{N^k_{s,a}+N^{\src}_{s,a}}}+\frac{14HL}{3(N^k_{s,a}+N^{\src}_{s,a})}\notag\\
     &\quad+\frac{H\nu(s,a)N^{\src}_{s,a}}{N^k_{s,a}+N^{\src}_{s,a}},Q^{*}(s,a)\}\nonumber\\
&= \min\{ r(s,a) +P_{s,a}^{\tar}\cdot V_{h+1}^*+\Bigg(\frac{\hat{P}^kN^k+\hat{P}^{\src}N^{\src}}{N^k+N^{\src}} +\frac{P^{\tar}N^{\src}}{N^k+N^\src}-\frac{P_{s,a}^\src N^\src}{N^k+n^\src}\nonumber\\
    &\quad -\frac{P^\tar(N^k+N^\src)}{N^k+N^\src}\Bigg)\cdot V_{h+1}^*+2\sqrt{\frac{\mathbb{V}_{\hat{P}^{k,\hyb}}(V^*_{h+1})L}{N^k_{s,a}+N^{\src}_{s,a}}}+\frac{14HL}{3(N^k_{s,a}+N^{\src}_{s,a})} +\frac{H\nu(s,a)N^{\src}_{s,a}}{N^k_{s,a}+N^{\src}_{s,a}},Q^{*}(s,a)\}\nonumber\\
    &\geq \min\{ r(s,a) +P_{s,a}^{\tar}\cdot V_{h+1}^*+\left(\frac{\hat{P}_{s,a}^kN_{s,a}^k+\hat{P}_{s,a}^{\src}N_{s,a}^{\src}}{N^k_{s,a}+N_{s,a}^{\src}} -\frac{P^\tar N_{s,a}^k+ P_{s,a}^\src N_{s,a}^\src}{N_{s,a}^k+N_{s,a}^\src}\right)\cdot V_{h+1}^*\nonumber\\
    &\quad +2\sqrt{\frac{\mathbb{V}_{\hat{P}^{k,\hyb}}(V^*_{h+1})L}{N^k_{s,a}+N^{\src}_{s,a}}}+\frac{14HL}{3(N^k_{s,a}+N^{\src}_{s,a})} -\frac{N_{s,a}^\src(P_{s,a}^\tar-P_{s,a}^\src)\cdot V^*_{h+1}}{N_{s,a}^k+N_{s,a}^\src} +\frac{H\nu(s,a)N^{\src}_{s,a}}{N^k_{s,a}+N^{\src}_{s,a}},Q^{*}(s,a)\}\nonumber\\
       &\geq \min\{ r(s,a) +P_{s,a}^{\tar}\cdot V_{h+1}^*+\left(\frac{\hat{P}_{s,a}^kN_{s,a}^k+\hat{P}_{s,a}^{\src}N_{s,a}^{\src}}{N_{s,a}^k+N_{s,a}^{\src}} -\frac{P_{s,a}^\tar N_{s,a}^k+ P_{s,a}^\src N_{s,a}^\src}{N_{s,a}^k+N_{s,a}^\src}\right)\cdot V_{h+1}^*\nonumber\\
    &\quad +2\sqrt{\frac{\mathbb{V}_{\hat{P}_{s,a}^{k,\hyb}}(V^*_{h+1})L}{N^k_{s,a}+N^{\src}_{s,a}}} +\frac{14HL}{3(N^k_{s,a}+N^{\src}_{s,a})}-\frac{N_{s,a}^\src(P_{s,a}^\tar-P_{s,a}^\src)\cdot V^*_{h+1}}{N_{s,a}^k+N_{s,a}^\src} \nonumber\\
    &\quad +\frac{\Vert V^*_{h+1} \Vert_{\infty}\Vert P_{s,a}^\tar-P_{s,a}^\src \Vert_1 N^{\src}_{s,a}}{N^k_{s,a}+N^{\src}_{s,a}},Q^{*}(s,a)\},\nonumber
\end{align}
where~\eqref{eq_lemma1_1} is by the definition of $b^{k,\hyb}_{h}(s,a)$.~\eqref{eq_lemma1_2} is by the  choice of $c_1,c_2$ and $c'_1,c'_2$.~\eqref{eq_lemma1_3} is by using the first property of Lemma~\ref{lemma:f_properties} and the induction of $V^k_{h+1}\geq V^*_{h+1}$.~\eqref{eq_lemma1_4} is by the second property of Lemma~\ref{lemma:f_properties}. \eqref{eq_lemma1_5} is by $(P^\tar-P^\src)\cdot V^*_{h+1}\leq \Vert V^*_{h+1} \Vert_{\infty}\Vert P^\tar-P^\src \Vert_1 $ and the definition of good events.

Then we have
\begin{align}
 Q^{k,\hyb}(s,a) &\geq \min\{ r(s,a) +P_{s,a}^{\tar}\cdot V_{h+1}^*+\left(\hat{P}_{s,a}^{k,\hyb}-P_{s,a}^{k,\hyb}\right)\cdot V_{h+1}^*\nonumber\\
&\quad+2\sqrt{\frac{\mathbb{V}_{\hat{P}_{s,a}^{k,\hyb}}(V^*_{h+1})L}{N^k_{s,a}+N^{\src}_{s,a}}}+\frac{14HL}{3(N^k_{s,a}+N^{\src}_{s,a})} ,Q^{*}(s,a)\}\label{eq_lemma1_5}\\
    &\geq   \min\{ r(s,a) +P_{s,a}^{\tar}\cdot V_{h+1}^*,Q^{*}(s,a)\}\nonumber\\
    &=Q^{*}(s,a),\nonumber
\end{align}
where~\eqref{eq_lemma1_5} is by $(P^\tar-P^\src)\cdot V^*_{h+1}\leq \Vert V^*_{h+1} \Vert_{\infty}\Vert P^\tar-P^\src \Vert_1 $ and the definition of good events.

\begin{lemma}
    \begin{align}
        b^{k,\hyb}_h(s,a)\leq \frac{2}{H}  \mathbb{E}_{\hat{P}^{k,\hyb}}(V^k_{h+1}-V^*_{h+1})^2+     2\sqrt{\frac{\mathbb{V}(V^*_{h+1})}{N^k_{s,a}+N^\src_{s,a}}}+\frac{22HL}{N^k_{s,a}+N^\src_{s,a}}+3H\nu(s,a).\nonumber
    \end{align}
    
    \begin{proof}
      Below we denote the variance term $\mathbb{V}_{s'\sim P_{s,a}}[V_{h+1}(s')]$ as $\mathbb{V}_{P}(V_{h+1})$ for simplicity, where $P$ is a probability distribution and $V$ is a function defined on the state space.  Recall that our choice of bonus in the algorithm is 
        \begin{align}
b^{k,\hyb}_h(s,a)=2\sqrt{\frac{\mathbb{V}_{\hat{P}^{k,\hyb}}(V^k_{h+1})L}{N^k_{s,a}+N^\src_{s,a}}}+\frac{10HL}{N^k_{s,a}+N^\src_{s,a}}+\frac{H\nu(s,a)N^\src_{s,a}}{N^k_{s,a}+N^\src_{s,a}}. \nonumber
        \end{align}

       We first focus on the first term. Notice that $ \sqrt{\mathbb{V}_{\hat{P}^{k,\hyb}}(V^k_{h+1})}$ can be bounded following
        \begin{align}
           &\quad \sqrt{\mathbb{V}_{\hat{P}^{k,\hyb}}(V^k_{h+1})}\nonumber\\
           &=    \sqrt{\mathbb{V}_{\hat{P}^{k,\hyb}}(V^k_{h+1})}-    \sqrt{\mathbb{V}_{\hat{P}^{k,\hyb}}(V^*_{h+1})}+    \sqrt{\mathbb{V}_{\hat{P}^{k,\hyb}}(V^*_{h+1})}-    \sqrt{\mathbb{V}_{P_{s,a}^\tar}(V^*_{h+1})}+\sqrt{\mathbb{V}_{P_{s,a}^\tar}(V^*_{h+1})}\nonumber\\
            &\leq \sqrt{\mathbb{E}_{\hat{P}^{k,\hyb}}(V^k_{h+1}-V^*_{h+1})^2}+    \sqrt{\mathbb{V}_{\hat{P}^{k,\hyb}}(V^*_{h+1})}-    \sqrt{\mathbb{V}_{P_{s,a}^\tar}(V^*_{h+1})}+\sqrt{\mathbb{V}_{P_{s,a}^\tar}(V^*_{h+1})}\label{eq_lemma2_1}\\
             &= \sqrt{\mathbb{E}_{\hat{P}^{k,\hyb}}(V^k_{h+1}-V^*_{h+1})^2}+    \sqrt{\mathbb{V}_{\hat{P}^{k,\hyb}}(V^*_{h+1})}- \sqrt{\mathbb{V}_{P_{s,a}^{k,\hyb}}(V^*_{h+1})} +\sqrt{\mathbb{V}_{P_{s,a}^{k,\hyb}}(V^*_{h+1})}\notag\\
             &\quad -   \sqrt{\mathbb{V}_{P_{s,a}^\tar}(V^*_{h+1})}+\sqrt{\mathbb{V}_{P_{s,a}^\tar}(V^*_{h+1})} \nonumber\\
        &\leq  \sqrt{\mathbb{E}_{\hat{P}^{k,\hyb}}(V^k_{h+1}-V^*_{h+1})^2}+   H\sqrt{\frac{2L}{N^K_{s,a}+N^\src_{s,a}-1}} +\sqrt{\mathbb{V}_{P_{s,a}^{k,\hyb}}(V^*_{h+1})}\nonumber\\
        &\quad-   \sqrt{\mathbb{V}_{P_{s,a}^\tar}(V^*_{h+1})}+\sqrt{\mathbb{V}_{P_{s,a}^\tar}(V^*_{h+1})}\label{eq_lemma2_2}\\
    &\leq  \sqrt{\mathbb{E}_{\hat{P}^{k,\hyb}}(V^k_{h+1}-V^*_{h+1})^2}+   H\sqrt{\frac{2L}{N^K_{s,a}+N^\src_{s,a}}} +\sqrt{|\mathbb{V}_{P_{s,a}^{k,\hyb}}(V^*_{h+1})-\mathbb{V}_{P_{s,a}^\tar}(V^*_{h+1})|}   +\sqrt{\mathbb{V}_{P_{s,a}^\tar}(V^*_{h+1})}\nonumber,
        \end{align}
where~\eqref{eq_lemma2_1} holds by Lemma~\ref{lemma_variance}, \eqref{eq_lemma2_2} is by the good events.

          Notice that
          \begin{align}
        &\quad \mathbb{V}_{P_{s,a}^{k,\hyb}}(V^*_{h+1})-\mathbb{V}_{P_{s,a}^\tar}(V^*_{h+1}) \nonumber\\
        &=      P_{s,a}^{k,\hyb}\cdot(V^*_{h+1})^2-(P_{s,a}^{k,\hyb}\cdot V^*_{h+1})^2-P_{s,a}^{\tar}\cdot(V^*_{h+1})^2+(P_{s,a}^{\tar}\cdot V^*_{h+1})^2 \nonumber\\
        &=(P_{s,a}^{k,\hyb}-P_{s,a}^{\tar})\cdot (V^*_{h+1})^2+(P_{s,a}^{\tar}\cdot V^*_{h+1}+P_{s,a}^{k,\hyb}\cdot V^*_{h+1})(P_{s,a}^{\tar}\cdot V^*_{h+1}-P_{s,a}^{k,\hyb}\cdot V^*_{h+1}) \nonumber\\
        &\leq \Vert P_{s,a}^{k,\hyb}-P_{s,a}^{\tar}\Vert_1 H^2+2H^2\Vert P_{s,a}^{k,\hyb}-P_{s,a}^{\tar}\Vert_1 \nonumber \\
        &\leq 3H^2\nu(s,a).\nonumber
          \end{align}

   Thus,
        \begin{align}
\sqrt{\frac{\mathbb{V}_{\hat{P}^{k,\hyb}}(V^k_{h+1})L}{N^k_{s,a}+N^\src_{s,a}}} &\leq  \sqrt{\frac{1}{H}\mathbb{E}_{\hat{P}^{k,\hyb}}(V^k_{h+1}-V^*_{h+1})^2 \cdot \frac{HL}{N^k_{s,a}+N^\src_{s,a}}}+\frac{2HL}{N^k_{s,a}+N^\src_{s,a}}\nonumber\\
&\quad +\sqrt{\frac{\mathbb{V}_{P_{s,a}^\tar}(V^*_{h+1})L}{N^k_{s,a}+N^\src_{s,a}}}+\sqrt{ H\nu(s,a)\cdot\frac{3HL}{N^k_{s,a}+N^\src_{s,a}}}\nonumber\\
     &\leq \frac{1}{H}\mathbb{E}_{\hat{P}^{k,\hyb}}(V^k_{h+1}-V^*_{h+1})^2+\sqrt{\frac{\mathbb{V}_{P_{s,a}^\tar}(V^*_{h+1})L}{N^k_{s,a}+N^\src_{s,a}}}+\frac{6HL}{N^k_{s,a}+N^\src_{s,a}}+H\nu(s,a),\nonumber
        \end{align}
where the inequalities hold by fully using the conclusion that $\sqrt{ab}\leq a+b, a,b\geq 0$.

          Consequently, we can upper bound the bonus term as
          \begin{align}
      b^{k,\hyb}_h(s,a)\leq  \frac{2}{H}  \mathbb{E}_{\hat{P}^{k,\hyb}}(V^k_{h+1}-V^*_{h+1})^2+     2\sqrt{\frac{\mathbb{V}_{P_{s,a}^\tar}(V^*_{h+1})L}{N^k_{s,a}+N^\src_{s,a}}}+\frac{22HL}{N^k_{s,a}+N^\src_{s,a}}+3H\nu(s,a).\nonumber
          \end{align}
 \end{proof}
\end{lemma}

\begin{lemma}\label{lemma:surplus_vanilla} Given the success of Lemma~\ref{lemma:reg_indpendent_good}, we have 
\begin{align}
      V^k_{h}-V^*_{h}\leq \mathbb{E}^{\pi^k}\left[ \sum_{h'=h}^H H \land 17H\sqrt{\frac{SL}{N^k_{s_{h'},a_{h'}}+N^\src_{s_{h'},a_{h'}}}} +2H\nu(s,a) \Bigg\vert s_h=s \right].\nonumber
\end{align}

\begin{proof}
   Below, we provide the upper bound of $V^k_{h}-V^*_{h}$.
       \begin{align}
    V^k_{h}-V^*_{h}&\leq    V^k_{h}(s)-Q^*_h(s,\pi^k(s))\nonumber\\
  &  =r(s,a)+b^{k,\hyb}_h(s,a)+\hat{P}^{k,\hyb}\cdot V^k_{h+1}-r(s,a)-P_{s,a}^{\tar}\cdot V^*_{h+1}\nonumber\\
  &=(\hat{P}_{s,a}^{k,\hyb}-P_{s,a}^{k,\hyb})\cdot (V^k_{h+1}-V^*_{h+1})+(\hat{P}^{k,\hyb}-P_{s,a}^{k,\hyb})\cdot V^*_{h+1}\nonumber\\
  &\quad +(\hat{P}_{s,a}^{k,\hyb}-P_{s,a}^{\tar})\cdot V^k_{h+1}+P_{s,a}^{\tar}\cdot (V^k_{h+1}-V^*_{h+1}) +b^{k,\hyb}_h(s,a)\nonumber\\
    &\leq  P_{s,a}^{\tar}\cdot (V^k_{h+1}-V^*_{h+1}) +(\hat{P}^{k,\hyb}-P_{s,a}^{k,\hyb})\cdot (V^k_{h+1}-V^*_{h+1})\nonumber\\
    &\quad +(\hat{P}_{s,a}^{k,\hyb}-P_{s,a}^{k,\hyb})\cdot V^*_{h+1}+H\nu(s,a)  +b^{k,\hyb}_h(s,a) ,\label{eq:lemma_gap}
       \end{align}
       where~\eqref{eq:lemma_gap} follows from  the fact that $(P^{\tar}_{s,a}-P_{s,a}^{k,\hyb})\cdot V^*_{h+1} \leq H\nu(s,a)$.

       By Lemma~\ref{lemma:reg_indpendent_good} and~\ref{lemma:reg_indpendent_good_two}, we have 
       \begin{align}
            V^k_{h}-V^*_{h} &\leq P_{s,a}^{\tar}\cdot (V^k_{h+1}-V^*_{h+1})+\sqrt{\frac{2S P_{s,a}^{k,\hyb}\cdot(V^k-V^*)^2L}{N^k_{s,a}+N^\src_{s,a}}}+\frac{SHL}{N^k_{s,a}+N^\src_{s,a}}\nonumber\\
  &\quad+\sqrt{2\frac{\mathbb{V}_{P_{s,a}^{k,\hyb}}(V^*_{h+1})L}{N^k_{s,a}+N^\src_{s,a}}}+\frac{HL}{N^k_{s,a}+N^\src_{s,a}} +H\nu(s,a)+2\sqrt{\frac{\mathbb{V}_{\hat{P}^{k,\hyb}}(V^k_{h+1})L}{N^k_{s,a}+N^\src_{s,a}}}\nonumber\\
  &\quad \quad +\frac{10HL}{N^k_{s,a}+N^\src_{s,a}}+\frac{H\nu(s,a)N^\src_{s,a}}{N^k_{s,a}+N^\src_{s,a}}\nonumber\\
  &\leq P_{s,a}^{\tar}\cdot (V^k_{h+1}-V^*_{h+1})+(2\sqrt{2}+2)H\sqrt{\frac{SL}{N^k_{s,a}+N^\src_{s,a}}}+\frac{12HSL}{N^k+N^\src}+2H\nu(s,a),\nonumber
       \end{align}

       If $SL\leq N^k_{s,a}+N^\src_{s,a}$ then we have
       \begin{align}
            V^k_{h}-V^*_{h} \leq \nonumber P_{s,a}^{\tar}\cdot (V^k_{h+1}-V^*_{h+1})+H\land 17H\sqrt{\frac{SL}{N^k_{s,a}+N^\src_{s,a}}}+2H\nu(s,a). 
       \end{align}

      If $SL> N^k_{s,a}+N^\src_{s,a}$ then the above inequality also holds since $ V^k_{h}-V^*_{h}\leq H$.  Thus we can conclude that
       \begin{align}
            V^k_{h}-V^*_{h} &\leq \mathbb{E}^{\pi^k}\left[ \sum_{h'=h}^H H \land 17H\sqrt{\frac{SL}{N_{s_{h'},a_{h'}}^k+N_{s_{h'},a_{h'}}^\src}}+2H\nu(s,a) \Bigg|s_h=s\right].\nonumber
       \end{align}

    
    \end{proof}
\end{lemma}

\begin{lemma}\label{lemma_surplus} Conditioned on success of Lemma~\ref{lemma:reg_indpendent_good}, if $a=\pi^k_h(s)$
    \begin{align}
      E^{k,\hyb}_h(s,a)&\leq \left( H\land 4\sqrt{\frac{\mathbb{V}_{P_{s,a}^{\tar}}(V^*_{h+1})L}{N^k_{s,a}+N^\src_{s,a}}}\right)+\mathbb{E}\left[\sum_{h'=h}^H 3H^2 \land \frac{3000SH^2L}{N^k_{s_{h'},a_{h'}}+N^\src_{s_{h'},a_{h'}}} \Bigg\vert (s_h,a_h)=(s,a)\right]\nonumber\\
      &\quad+40\max\{H^2\nu^2_{s,a},H\nu(s,a)\}.\nonumber
    \end{align}

    \begin{proof}     Recall the definition of the surplus $E^{k,\hyb}_h(s,a)$ is $E^{k,\hyb}_h(s,a)=Q^{k,\hyb}_h(s,a)-r(s,a)-P^{\tar}_{s,a}\cdot V^k_{h+1}$. Hence we have
  \begin{align}
          &\quad E^{k,\hyb}_h(s,a)\notag\\
          &=r(s,a)+b^{k,\hyb}_h(s,a)+\hat{P}^{k,\hyb}_{s,a}\cdot V^k_{h+1}-r(s,a)-P^{\tar}_{s,a}\cdot V^k_{h+1}\nonumber\\
      &=b^{k,\hyb}_h(s,a)+\hat{P}^{k,\hyb}_{s,a}\cdot V^k_{h+1}-P^{k,\hyb}_{s,a}\cdot V^k_{h+1}+P^{k,\hyb}_{s,a}\cdot V^k_{h+1}-P^{\tar}_{s,a}\cdot V^k_{h+1}\nonumber\\
      &\leq b^{k,\hyb}_h(s,a)+|(\hat{P}^{k,\hyb}_{s,a}-P^{k,\hyb}_{s,a})\cdot (V^k_{h+1}-V^*_{h+1})|+|(\hat{P}^{k,\hyb}_{s,a}-P^{k,\hyb}_{s,a})\cdot V^*_{h+1}|\nonumber\\
      &\quad +|(P^{k,\hyb}_{s,a}-P^{\tar}_{s,a})\cdot V^k_{h+1}|\nonumber \\
       &\leq b^{k,\hyb}_h(s,a)+ |(\hat{P}^{k,\hyb}_{s,a}-P^{k,\hyb}_{s,a})\cdot (V^k_{h+1}-V^*_{h+1})|+|(\hat{P}^{k,\hyb}_{s,a}-P^{k,\hyb}_{s,a})\cdot V^*_{h+1}|+H\nu(s,a).\nonumber
  \end{align}

  For the term $|(\hat{P}^{k,\hyb}_{s,a}-P^{k,\hyb}_{s,a})\cdot (V^k_{h+1}-V^*_{h+1})|$, we have
  \begin{align}
      |(\hat{P}^{k,\hyb}_{s,a}-P^{k,\hyb}_{s,a})\cdot (V^k_{h+1}-V^*_{h+1})|&\leq  \sqrt{\frac{2SP^{k,\hyb}_{s,a}\cdot (V^k_{h+1}-V^*_{h+1})^2 L}{N^k_{s,a}+N^{\src}_{s,a}}}+ \frac{SHL}{N^k_{s,a}+N^{\src}_{s,a}}\nonumber\\
      &\leq \frac{P^{k,\hyb}\cdot(V^k_{h+1}-V^*_{h+1})^2}{H}+ \frac{3SHL}{N^k_{s,a}+N^{\src}_{s,a}}.\nonumber
  \end{align}

  For the term $|(\hat{P}^{k,\hyb}_{s,a}-P^{k,\hyb}_{s,a})\cdot V^*_{h+1}|$ we have
  \begin{align}
      |(\hat{P}^{k,\hyb}_{s,a}-P^{k,\hyb}_{s,a})\cdot V^*_{h+1}|&\leq \sqrt{\frac{2\mathbb{V}_{P^{k,\hyb}}(V^*_{h+1})}{N^k_{s,a}+N^{\src}_{s,a}}}+\frac{HL}{N^k_{s,a}+N^{\src}_{s,a}} \nonumber\\
      &\leq  \sqrt{\frac{2\mathbb{V}_{P^{\tar}}(V^*_{h+1})}{N^k_{s,a}+N^{\src}_{s,a}}}+\frac{HL}{N^k_{s,a}+N^{\src}_{s,a}}  +\sqrt{\frac{6H^2\nu(s,a)L}{N^k_{s,a}+N^{\src}_{s,a}}}\nonumber\\
      &\leq    \sqrt{\frac{2\mathbb{V}_{P^{\tar}}(V^*_{h+1})}{N^k_{s,a}+N^{\src}_{s,a}}}+\frac{7HL}{N^k_{s,a}+N^{\src}_{s,a}}  +H\nu(s,a). \nonumber
  \end{align}

   Consequently, we can upper bound $E^k_h(s,a)$ as
    \begin{align}
       E^k_h(s,a)&\leq (2+\sqrt{2})\sqrt{\frac{\mathbb{V}_{P^\tar}(V^*_{h+1})L}{N^k_{s,a}+N^{\src}_{s,a}}}+\frac{1}{H} P^{k,\hyb}\cdot(V^k_{h+1}-V^*_{h+1})^2  \nonumber\\
    &\quad  +\frac{2}{H} \hat{P}^{k,\hyb}\cdot(V^k_{h+1}-V^*_{h+1})^2 +\frac{32SHL}{N^k_{s,a}+N^\src_{s,a}}+5H\nu(s,a)\nonumber\\
    &\leq (2+\sqrt{2})\sqrt{\frac{\mathbb{V}_{P^\tar}(V^*_{h+1})L}{N^k_{s,a}+N^{\src}_{s,a}}}+\frac{5}{H} P^{k,\hyb}\cdot(V^k_{h+1}-V^*_{h+1})^2   +\frac{36SHL}{N^k_{s,a}+N^\src_{s,a}}+5H\nu(s,a)\nonumber\\
     &\leq (2+\sqrt{2})\sqrt{\frac{\mathbb{V}_{P^\tar}(V^*_{h+1})L}{N^k_{s,a}+N^{\src}_{s,a}}}+\frac{5}{H} P^{\tar}\cdot(V^k_{h+1}-V^*_{h+1})^2   +\frac{36SHL}{N^k_{s,a}+N^\src_{s,a}}+10H\nu(s,a).\nonumber
    \end{align}.

    By Lemma~\ref{lemma:surplus_vanilla}, 
    \begin{align}
        (V^k_{h+1}-V^*_{h+1})^2\leq \mathbb{E}^{\pi^k}\left[ \sum_{h'=h}^H H^3 \land \frac{600SH^3L}{N^k_{s_{h'},a_{h'}}+N^{\src}_{s_{h'},a_{h'}}}+8H^2\nu^2_{s,a}\Bigg|(s_h=s,a_h=a) \right].\nonumber
    \end{align}

    Hence,
    \begin{align}
     &\quad E^{k,\hyb}_h(s,a) \notag\\
     &\leq \left( H\land 4\sqrt{\frac{\mathbb{V}_{P_{s,a}^{\tar}}(V^*_{h+1})L}{N^k_{s,a}+N^\src_{s,a}}}\right)+\mathbb{E}^{\pi^k}\left[\sum_{h'=h}^H 3H^2 \land \frac{3000SH^2L}{N^k_{s_{h'},a_{h'}}+N^\src_{s_{h'},a_{h'}}} \Bigg\vert (s_h,a_h)=(s,a)\right]\nonumber\\
      &\quad +40\max\{H^2\nu^2_{s,a},H\nu(s,a)\}.\nonumber
    \end{align}

    
    \end{proof}

\end{lemma}

\begin{proof} Now we begin to prove the regret upper bound.
\begin{align}
     &\quad \mathrm{Regret}(K)\nonumber\\
     &=\sum_{k=1}^K(V_1^{*}-V_1^{\pi^k})\nonumber\\
      &\leq \sum_{k=1}^K(V^{k}_1-V_1^{\pi^k})\label{eq:optimism}\\
      &\leq \sum_{k=1}^K \sum_{h=1}^H \mathbb{E}^{\pi^k}\left[ \min\{Q^k_h(s,a),Q^{k,\hyb}_h(s,a)\}-r(s,a)-P^{\tar}_{s,a}\cdot V^k_{h+1}\right]\label{eq:ucb_v_k}\\
      &=\sum_{k=1}^K  \sum_{h=1}^H \mathbb{E}^{\pi^k}\left[\min\{b^k_h(s,a)+(\hat{P}^k-P^\tar)\cdot V^k_{h+1},b^{k,\hyb}_h(s,a)+(\hat{P}^{k,\hyb}-P^\tar)\cdot V^k_{h+1}\}\right]\nonumber\\
      &\leq\sum_{k,h}\mathbb{E}^{\pi^k}\left[b^k_h(s,a)+(\hat{P}^k-P^\tar)\cdot V^k_{h+1}\right] \land \sum_{k,h} \mathbb{E}^{\pi^k}\left[b^{k,\hyb}_h(s,a)+(\hat{P}^{k,\hyb}-P^\tar)\cdot V^k_{h+1}\right],\label{eq:ucb_min}
\end{align}
where~\eqref{eq:optimism} holds by the optimism and~\eqref{eq:ucb_v_k} holds by our construction of $V^k_1$.

We then upper bound the second term in~\eqref{eq:ucb_min} and denote $\max_{s,a}40\max\{ H^3\nu^2_{s,a}, H^2\nu(s,a)\}$ as $\Tilde{\nu}_{\max}$. By Lemma~\ref{lemma_surplus}, we have
\begin{align}
        &   \quad \sum_{k,h} \mathbb{E}^{\pi^k}\left[b^{k,\hyb}_h(s,a)+(\hat{P}^{k,\hyb}-P^\tar)\cdot V^k_{h+1}\right]\nonumber\\
   &= \sum_{k,h} \mathbb{E}^{\pi^k}\left[2\sqrt{\frac{\mathbb{V}_{\hat{P}^{k,\hyb}}(V^k_{h+1})L}{N^k_{S^k_h,A^k_h}+N^{\src}_{S^k_h,A^k_h}}}+\frac{10HL+H\nu(s,a)N^{\src}_{S^k_h,A^k_h}}{N^k_{S^k_h,A^k_h}+N^{\src}_{S^k_h,A^k_h}}+(\hat{P}^{k,\hyb}-P^\tar)\cdot V^k_{h+1}\right]\nonumber\\
    &\leq KH\Tilde{\nu}_{\max} +\sum_{k,h} \mathbb{E}^{\pi^k} \left[ 4\sqrt{\frac{\mathbb{V}_{P^\tar}(V^*_{h+1})L}{N^k_{S^k_h,A^k_h}+N^\src_{S^k_h,A^k_h}}}\right] + 3000\sum_{k,h} \frac{SH^3L}{N^k_{S^k_h,A^k_h}}\nonumber\\
        &= KH\Tilde{\nu}_{\max} +\sum_{k,h} \mathbb{E}^{\pi^k} \left[ 4\sqrt{\frac{\mathbb{V}_{P^\tar}(V^*_{h+1})L}{N^k_{S^k_h,A^k_h}+N^\src_{S^k_h,A^k_h}}}\right] + 3000\sum_{s,a}\sum_{n=1}^{N^K_{s,a}} \frac{SH^3L}{n}\nonumber\\
     &\leq KH\Tilde{\nu}_{\max} +4\sqrt{\sum_{k}\mathbb{E}\sum_h\mathbb{V}(V^*_{h+1})L}\sqrt{\sum_{k,h}\frac{1}{N^k_{S^k_h,A^k_h}+N^\src_{S^k_h,A^k_h}}}+3000S^2AH^3L^2\nonumber\\
     &\leq KH\Tilde{\nu}_{\max} +4\sqrt{KH^2L}\sqrt{\sum_{k,h}\frac{1}{N^k_{S^k_h,A^k_h}+N^\src_{S^k_h,A^k_h}}}+3000S^2AH^3L^2 \label{eq:ucb_variance}\\
            &\leq KH\Tilde{\nu}_{\max} +4\sqrt{KH^2L}\sqrt{\sum_{s,a}\sum_{n=1}^{N^K_{s,a}}\frac{1}{n+N^\src_{s,a}}}+3000S^2AH^3L^2,\nonumber
\end{align}
where~\eqref{eq:ucb_variance} holds by Lemma~\ref{lemma:variance_bound}.


Following the similar derivation on the multi-armed bandit setting in~\citet{DBLP:conf/icml/CheungL24}, we claim that 
\begin{align}
\sqrt{\sum_{s,a}\sum_{n=1}^{N^K_{s,a}}\frac{1}{n+N^\src_{s,a}}} &\leq \sqrt{ \sum_{s,a}\sum_{n_{s,a}=1}^{\tau^{*}-N^\src_{s,a}} \frac{1}{n_{s,a}+N^\src_{s,a}}}.\nonumber
\end{align}

Given this claim, for each $(s,a) \in \State \times \Action$, we have
\begin{align}
\sqrt{ \sum_{s,a}\sum_{n_{s,a}=1}^{\tau^{*}-N^\src_{s,a}} \frac{1}{n_{s,a}+N^\src_{s,a}}} \leq \sqrt{\sum_{s,a}\frac{\tau_*-N^\src_{s,a}}{\tau_*}\sum_{n=1}^{\tau_*} \frac{1}{n} }\leq \sqrt{L\sum_{s,a}\frac{N^K_{s,a}+1}{\tau_*}}\leq \sqrt{\frac{HKL}{\tau_*}}.\nonumber
 \end{align}


To sum up, we have
\begin{align}
     \mathrm{Regret}(K)&\leq \Tilde{\mathcal{O}}\left( \min\left\{H\sqrt{SAK}, \left(KH\Tilde{\nu}_{\max}+KH\sqrt{H/\tau_*}\right)\right\} \right).\nonumber
\end{align}
\end{proof}


\section{Proof for Instance-Dependent Regret Upper Bound} \label{sec:regret_dependent_proof}

In this subsection, we provide the proof for the instance-dependent regret upper bound.

\begin{proof}
We first introduce some quantities frequently used in the following proofs.
\begin{align}
\Var^*(s,a)
&:=
\max_{h\in[H]}
\Var_{s'\sim P^{\tar}(\cdot\mid s,a)}
\left[V_{h+1}^*(s')\right],
\notag\\
\Delta(s,a)
&:=
\min_{h: \Delta_h(s,a)>0}\Delta_h(s,a),
\qquad
\Delta_h(s,a):=V_h^*(s)-Q_h^*(s,a).\notag
\end{align}

The summation below is over state-action pairs with $\Delta(s,a)>0$.  To simplify the analysis, we assume that for each $(s,a)$, the action is either optimal at all stages or suboptimal at all stages. Pairs
with zero step-independent gap can be absorbed into the standard lower-order
terms or handled by a global positive-gap convention. The pooled online count before episode $k$ is $N^k(s,a):=
\sum_{\tau<k}\sum_{h=1}^H
\1\{(s_h^\tau,a_h^\tau)=(s,a)\}$ and the pooled offline source count is denoted by $N^{\src}(s,a)$.  This
pooled version uses the time-homogeneous transition setting.  The reward may
still be step-dependent.

We then define the actual MIN-UCB-VI surplus by
\begin{align*}
 E_h^k(s,a)
&:=
 Q_h^k(s,a)
-
\left[
r(s,a)+P^{\tar}(\cdot\mid s,a)V_{h+1}^k
\right].
\end{align*}

The online and hybrid branch surpluses are
\begin{align*}
E_{h}^{k,\on}(s,a)
&:=
Q_{h}^{k,\on}(s,a)
-
\left[
r(s,a)+P^{\tar}(\cdot\mid s,a) V_{h+1}^k
\right],
\\
E^{k,\hyb}_h(s,a)
&:=
Q^{k,\hyb}_h(s,a)
-
\left[
r(s,a)+P^{\tar}(\cdot\mid s,a)V_{h+1}^k
\right].
\end{align*}
By the pointwise MIN construction,
\begin{align*}
 E_h^k(s,a)
&=
\min\{E_{h}^{k,\on}(s,a),E^{k,\hyb}_h(s,a)\}.
\end{align*}

Based on the analysis provided in the previous subsection, we know the following facts hold: On the good event, both branches are optimistic with respect to the target
Bellman backup and their leading surplus terms satisfy, for all $k,h,s,a$,
\begin{align}
E_{h}^{k,\on}(s,a)
&\le
c_b
\sqrt{
\frac{\Var^*(s,a)L}
{N^k(s,a)\vee 1}
}
+
\loth,
\\
E^{k,\hyb}_h(s,a)
&\le
c_b
\sqrt{
\frac{\Var^*(s,a)L}
{\bigl(N^k(s,a)+N^{\src}(s,a)\bigr)\vee 1}
}
+
\tilde{\nu}(s,a)
+
\loth . \notag
\end{align}

Below we define the clipping operator
\begin{align}
\clip[x\mid\epsilon]
&:=x\1\{x\ge \epsilon\},
\qquad
C_H:=4H. \notag
\end{align}

In this subsection. we use the step-independent clipping threshold $\epsilon(s,a):=\frac{\Delta(s,a)}{C_H}$.

Below we consider a fixed episode $k$ and hide the superscript $k$.  Let $\pi$ be the greedy
policy induced by $ Q$.  We define $D_h(s):= V_h(s)-V_h^\pi(s)$.

At the realized state-action pair $(s_h,a_h)$,
\begin{align}
D_h(s_h)
&=
 Q_h(s_h,a_h)
-
\left[
r(s_h,a_h)+P^{\tar}(\cdot\mid s_h,a_h)V_{h+1}^{\pi}
\right]
\\
&=
 E_h(s_h,a_h)
+
P^{\tar}(\cdot\mid s_h,a_h)( V_{h+1}-V_{h+1}^{\pi})
\\
&=
 E_h(s_h,a_h)+P^{\tar}D_{h+1}.\notag
\end{align}

Unrolling the recursion and using optimism gives
\begin{align}
\Reg(k)
&:=
V_1^*(s_1)-V_1^{\pi}(s_1)
\le
\E_{\pi}
\left[
\sum_{h=1}^H  E_h(s_h,a_h)\notag
\right].
\end{align}

Then we split each surplus into a clipped part and a small part:
\begin{align*}
 E_h(s_h,a_h)
&=
\clip\left[
 E_h(s_h,a_h)
\ \middle|\
\frac{\Delta(s_h,a_h)}{C_H}
\right]+
 E_h(s_h,a_h)
\1\left\{
 E_h(s_h,a_h)
<
\frac{\Delta(s_h,a_h)}{C_H}
\right\}.
\end{align*}

For the executed action $a_h=\pi_h(s_h)$, we define
$d_h(s):=V_h^*(s)-V_h^\pi(s)$.  Then
\begin{align}
d_h(s_h)
&=
\Delta_h(s_h,a_h)
+
P^{\tar}(V_{h+1}^*-V_{h+1}^{\pi})
\ge
\Delta_h(s_h,a_h)
\ge
\Delta(s_h,a_h).\notag
\end{align}

Therefore,
\begin{align}
 E_h(s_h,a_h)
\1\left\{
 E_h(s_h,a_h)<\frac{\Delta(s_h,a_h)}{C_H}
\right\}
&\le
\frac{d_h(s_h)}{C_H}.\notag
\end{align}

We further have
\begin{align*}
\E_{\pi}\sum_{h=1}^H d_h(s_h)
&\le
H\Reg(k) .
\end{align*}

Since $C_H=4H$, the total small part is at most $\Reg(k)/4$.  Absorbing it into
the left-hand side yields
\begin{align*}
\Reg_k
&\lesssim
\E_{\pi}\sum_{h=1}^H
\clip\left[
\bar E_h(s_h,a_h)
\ \middle|\
\frac{\Delta(s_h,a_h)}{C_H}
\right]
+
\loth .
\end{align*}

Summing over episodes further gives
\begin{align*}
\Reg(K)
&\lesssim
\sum_{k=1}^K\sum_{h=1}^H
\clip\left[
\bar E_h^k(s_h^k,a_h^k)
\ \middle|\
\frac{\Delta(s_h^k,a_h^k)}{C_H}
\right]
+
\loth .
\end{align*}

Below we define $B(s,a):=c_b^2\Var^*(s,a)L$. Consider a  given state-action pair $(s,a)$ and index all online visits to this pair across
all stages by $n=1,2,\ldots,N^K(s,a)$.  The online branch has leading bonus
$b_n^{\on}(s,a):=
\sqrt{\frac{B(s,a)}{n}}$.

Consider the clipped online charge
\begin{align*}
S_{s,a}^{\on}
&:=
\sum_{n=1}^{N^K(s,a)}
\clip\left[
b_n^{\on}(s,a)
\ \middle|\
\frac{\Delta(s,a)}{C_H}
\right].
\end{align*}

If the $n$-th term is nonzero, then we have $\sqrt{\frac{B(s,a)}{n}}
\ge
\frac{\Delta(s,a)}{C_H}$, and hence $n
\le
\frac{C_H^2B(s,a)}{\Delta^2(s,a)}$.

We then define
\begin{align*}
N_{\bud}^{\on}(s,a)
&:=
\left\lceil
\frac{C_H^2B(s,a)}{\Delta^2(s,a)}
\right\rceil.
\end{align*}

Only the first $N_{\bud}^{\on}(s,a)$ visits can have nonzero online clipped
charge.  Therefore,
\begin{align}
S_{s,a}^{\on}
&\le
\sum_{n=1}^{N_{\bud}^{\on}(s,a)}
\sqrt{\frac{B(s,a)}{n}}&\le
2\sqrt{B(s,a)N_{\bud}^{\on}(s,a)}
&\lesssim
\Delta(s,a)N_{\bud}^{\on}(s,a)
\lesssim
\frac{H^2\Var^*(s,a)L}{\Delta(s,a)},
\label{eq:si_online_final_scale}
\end{align}
where the last inequality uses $C_H=4H$.

Similarly, the  bonus  of the hybrid branch satisfies
\begin{align*}
b_n^{\hyb}(s,a)
\leq 
\sqrt{\frac{B(s,a)}{n+N^{\src}(s,a)}}
+
\tilde{\nu}(s,a).
\end{align*}

Consider
\begin{align*}
S_{s,a}^{\hyb}
&:=
\sum_{n=1}^{N^K(s,a)}
\clip\left[
b_n^{\hyb}(s,a)
\ \middle|\
\frac{\Delta(s,a)}{C_H}
\right].
\end{align*}

If the $n$-th term is nonzero and $\Delta(s,a)>C_H \tilde{\nu}(s,a)$, then
\begin{align*}
\sqrt{\frac{B(s,a)}{n+N^{\src}(s,a)}}
+
\tilde{\nu}(s,a)
&\ge
\frac{\Delta(s,a)}{C_H},
\end{align*}

and therefore
\begin{align*}
\sqrt{\frac{B(s,a)}{n+N^{\src}(s,a)}}
&\ge
\frac{\Delta(s,a)-C_H \tilde{\nu}(s,a)}{C_H}.
\end{align*}

Thus only
\begin{align*}
N_{\bud}^{\hyb}(s,a)
&:=
\left[
\frac{C_H^2B(s,a)}
{(\Delta(s,a)-C_H \tilde{\nu}(s,a))^2}
-
N^{\src}(s,a)
\right]_+
\end{align*}
online visits can produce nonzero hybrid clipped charge.  If
$\Delta(s,a)\le C_H \tilde{\nu}(s,a)$, the hybrid branch gives no useful subtraction, and
the online branch remains valid.

We now upper-bound the clipped hybrid sum.  Using
$\sum_{n=1}^{M}1/\sqrt{N+n}\le 2(\sqrt{N+M}-\sqrt N)$,
\begin{align*}
\sum_{n=1}^{N_{\bud}^{\hyb}(s,a)}
\sqrt{
\frac{B(s,a)}
{n+N^{\src}(s,a)}
}
&\le
2\sqrt{B(s,a)}
\left(
\sqrt{N^{\src}(s,a)+N_{\bud}^{\hyb}(s,a)}
-
\sqrt{N^{\src}(s,a)}
\right)
\\
&\quad\lesssim
\left(
\frac{\Delta(s,a)-C_H \tilde{\nu}(s,a)}{C_H}
\right)
N_{\bud}^{\hyb}(s,a).
\end{align*}

The bias part contributes
\begin{align*}
\tilde{\nu}(s,a)N_{\bud}^{\hyb}(s,a)
&=
\frac{C_H \tilde{\nu}(s,a)}{C_H}N_{\bud}^{\hyb}(s,a).
\end{align*}

We combine the two displays and then we have
\begin{align}
S_{s,a}^{\hyb}
&\lesssim
\Delta(s,a)N_{\bud}^{\hyb}(s,a)
\lesssim
\Delta(s,a)
\left[
\frac{C_H^2B(s,a)}
{(\Delta(s,a)-C_H \tilde{\nu}(s,a))^2}
-
N^{\src}(s,a)
\right]_+ .
\label{eq:si_hybrid_final}
\end{align}

Since $\bar E_h^k(s,a)=\min\{E_{h,\on}^k(s,a),E_{h,\hyb}^k(s,a)\}$ and all
surpluses are nonnegative on the good event, the cumulative clipped MIN charge
for each $(s,a)$ is bounded by the smaller of the online and hybrid charged
budgets.  Combining Eq.~\eqref{eq:si_online_final_scale} and
Eq.~\eqref{eq:si_hybrid_final}, we get
\begin{align*}
\Reg(K)
&\lesssim
\sum_{s,a}
\Delta(s,a)
\min\left\{
\frac{C_H^2B(s,a)}{\Delta^2(s,a)},
\left[
\frac{C_H^2B(s,a)}
{(\Delta(s,a)-C_H \tilde{\nu}(s,a))^2}
-
N^{\src}(s,a)
\right]_+
\right\}
+
\loth .
\end{align*}

We use the elementary inequality: for $A>0$, $\Delta>0$, $N\ge 0$, and
$c\nu\ge 0$,
\begin{align}
&\min\left\{
\frac{A}{\Delta^2},
\left[
\frac{A}{(\Delta-c\nu)^2}-N
\right]_+
\right\}
\le
\left[
\frac{2A}{\Delta^2}
-
N\left(1-\frac{c\nu}{\Delta}\right)_+^2
\right]_+ .
\label{eq:si_algebra_inequality}
\end{align}
Applying Eq.~\eqref{eq:si_algebra_inequality} with
\begin{align*}
A&=C_H^2B(s,a),
\qquad
\Delta=\Delta(s,a),
\qquad
c\nu=C_H \tilde{\nu}(s,a),
\qquad
N=N^{\src}(s,a),
\end{align*}
yields
\begin{align}
\Reg(K)
&\lesssim
\sum_{s,a}
\left[
\frac{2C_H^2B(s,a)}{\Delta(s,a)}
-
\Delta(s,a)N^{\src}(s,a)
\left(
1-\frac{C_H \tilde{\nu}(s,a)}{\Delta(s,a)}
\right)_+^2
\right]_+
+
\loth .
\label{eq:si_regret_form_bound}
\end{align}

Finally, we use $C_H=4H$, $B(s,a)=c_b^2\Var^*(s,a)L$, and
$C_H \tilde{\nu}(s,a)\le C_{\nu}H\tilde{\nu}(s,a)$  to obtain the below result
\begin{align}
\Reg(K)
&\lesssim
\sum_{s,a}
\left[
\frac{H^2\Var^*(s,a)L}{\Delta(s,a)}
-
\Delta(s,a)N^{\src}(s,a)
\left(
1-\frac{H\tilde{\nu}(s,a)}{\Delta(s,a)}
\right)_+^2
\right]_+
+
\loth.
\notag
\end{align}

\end{proof}

\section{Proof for Instance-Independent Sub-optimality Gap Upper Bound}\label{sec:bpi_independent_proof}

\begin{lemma}[Binomial concentration~\citep{NEURIPS2021_e61eaa38}]  \label{lemma:bino} Suppose $N\sim \mathrm{Bin}(n,p)$ where $n\geq 1$ and $p\in[0,1]$. Then with probability at least $1-\delta$, we have
\begin{align}
    \frac{p}{N}\leq \frac{8\log(1/\delta)}{n}.\nonumber
\end{align}
    
\end{lemma}

\begin{lemma}[Inequality (131) in~\citet{li2024settling}] \label{lemma:variance_offline}
    \begin{align}
        \sum_{j=h}^H\mathbb{E}^{\pi^*}[\mathbb{V}(\hat{V}_{j+1})]\leq 4H^2+4H\sum_{j=h}^H\sum_{s,a}d^{\pi^*}_h(s,a)\left[\min\left\{b^{K}_h(s,a),b^{K,\hyb}_h(s,a)\right\}\right].\nonumber
    \end{align}
\end{lemma}

\begin{lemma}[Good events~\citep{li2024settling,NEURIPS2021_e61eaa38}] Let $L=\log(SAKH/\delta)$, with probability at least $1-\delta$, we have
\begin{align}
    |(\hat{P}^{K,\hyb}-P^{K,\hyb})\cdot (\hat{V}_{h+1})|&\leq \sqrt{\frac{48\mathbb{V}_{\hat{P}^{K,\hyb}}(\hat{V}_{h+1})}{N^K_{s,a}+N^\src_{s,a}}L}+\frac{48HL}{N^K_{s,a}+N^\src_{s,a}},\nonumber\\
    \mathbb{V}_{\hat{P}^{K,\hyb}}(\hat{V}_{h+1})&\leq 2\mathbb{V}_{P^{K,\hyb}}(\hat{V}_{h+1}) +\frac{5H^2L}{3(N^K_{s,a}+N^\src_{s,a})}\nonumber\\
   \frac{1}{N^\src_{h,s,a}} &\leq \frac{8L}{N d^{\pi^b}_h(s,a)},\nonumber
\end{align}
where $P^{K,\hyb}_{s,a}:=\frac{P^\tar N^K_{s,a}+P^\src N^\src_{s,a}}{N^K_{s,a}+N^\src_{s,a}}$.
    
\end{lemma}

\begin{lemma}[Monotonicity for MAX-VI-LCB]
    \begin{align}
        \hat{V}_h(s)\leq V^{\hat{\pi}}_h(s)\leq V^*_h(s).\nonumber
    \end{align}
    
    \begin{proof}
    We first note that $V^{\hat{\pi}}_h(s)\leq V^*_h(s)$ holds by the definition of $V^*_h(s)$. Thus below we aim to prove that $  \hat{V}_h(s)\leq V^{\hat{\pi}}_h(s)$.
        \begin{align}
         V^{\hat{\pi}}_h(s)-      \hat{V}_h(s)=r(s,a)+P^{\tar}\cdot V^{\hat{\pi}}_{h+1}-\max\{\hat{Q}^{K}_h(s,a),\hat{Q}^{K,\hyb}_h(s,a)\}.\nonumber
        \end{align}

        Below we consider the situation that $\hat{Q}^{K,\hyb}_h(s,a)\geq \hat{Q}^{K}_h(s,a)$.
        \begin{align}
          &\quad  V^{\hat{\pi}}_h(s)-      \hat{V}_h(s)\nonumber\\
          &=P^{\tar}\cdot V^{\hat{\pi}}_{h+1}-\hat{P}^{K,\hyb}\cdot \hat{V}_{h+1}+b^{K,\hyb}_h(s,a)\nonumber\\
             &= P^{\tar}\cdot(V^{\hat{\pi}}_{h+1}-\hat{V}_{h+1})+(P^{\tar}-\hat{P}^{K,\hyb})\cdot \hat{V}_{h+1}+b^{K,\hyb}_h(s,a)\nonumber\\
              &\geq (P^{\tar}-\hat{P}^{K,\hyb})\cdot \hat{V}_{h+1}+b^{K,\hyb}_h(s,a)\nonumber\\
              &=(P^{\tar}-\hat{P}^{K,\hyb})\cdot \hat{V}_{h+1}+ c_1\sqrt{\frac{\mathbb{V}_{\hat{P}^{K,\hyb}}(\hat{V}_{h+1})}{N^K_{s,a}+N^{\src}_{s,a}}}+\frac{c_2HL}{N^K_{s,a}+N^{\src}_{s,a}}+\frac{H\nu(s,a)N^{\src}_{s,a}} {N^K_{s,a}+N^{\src}_{s,a}}\nonumber\\
 & =\left(\frac{P^\tar N^K_{s,a}+P^\src N^\src_{s,a}}{N^K+N^\src}-\frac{\hat{P}^\tar N^K_{s,a}+\hat{P}^\src N^\src_{s,a}}{N^K+N^\src}\right)\cdot \hat{V}_{h+1}+c_1\sqrt{\frac{\mathbb{V}_{\hat{P}^{K,\hyb}}(\hat{V}_{h+1})}{N^K_{s,a}+N^{\src}_{s,a}}}+\frac{c_2HL}{N^K_{s,a}+N^{\src}_{s,a}}\nonumber\\
 &\quad -\frac{N^\src(P^\tar-P^\src)\cdot \hat{V}_{h+1}}{N^K+N^\src}+\frac{H\nu(s,a)N^{\src}_{s,a}} {N^K_{s,a}+N^{\src}_{s,a}}\nonumber\\
  & \geq \left(\frac{P^\tar N^K_{s,a}+P^\src N^\src_{s,a}}{N^K+N^\src}-\frac{\hat{P}^\tar N^K_{s,a}+\hat{P}^\src N^\src_{s,a}}{N^K+N^\src}\right)\cdot \hat{V}_{h+1}+c_1\sqrt{\frac{\mathbb{V}_{\hat{P}^{K,\hyb}}(\hat{V}_{h+1})}{N^K_{s,a}+N^{\src}_{s,a}}}+\frac{c_2HL}{N^K_{s,a}+N^{\src}_{s,a}}\nonumber\\
 &\quad -\frac{N^\src}{N^K+N^\src}(P^\tar-P^\src)\cdot \hat{V}_{h+1}+\frac{\Vert\hat{V}_{h+1} \Vert_{\infty} \Vert P^\tar-P^\src \Vert_1} {N^K_{s,a}+N^{\src}_{s,a}}\nonumber\\
              &\geq 0,\nonumber
        \end{align}
      where  the last inequality holds by $\Vert\hat{V}_{h+1} \Vert_{\infty} \Vert P^\tar-P^\src \Vert_1\geq (P^\tar-P^\src)\cdot \hat{V}_{h+1}$.
    \end{proof}
\end{lemma}

\begin{lemma}[Performance decomposition for VI-LCB]
\begin{align}
    \sum_{s\in\mathcal{S}}d^{\pi^*}_h(s)(V^*_h(s)-\hat{V}_h(s))\leq 2\sum_{h'=h}^H\sum_{s,a \in \State \times \Action}d^{\pi^*}_{h'}(s,a)\min\left\{b^K_{h'}(s,a), b^{K,\hyb}_{h'}(s,a)+H\nu(s,a)\right\}.\nonumber
\end{align}

 \begin{proof}
 Based on the definition of $\hat{V}_h(s)$ we have
       \begin{align}
           &\quad \sum_{s\in\mathcal{S}}d^{\pi^*}_h(s)(V^*_h(s)-\hat{V}_h(s))\nonumber\\
           &\leq \sum_{s\in\mathcal{S}}d^{\pi^*}_h(s)(V^*_h(s)-\max_a \max\{\hat{Q}^K_h(s,a),\hat{Q}^{K,\hyb}_h(s,a)\})\nonumber\\
           &\leq  \sum_{s\in\mathcal{S}}d^{\pi^*}_h(s)(V^*_h(s)- \max\{\hat{Q}^K_h(s,\pi^*(s)),\hat{Q}^{K,\hyb}_h(s,\pi^*(s))\})\nonumber\\
           &=\sum_{h'=h}^H\sum_{s,a}d^{\pi^*}_{h'}(s,a)r(s,a)-\sum_{s}d^{\pi^*}_{h}(s)\max\{\hat{Q}^K_h(s_1,\pi^*(s)),\hat{Q}^{K,\hyb}_h(s_1,\pi^*(s))\}.\nonumber
       \end{align}

       Then we have 
       \begin{align}
             &\quad \sum_{s\in\mathcal{S}}d^{\pi^*}_h(s)(V^*_h(s)-\hat{V}_h(s))\nonumber\\&\leq \sum_{h'=h}^H\sum_{s,a} d^{\pi^*}_{h'}(s,a) r(s,a)-\sum_{h'=h}^H\sum_{s,a,s'}\Bigg(d^{\pi^*}_{h'}(s,a) \max\{\hat{Q}^K_{h'}(s,\pi^*(s)),\hat{Q}^{K,\hyb}_{h'}(s,\pi^*(s))\}\nonumber\\
             &\quad -d^{\pi^*}_{h'+1}(s)\hat{V}_{h'+1}(s')\Bigg)\notag \\
           &= \sum_{h'=h}^H\sum_{s,a} d^{\pi^*}_{h'}(s,a) r(s,a)-\sum_{h'=h}^H\sum_{s,a,s'}d^{\pi^*}_{h'}(s,a) \Bigg( \max\{\hat{Q}^K_{h'}(s,\pi^*(s)),\hat{Q}^{K,\hyb}_{h'}(s,\pi^*(s))\}\notag\\
           &\quad-[P^{\tar}\cdot\hat{V}_{h'+1}](s,a)\Bigg)\nonumber\\
           &=\sum_{h'}\sum_{s,a}d^{\pi^*}_{h'}(s,a)\left(r(s,a)+[P^{\tar}\cdot\hat{V}_{h'+1}](s,a)-\max\{\hat{Q}^K_{h'}(s,\pi^*(s)),\hat{Q}^{K,\hyb}_{h'}(s,\pi^*(s))\}\right)\nonumber\\
        &=\sum_{h'}\sum_{s,a}d^{\pi^*}_{h'}(s,a)\Bigg(\min\Big\{b^{K}_{h'}(s,a)+[(P^{\tar}-\hat{P}^K)\cdot\hat{V}_{h'+1}](s,a),b^{K,\hyb}_{h'}(s,a)\nonumber\\
        &\quad +[(P^{\tar}-\hat{P}^{K,\hyb})\cdot\hat{V}_{h'+1}](s,a)\Big\}\Bigg)\label{lemma_14_1}.
       \end{align}
           
 We notice that $b^{K,\hyb}_h(s,a)+[(P^{\tar}-\hat{P}^{K,\hyb})\cdot\hat{V}_{h'+1}](s,a)$ in~\eqref{lemma_14_1} can be bounded by
 \begin{equation}
     \begin{aligned}
  &\quad b^{K,\hyb}_h(s,a)+[(P^{\tar}-\hat{P}^{K,\hyb})\cdot\hat{V}_{h'+1}](s,a) \notag\\
  &=b^{K,\hyb}_h(s,a)+[(P^{\tar}-P^{K,\hyb}+P^{K,\hyb}-\hat{P}^{K,\hyb})\cdot\hat{V}_{h'+1}](s,a) \nonumber\\
 &\leq 2b^{K,\hyb}_{h'}(s,a)+H\nu(s,a).\nonumber
     \end{aligned}
 \end{equation}

 Thus we have
 \begin{equation}
     \begin{aligned}
         \sum_{s\in\mathcal{S}}d^{\pi^*}_h(s)(V^*_h(s)-\hat{V}_h(s)) &\leq 2\sum_{h'}\sum_{s,a}d^{\pi^*}_{h'}(s,a)\min\left\{b^K_{h'}(s,a),b^{K,\hyb}_{h'}(s,a)+H\nu(s,a)\right\}.\nonumber
     \end{aligned}
 \end{equation}
 \end{proof}
 \end{lemma}

\begin{proof} Then we move to prove the sub-optimality gap upper bound.
\begin{align}
    V^*_1-V^{\hat{\pi}}_1&\leq   V^*_1-\hat{V}_1\leq 2\min\left\{\mathbb{E}^{\pi^*}\left[\sum_{h=1}^H b^K_h(s,a)\right], \mathbb{E}^{\pi^*}\left[\sum_{h=1}^H b^{K,\hyb}_h(s,a)\right]+H^2\nu(s,a)\right\}.\nonumber
\end{align}

Below we focus on the term $    \mathbb{E}^{\pi^*}\left[\sum_{h=1}^H b^{K,\hyb}_h(s,a)\right]$.
\begin{align}
     \mathbb{E}^{\pi^*}\left[\sum_{h=1}^H b^{K,\hyb}_h(s,a)\right]=\sum_{s,a,h}d^{\pi^*}_h(s,a)\left( \sqrt{\frac{c\mathbb{V}_{\hat{P}^{K,\hyb}}(\hat{V}_{h+1})L}{N^K_{s,a}+N^{\src}_{s,a}}}+\frac{c HL}{N^K_{s,a}+N^{\src}_{s,a}}+\frac{\nu(s,a) H N^{\src}_{s,a}}{N^K_{s,a}+N^{\src}_{s,a}}\right).\nonumber
\end{align}

We first upper bound the third term following
\begin{align}
    \sum_h\sum_{s,a}d^{\pi^*}_h(s,a)\frac{\nu(s,a) H N^{\src}_{s,a}}{N^K_{s,a}+N^{\src}_{s,a}}\leq \sum_h\sum_{s,a}d^{\pi^*}_h(s,a)\nu(s,a) H\leq H^2\nu_{\max},\nonumber
\end{align}
where $\nu_{\max}:=\max_{s,a}\nu(s,a)$.

Then we focus on the second term. We notice that
\begin{align}
    \sum_h\sum_{s,a}d^{\pi^*}_h(s,a) \frac{HL}{N^K_{s,a}+N^{\src}_{s,a}}\leq \sum_h\sum_{s,a}d^{\pi^*}_h(s,a)   \frac{ HL}{N^{\src}_{s,a}}=c\sum_h\sum_{s,a}d^{\pi^*}_h(s,a)   \frac{HL}{\sum_h N^{\src}_{s,a,h}}.\nonumber
\end{align}

Under the good event we have $\frac{1}{N^{\src}_{s,a,h}} \leq \frac{L}{N^\src d^{\pi^b}_h(s,a)}$.

Thus we have
\begin{align}
     \sum_h\sum_{s,a}d^{\pi^*}_h(s,a) \frac{ HL}{N^K_{s,a}+N^{\src}_{s,a}}&\leq \frac{HL^2}{N^{\src}}\sum_{s,a}   \frac{\sum_hd^{\pi^*}_h(s,a)}{\sum_h d^{\pi^b}_h(s,a) }\notag\\
     &=\frac{HL^2}{N^{\src}}\sum_{s,a} \mathbbm{1}\{\pi^*(s)=a\}  \frac{\sum_hd^{\pi^*}_h(s,a)}{\sum_h d^{\pi^b}_h(s,a) } \nonumber\\
     &\leq \frac{HL^2 C^*_{\nu}S}{N^\src}.\nonumber
\end{align}

We also have
\begin{align}
       \sum_h\sum_{s,a}d^{\pi^*}_h(s,a) \frac{HL}{N^K_{s,a}+N^{\src}_{s,a}}\leq  \frac{HL}{K}\sum_{k,h}\sum_{s,a}d^{\pi^k}_h(s,a) \frac{1}{N^k_{s,a}}=\frac{HL}{K}\sum_{s,a}\sum_{n=1}^{N^k_{s,a}}\frac{1}{n}\leq \frac{HL^2SA}{K}.\nonumber
\end{align}

Thus this term can be bounded as 
\begin{align}
        \sum_h\sum_{s,a}d^{\pi^*}_h(s,a) \frac{c HL}{N^K_{s,a}+N^{\src}_{s,a}}\leq \frac{2cHL^2S}{ K/A+N^\src/C^*_{\nu}}.\nonumber
\end{align}

Thus, we have
\begin{align}
    \mathbb{E}^{\pi^*}\left[\sum_{h=1}^H b^{K,\hyb}_h(s,a)\right]
&\leq c\sum_h\sum_{s,a}d^{\pi^*}_h(s,a)\sqrt{\frac{\mathbb{V}_{\hat{P}^{K,\hyb}}(\hat{V}_{h+1})L}{N^K_{s,a}+N^{\src}_{s,a}}}+\frac{2cHL^2S}{ K/A+N^\src/C^*_{\nu}}+\nu_{\max} H^2.\nonumber
\end{align}

Then we move to bound the first term. We first observe that 
\begin{align}
    \mathbb{V}_{\hat{P}^{K,\hyb}}(\hat{V}_{h+1})\leq 2\mathbb{V}_{P^{\tar}}(\hat{V}_{h+1})+\frac{5H^2L}{3(N^K_{s,a}+N^{\src}_{s,a})}+6H^2\nu_{\max}.\nonumber
\end{align}

The reason is as follows. Based on the good event, we have
\begin{align}
&\quad\mathbb{V}_{\hat{P}^{K,\hyb}}(\hat{V}_{h+1})\nonumber\\
&\leq 2  \mathbb{V}_{P^{K,\hyb}}(\hat{V}_{h+1})+\frac{5H^2L}{3(N^K_{s,a}+N^{\src}_{s,a})}\nonumber\\
  &=2 \mathbb{V}_{P^{\tar}}(\hat{V}_{h+1})+2\left( \mathbb{V}_{P^{K,\hyb}}(\hat{V}_{h+1})-\mathbb{V}_{P^{\tar}}(\hat{V}_{h+1})\right)+\frac{5H^2L}{3(N^K_{s,a}+N^{\src}_{s,a})}\nonumber\\
    &=2 \mathbb{V}_{P^{\tar}}(\hat{V}_{h+1})+2\left( P^{K,\hyb}\cdot(\hat{V}_{h+1})^2-(P^{K,\hyb}\cdot\hat{V}_{h+1})^2-P^{\tar}\cdot(\hat{V}_{h+1})^2+(P^{\tar}\cdot\hat{V}_{h+1})^2\right)\notag\\
    &\quad +\frac{5H^2L}{3(N^K_{s,a}+N^{\src}_{s,a})}\nonumber\\
  &\leq 2 \mathbb{V}_{P^{\tar}}(\hat{V}_{h+1})+\frac{5H^2L}{3(N^K_{s,a}+N^{\src}_{s,a})}+6H^2\nu_{\max}.\nonumber
\end{align}


Thus we have
\begin{align}
   \sqrt{ \frac{\mathbb{V}_{\hat{P}^{K,\hyb}}(\hat{V}_{h+1})L}{N^K_{s,a}+N^{\src}_{s,a}}}\nonumber&\leq    \sqrt{ \frac{2\mathbb{V}_{P^{\tar}}(\hat{V}_{h+1})L}{N^K_{s,a}+N^{\src}_{s,a}}}+\sqrt{\frac{5H^2L^2}{3(N^K_{s,a}+N^{\src}_{s,a})^2}}+\sqrt{H\nu_{\max} \cdot\frac{6H\nu_{\max} L}{N^K_{s,a}+N^{\src}_{s,a}}}\nonumber\\
   &\leq    \sqrt{ \frac{2\mathbb{V}_{P^{\tar}}(\hat{V}_{h+1})L}{N^K_{s,a}+N^{\src}_{s,a}}}+\frac{2\sqrt{2}HL}{N^K_{s,a}+N^{\src}_{s,a}}+H\nu_{\max}.\nonumber\\
\end{align}

Then we move to upper-bound the term$\sum_h\sum_{s,a}d^{\pi^*}_h(s,a)\sqrt{\frac{\mathbb{V}(\hat{V}_{h+1})L}{N^K_{s,a}+N^{\src}_{s,a}}}$. First we notice that this term can be bounded as follows,
\begin{align}
 \sum_h\sum_{s,a}d^{\pi^*}_h(s,a)\sqrt{\frac{\mathbb{V}(\hat{V}_{h+1})L}{N^K_{s,a}+N^{\src}_{s,a}}}\nonumber
 &\leq      \sum_h\sum_{s,a}d^{\pi^*}_h(s,a)\sqrt{\frac{\mathbb{V}(\hat{V}_{h+1})L}{N^{\src}_{s,a}}}\nonumber\\
    &=\sum_h\sum_{s,a}d^{\pi^*}_h(s,a)\sqrt{\frac{\mathbb{V}(\hat{V}_{h+1})L}{\sum_hN^{\src}_{s,a,h}}}\nonumber\\
    &\leq \sum_h\sum_{s,a}d^{\pi^*}_h(s,a)\sqrt{\frac{\mathbb{V}(\hat{V}_{h+1})L^2}{N^{\src}\sum_h d_h^{\pi_b}(s,a)}}\label{eq:bino_1}\\
    &=\sqrt{\frac{L^2}{N^\src}}\sum_h\sum_{s,a}d^{\pi^*}_h(s,a)\sqrt{\frac{\mathbb{V}(\hat{V}_{h+1})}{\sum_h d_h^{\pi_b}(s,a)}}\nonumber\\
     &=\sqrt{\frac{L^2}{N^\src}}\sum_h\sum_{s,a}\sqrt{d^{\pi^*}_h(s,a)\mathbb{V}(\hat{V}_{h+1})}\sqrt{\frac{d^{\pi^*}_h(s,a)}{\sum_h d_h^{\pi_b}(s,a)}}\nonumber\\
    &\leq \sqrt{\frac{L^2}{N^\src}}\sum_{s,a}\sqrt{\sum_h d^{\pi^*}_h(s,a)\mathbb{V}(\hat{V}_{h+1})}\sqrt{\frac{\sum_h d^{\pi^*}_h(s,a)}{\sum_h d_h^{\pi_b}(s,a)}},\nonumber
\end{align}
where~\eqref{eq:bino_1} holds by Lemma~\ref{lemma:bino}, and the last inequality holds by the Cauchy–Schwarz inequality. 

Since $\frac{\sum_h d^{\pi^*}_h(s,a)}{\sum_h d_h^{\pi_b}(s,a)}\leq \max_h \frac{d^{\pi^*}_h(s,a)}{d^{\pi_b}_h(s,a)}\leq\max_{h,s,a}\frac{d^{\pi^*}_h(s,a)}{d^{\pi_b}_h(s,a)} =C^*_{\nu}$, we then have
\begin{align}\sum_h\sum_{s,a}d^{\pi^*}_h(s,a)\sqrt{\frac{\mathbb{V}(\hat{V}_{h+1})L}{N^K_{s,a}+N^{\src}_{s,a}}}  \nonumber
&\leq   \sqrt{\frac{L^2C^*_{\nu}}{N^\src}}\sum_{s,a}\sqrt{\sum_h d^{\pi^*}_h(s,a)\mathbb{V}(\hat{V}_{h+1})}\nonumber\\
&= \sqrt{\frac{L^2C^*_{\nu}}{N^\src}}\sum_{s,a}\sqrt{\sum_h 1\{\pi^*(s)=a\}d^{\pi^*}_h(s,a)\mathbb{V}(\hat{V}_{h+1})}\nonumber\\
  &  \leq \sqrt{\frac{L^2C^*_{\nu}}{N^\src}}\sqrt{\sum_{s,a} 1\{\pi^*(s)=a\}}\sqrt{\sum_{s,a}\sum_h d^{\pi^*}_h(s,a)\mathbb{V}(\hat{V}_{h+1})}\nonumber\\
    &  =\sqrt{\frac{L^2C^*_{\nu}S}{N^\src}}\sqrt{\sum_{s,a}\sum_h d^{\pi^*}_h(s,a)\mathbb{V}(\hat{V}_{h+1})},\nonumber
\end{align}
where the first equality holds since we assume $\pi^*$ is deterministic.

According to Lemma~\ref{lemma:variance_offline}, we know $\sum_{h=1}^H\mathbb{E}^{\pi^*}[\mathbb{V}(\hat{V}_{h+1})]\leq 4H^2+4H\mathbb{E}^{\pi^*}[\min\{b_h^{K},b_h^{K,\hyb}\}]$. Thus,
\begin{align}
\sum_h\sum_{s,a}d^{\pi^*}_h(s,a)\sqrt{\frac{\mathbb{V}(\hat{V}_{h+1})L}{N^K_{s,a}+N^{\src}_{s,a}}}  \nonumber
    &\leq \sqrt{\frac{L^2 C^*_{\nu}S}{N^\src}}\sqrt{4H^2+4H\mathbb{E}^{\pi^*}\left[\sum_h b^{K,\hyb}_h\right]}\nonumber\\
    &\leq 2\sqrt{\frac{C^*_{\nu}SH^2L^2}{N^\src}}+\sqrt{\frac{4HC^*_{\nu}SL^2}{N^\src}\mathbb{E}^{\pi^*}\left[\sum_h b^{K,\hyb}_h\right]}\nonumber\\
&\leq 2\sqrt{\frac{C^*_{\nu}SH^2L^2}{N^\src}}+\frac{8C^*_{\nu}HSL^2}{N^\src}+\frac{1}{2}\mathbb{E}^{\pi^*}\left[\sum_h b^{K,\hyb}_h\right].\nonumber
\end{align}

We  also know that the term $\sum_h\sum_{s,a}d^{\pi^*}_h(s,a)\sqrt{\frac{\mathbb{V}(\hat{V}_{h+1})L}{N^K_{s,a}+N^{\src}_{s,a}}}$ can be upper bounded by the following way
\begin{align}
\sum_h\sum_{s,a}d^{\pi^*}_h(s,a)\sqrt{\frac{\mathbb{V}(\hat{V}_{h+1})L}{N^K_{s,a}+N^{\src}_{s,a}}}&\leq      \sum_h\sum_{s,a}d^{\pi^*}_h(s,a)\sqrt{\frac{\mathbb{V}(\hat{V}_{h+1})L}{N^{K}_{s,a}}}\nonumber\\
      &\leq \sqrt{\sum_{s,a}\sum_h d^{\pi^*}_h(s,a)\mathbb{V}(\hat{V}_{h+1})}\sqrt{\sum_h\sum_{s,a}d^{\pi^*}_h(s,a)\frac{L}{N^K_{s,a}}}\nonumber\\
      &\leq \sqrt{4H^2+4H\mathbb{E}^{\pi^*}\left[\sum_h b^{K,\hyb}_h\right]}\sqrt{\sum_h\sum_{s,a}d^{\pi^*}_h(s,a)\frac{L}{N^K_{s,a}}}\nonumber\\
      &=\sqrt{4H^2+4H\mathbb{E}^{\pi^*}\left[\sum_h b^{K,\hyb}_h\right]}\sqrt{\frac{1}{K} \sum_{k,h}\sum_{s,a}d^{\pi^*}_h(s,a)\frac{L}{N^K_{s,a}}}\nonumber\\
      &\leq \sqrt{4H^2+4H\mathbb{E}^{\pi^*}\left[\sum_h b^{K,\hyb}_h\right]}\sqrt{\frac{1}{K} \sum_{k,h}\sum_{s,a}d^{\pi^k}_h(s,a)\frac{L}{N^K_{s,a}}}\nonumber\\
      & =\sqrt{4H^2+4H\mathbb{E}^{\pi^*}\left[\sum_h b^{K,\hyb}_h\right]}\sqrt{\frac{L}{K} \sum_{s,a}\sum_{n=1}^{N^k_{s,a}}\frac{1}{n}}\nonumber\\
      &\leq \sqrt{4H^2+4H\mathbb{E}^{\pi^*}\left[\sum_h b^{K,\hyb}_h\right]}\sqrt{\frac{SAL^2}{K} }\nonumber\\
&\leq 2\sqrt{\frac{H^2SAL^2}{K}}+\frac{8HSAL^2}{K}+\frac{1}{2}\mathbb{E}^{\pi^*}\left[\sum_h b^{K,\hyb}_h\right].\nonumber
\end{align}

Altogether, we have
\begin{align}
    \sum_h\sum_{s,a}d^{\pi^*}_h(s,a)\sqrt{\frac{2\mathbb{V}(\hat{V}_{h+1})L}{N^K_{s,a}+N^{\src}_{s,a}}}\leq 4\sqrt{\frac{SH^2}{K/A+N^\src/C^*_{\nu}}}+\frac{32HSL^2}{N^\src/C^*_{\nu}+K/A}+\frac{1}{2}\mathbb{E}^{\pi^*}\left[\sum_h b^{K,\hyb}_h\right]\nonumber
\end{align}

To sum up, we have
\begin{align}
    V_1^*-V_1^{\hat{\pi}}\leq \Tilde{\mathcal{O}}\left( \min\left\{ \sqrt{\frac{H^2SA}{K}}, \sqrt{\frac{SH^2}{K/A+N^{\src}/C^*_{\nu}}}+\nu_{\max}H^2\right\}\right).\nonumber
\end{align}

\end{proof}

\section{Proof for Instance-Dependent Sub-optimality Gap Upper Bound}\label{sec:bpi_dependent_proof}
In this subsection, we provide the proof for the instance-dependent sub-optimality gap upper bound.

We first define some quantities frequently used below. We define 
\begin{align}
\Delta_h(s,a)
&:=V_h^*(s)-Q_h^*(s,a),
\qquad \Delta_{\min}
:=
\min_{\substack{h,s,a:\\\Delta_h(s,a)>0}}
\Delta_h(s,a)>0.
\nonumber
\end{align}

Recall the pessimistic Bellman backups defined in our  devised MAX-LCB-VI algorithm:
\begin{align}
\widehat Q_h^{K,\rm on}(s,a)
&=
\left[
r(s,a)
+\widehat P^K_{s,a}\widehat V_{h+1}
-b_h^{K,\rm on}(s,a)
\right]_+,
\nonumber\\
\widehat Q_h^{K,\rm hyb}(s,a)
&=
\left[
r(s,a)
+\widehat P^{K,\rm hyb}_{s,a}\widehat V_{h+1}
-b_h^{K,\rm hyb}(s,a)
\right]_+,
\nonumber\\
\widehat Q_h(s,a)
&=
\max\left\{
\widehat Q_h^{K,\rm on}(s,a),
\widehat Q_h^{K,\rm hyb}(s,a)
\right\},
\nonumber
\end{align}

Following the deficit technique used in~\citet{wang2022gap}, 
we define the reward of an imaginary MDP with target transition kernel by
\begin{align}
\widetilde r_h(s,a)
:=
\widehat Q_h(s,a)
-
P^{\rm tar}_{s,a}\widehat V_{h+1},
\nonumber
\end{align}
and define its deficit relative to the true reward by
\begin{align}
D_h(s,a)
:=
r(s,a)-\widetilde r_h(s,a).
\nonumber
\end{align}

By pessimism, $D_h(s,a)\geq0$.
Moreover, $\widehat Q,\widehat V$ are the optimal value functions of this
imaginary MDP and $\widehat\pi$ is an optimal policy in the imaginary MDP.

Below we provide some useful lemmas used in this subsection. 
\begin{lemma}\label{lem:min-deficit}
On the good event,
\begin{align}
D_h(s,a)
\leq
2\min\left\{
b_h^{K,\rm on}(s,a),
b_h^{K,\rm hyb}(s,a)
\right\}.
\label{eq:min-deficit}
\end{align}
\end{lemma}

\begin{proof}
For a branch $j\in\{\rm on,hyb\}$, we define
\begin{align}
D_h^j(s,a)
:=
r(s,a)+P^{\rm tar}_{s,a}\widehat V_{h+1}
-\widehat Q_h^{K,j}(s,a).
\nonumber
\end{align}

Since $\widehat Q_h$ is the maximum of the two lower confidence bounds $
\min\{
D_h^{\rm on}(s,a),
D_h^{\rm hyb}(s,a)
\}$. If truncation at zero is inactive, we know
$D_h^j(s,a)\leq2b_h^{K,j}(s,a)$.
If truncation is active, then
\(
r+\widehat P^{K,j}\widehat V_{h+1}-b_h^{K,j}\leq0
\),
and hence
\begin{align}
r+P^{\rm tar}\widehat V_{h+1}
&\leq
b_h^{K,j}
+
\left|
\bigl(P^{\rm tar}-\widehat P^{K,j}\bigr)
\widehat V_{h+1}
\right|
\leq2b_h^{K,j}.
\nonumber
\end{align}
Thus $D_h^j\leq2b_h^{K,j}$ in both cases, proving
\eqref{eq:min-deficit}.
\end{proof}

\begin{lemma}[Deficit thresholding]\label{lem:deficit-threshold}
For every pessimistic algorithm,
\begin{align}
V_1^*(\rho)-V_1^{\widehat\pi}(\rho)
\leq
2\sum_{h=1}^{H}
\E_{\pi^*,P^{\rm tar}}
\left[
\left(
D_h(s_h,a_h^*)
-
\frac{\Delta_{\min}}{2H}
\right)_+
\right].
\label{eq:deficit-threshold}
\end{align}
\end{lemma}

\begin{proof}
This is the deficit-thresholding lemma of
Wang, Cui, and Du~\cite[Corollary 6.1]{wang2022gap}.
\end{proof}

We combine Lemmas~\ref{lem:min-deficit} and
\ref{lem:deficit-threshold}, and use
$[2x-2t]_+\leq2[x-t]_+$ to obtain
\begin{align}
V_1^*(\rho)-V_1^{\widehat\pi}(\rho)
&\leq
4\sum_{h=1}^{H}
\E_{\pi^*,P^{\rm tar}}
\left[
\left(
\min\{b_h^{K,\rm on},b_h^{K,\rm hyb}\}
-
\frac{\Delta_{\min}}{4H}
\right)_+
\right].
\label{eq:thresholded-bonus}
\end{align}

\begin{lemma}\label{lem:linear-credit}
For $n,m\geq0$, $B>0$, $t>0$, and $0\leq\beta<t$,
\begin{align}
\left[
\sqrt{\frac{B}{(n+m)\vee1}}
+
\frac{m\beta}{(n+m)\vee1}
-t
\right]_+
\leq
\frac{B}{
t\left\{
n+m(1-\beta/t)
\right\}
}.
\label{eq:linear-credit}
\end{align}
If $\beta\geq t$:
\begin{align}
\left[
\sqrt{\frac{B}{n\vee1}}-t
\right]_+
\leq
\frac{B}{t(n\vee1)}.
\nonumber
\end{align}
\end{lemma}

\begin{proof}
The claim holds if the left-hand side of
\eqref{eq:linear-credit} is zero.  Otherwise,
\begin{align}
\sqrt{\frac{B}{n+m}}
&>
\frac{
t\{n+m(1-\beta/t)\}
}{
n+m
}.
\nonumber
\end{align}

On this event,
\begin{align}
\sqrt{\frac{B}{n+m}}
-
\frac{
t\{n+m(1-\beta/t)\}
}{
n+m
}
\leq
\frac{
B/(n+m)
}{
t\{n+m(1-\beta/t)\}/(n+m)
}
=
\frac{B}{
t\{n+m(1-\beta/t)\}
}.
\nonumber
\end{align}
This proves \eqref{eq:linear-credit}.  The online-only inequality is obtained
by setting $m=0$.
\end{proof}

Then we formally provide the proof of the instance-dependent sub-optimality gap upper bound.

\begin{proof}
The proof below is largely based on~\eqref{eq:thresholded-bonus}.  Since $\left[
\min\{x,y\}-t
\right]_+
\leq
\min\left\{
[x-t]_+,[y-t]_+
\right\}$, the total thresholded deficit is no larger than either the total online
thresholded deficit or the total hybrid thresholded deficit.

For the online branch, apply
$[x-t]_+\leq x^2/t$, use
$t=\Delta_{\min}/(4H)$, and separate the Bernstein variance and lower-order
terms.  Based on the instance-independent sub-optimality gap upper bound obtained in the previous subsection, it is easy to have 
\begin{align}
V_1^*-V_1^{\widehat\pi}
\lesssim
\frac{H^3SA L^2}{K\Delta_{\min}}.
\label{eq:online-gap-bound}
\end{align} 

For the hybrid branch, we have
\begin{align}
[x+y-t]_+
\leq
[x-t/2]_+ + [y-t/2]_+.
\nonumber
\end{align}

Apply Lemma~\ref{lem:linear-credit} to the square-root and weighted-bias
terms.  The source contribution in the resulting denominator is at least
$( 1- \frac{H^2\nu_{\max}}{\Delta_{\min}} )_+ N^{\rm src}(s,a)$.  Dropping the nonnegative online count only enlarges
the upper bound.  Then we have
\begin{align}
V_1^*-V_1^{\widehat\pi}
\lesssim
\frac{
H^3S C_\nu^*L^2
}{
( 1- \frac{ H^2\nu_{\max}}{\Delta_{\min}} )_+ N^{\rm src}\Delta_{\min}
}.
\label{eq:source-gap-bound}
\end{align}


If we set $x:=K/A$ and $
y:=( 1- \frac{ H^2\nu_{\max}}{\Delta_{\min}} )_+ N^{\rm src}/C_\nu^*$, equations~\eqref{eq:online-gap-bound} and
\eqref{eq:source-gap-bound} imply
\begin{align}
V_1^*(\rho)-V_1^{\widehat\pi}(\rho)
&\leq
\frac{CH^3SL^2}{\Delta_{\min}}
\min\left\{\frac1x,\frac1y\right\}
=
\frac{CH^3SL^2}{
\Delta_{\min}\max\{x,y\}
}
\leq
\frac{2CH^3SL^2}{
\Delta_{\min}(x+y)
}.
\nonumber
\end{align}

Then we have 
\begin{align}
  V_1^*(\rho)-V_1^{\widehat\pi}(\rho)
\leq
\frac{
CH^3SL^2
}{
\Delta_{\min}
\left\{
K/A+( 1- \frac{ H^2\nu_{\max}}{\Delta_{\min}} )_+ N^{\rm src}/C_\nu^*
\right\}
}.\nonumber
\end{align}

Hence, MAX-LCB-VI returns an
$\epsilon$-optimal policy with probability at least $1-\delta$ whenever $\frac KA
+
\frac{\left(
1-\frac{ H^2\nu_{\max}}{\Delta_{\min}}
\right)_+ N^{\rm src}}{C_\nu^*}
\geq
\frac{CH^3SL^2}{\epsilon\Delta_{\min}}$, i.e.,
\begin{align}
K
\geq
\Omega\left(A\left[
\frac{H^3SL^2}{\epsilon\Delta_{\min}}
-
\frac{N^{\rm src}}{^*}
\left(
1-\frac{ H^2\nu_{\max}}{\Delta_{\min}}
\right)_+
\right]_+\right).
\nonumber
\end{align}

\end{proof}

\section{Proof for Instance-Independent Regret Lower Bound}\label{sec:regret_independent_lower_proof}

In this subsection, we provide the proof for instance-independent regret lower bound. We first provide some useful lemmas below.

\begin{lemma}[Pinsker plus data processing]\label{lem:pinsker}
Let $\rho$ and $\widetilde\rho$ be two probability distributions.  For any $Z\in[0,1]$,
\begin{align}
\left|\E_\rho Z-\E_{\widetilde\rho}Z\right|
\le
\sqrt{\frac12\KL(\rho,\widetilde\rho)}.
\end{align}
\end{lemma}

\begin{lemma}[Bernoulli KL bound]\label{lem:bernoulli}
There exists a universal constant $C_{\mathrm{kl}}>0$ such that, for all $|u|,|v|\le 1/4$,
\begin{align}
\kl\left(\frac12+u,\frac12+v\right)
\le
C_{\mathrm{kl}}(u-v)^2.
\end{align}
\end{lemma}

\begin{proof}
    Below we consider  $S\geq6$, $A\geq2$, $H\geq8$, and $K\geq1$.  For some integer
$n_0\geq0$, we suppose that $N^{\src}(x_i,a)=n_0,
i\in[S-3],\ a\in[A]$ and $N^{\src}(s,a)
\geq
n_0+\frac{KH}{(S-3)A},
\qquad
s\notin\{x_1,\ldots,x_{S-3}\},\ a\in[A]$. Let $\tau^*$ be the optimum of LP defined previously. 

We first notice that $\tau^*= n_0+\frac{KH}{(S-3)A}$. Then we provide the construction of MDP instances.
 
  \par\medskip\noindent\emph{Hard instances.}
For every vector $(a_1^*,\ldots,a_{S-3}^*) \in
[A]^{S-3}$, we construct a target MDP and a source MDP as follows.

\begin{itemize}
\item \emph{State space.}
There is one initial gateway state $s_0$, $S-3$ informative states
$\{x_i:i\in[S-3]\}$, one good absorbing state $g$, and one bad absorbing
state $b$:
\begin{align}
\mathcal S
&=
\{s_0,x_1,\ldots,x_{S-3},g,b\}.
\nonumber
\end{align}

\item \emph{Action space.}
The action space is $\mathcal A=[A]$ at every state.

\item \emph{Initial state.}
Every episode starts from the deterministic initial state $s_0$.

\item \emph{Transition at the gateway state.}
For every $a\in[A]$ and $i\in[S-3]$,
\begin{align}
P^{\tar}(x_i\mid s_0,a)
&=
P^{\src}(x_i\mid s_0,a)
=
\frac{1}{S-3}.
\nonumber
\end{align}

\item \emph{Target transition at an informative state.}
At $x_i$, the action $a_i^*$ is special.  For every $a\neq a_i^*$,
\begin{align}
P^{\tar}(x_i\mid x_i,a)
&=
1-\frac1H,
\quad
P^{\tar}(g\mid x_i,a)
=
P^{\tar}(b\mid x_i,a)
=
\frac1{2H}.
\nonumber
\end{align}
For the special action,
\begin{align}
P^{\tar}(x_i\mid x_i,a_i^*)
&=
1-\frac1H,
\quad
P^{\tar}(g\mid x_i,a_i^*)
=
\frac{1/2+\gamma}{H},
\quad
P^{\tar}(b\mid x_i,a_i^*)
=
\frac{1/2-\gamma}{H}.
\nonumber
\end{align}

\item \emph{Source transition at an informative state.}
Every non-special action has the same transition as above.  At
$a_i^*$,
\begin{align}
P^{\src}(x_i\mid x_i,a_i^*)
&=
1-\frac1H,
\quad
P^{\src}(g\mid x_i,a_i^*)
=
\frac{1/2+\gamma_{\src}}{H},
\quad
P^{\src}(b\mid x_i,a_i^*)
=
\frac{1/2-\gamma_{\src}}{H},
\nonumber
\end{align}
where $0\leq\gamma_{\src}\leq\gamma$.

\item \emph{Absorbing transitions.}
For every $a\in[A]$,
\begin{align}
P^{\tar}(g\mid g,a)
=
P^{\src}(g\mid g,a)
=1,
\qquad
P^{\tar}(b\mid b,a)
=
P^{\src}(b\mid b,a)
=1.
\nonumber
\end{align}
\end{itemize}

We consider that the source and target MDPs have the common step-independent reward $r(s,a)
=\mathbf 1\{s=g\}$. All transition probabilities not specified above are zero.  The parameters
$\gamma$ and $\gamma_{\src}$ are chosen at the end of the proof.

  \par\medskip\noindent\emph{Optimal policy and target action gap.}
In every hard target MDP, a deterministic optimal policy satisfies
\begin{align}
\pi_h^*(x_i)
&=
a_i^*,
\qquad
i\in[S-3],\ h\in[H],
\nonumber
\end{align}
and may take any action at $s_0,g$, and $b$.  Since
$V_{h+1}^*(g)=H-h$ and $V_{h+1}^*(b)=0$, for every $a\neq a_i^*$,
\begin{align}
Q_h^*(x_i,a_i^*)-Q_h^*(x_i,a)
&=
\frac{\gamma(H-h)}{H}
\geq
\frac{\gamma}{2},
\qquad
h\leq\frac H2.
\label{eq:reference-style-gap}
\end{align}

For $u\in[0,1/4]$, we have
\begin{align}
\KL\left(
\left(1-\frac1H,\frac1{2H},\frac1{2H}\right),
\left(1-\frac1H,\frac{1/2+u}{H},\frac{1/2-u}{H}\right)
\right)
&=
-\frac1{2H}\log(1-4u^2)
\leq
\frac{4u^2}{H}.
\label{eq:reference-style-one-step-kl}
\end{align}

Let $M_i$ be the number of target episodes that enter $x_i$.  Let $T_i(a)$
be the number of visits to $x_i$ at steps
$2,\ldots,\lfloor H/2\rfloor$ on which action $a$ is selected, and let
$T_i=\sum_aT_i(a)$.  The self-loop probability is $1-1/H$ under every
action, and therefore
\begin{align}
\E[T_i\mid M_i]
&=
M_i
\sum_{h=2}^{\lfloor H/2\rfloor}
\left(1-\frac1H\right)^{h-2}
\geq
\frac{H}{8}M_i.
\label{eq:reference-style-visits}
\end{align}

Fix $i$, fix the special actions at all states other than $x_i$, and
condition on the gateway outcomes.  Let $\Prb_{i,0}$ be the conditional law
of the source data, target observations, and algorithmic randomization when
all actions at $x_i$ are neutral in both domains.  For $a\in[A]$, let
$\Prb_{i,a}$ be the corresponding law when $a$ is special at $x_i$, with
target perturbation $\gamma$ and source perturbation $\gamma_{\src}$.

Under $\Prb_{i,0}$, the two laws differ in exactly $n_0$ source transitions
from $(x_i,a)$.  By the KL chain rule, the expected target contribution is
the one-step KL multiplied by the number of target visits to $(x_i,a)$.
Averaging over $a\in[A]$ and using
Equation~\eqref{eq:reference-style-one-step-kl} give
\begin{align}
\frac1A\sum_{a=1}^{A}
\KL(\Prb_{i,0},\Prb_{i,a})
&\leq
\frac{4n_0\gamma_{\src}^2}{H}
+
\frac{4\gamma^2M_i}{A}.
\label{eq:reference-style-average-kl}
\end{align}
Here the target contribution uses
$\sum_a N_i^K(a)\leq HM_i$, where $N_i^K(a)$ is the total number of target
visits to $(x_i,a)$ in the $M_i$ episodes.

For every $a$, Pinsker's inequality and $0\leq T_i(a)\leq HM_i$ imply
\begin{align}
\E_{\Prb_{i,a}}[T_i(a)]
&\leq
\E_{\Prb_{i,0}}[T_i(a)]
+
HM_i
\sqrt{\frac12\KL(\Prb_{i,0},\Prb_{i,a})}.
\nonumber
\end{align}
Averaging over $a$, using
$\sqrt{x+y}\leq\sqrt{x}+\sqrt{y}$ and Jensen's inequality, and then
averaging over the gateway outcomes yield
\begin{align}
\frac1A\sum_{a=1}^{A}\E_{\Prb_{i,a}}[T_i(a)]
&\leq
\frac{\E[T_i]}{A}
+
\sqrt{2}\,H\gamma_{\src}\sqrt{\frac{n_0}{H}}\E[M_i]
+
\frac{\sqrt{2}\,H\gamma}{\sqrt A}\E[M_i^{3/2}].
\label{eq:reference-style-pinsker}
\end{align}

The gateway transition gives
$M_i\sim\operatorname{Binomial}(K,1/(S-3))$.  Hence we have $\E[M_i^{3/2}]
\leq
\frac{K}{S-3}
\sqrt{1+\frac{K}{S-3}}$. We then select a universal constant $0<\kappa
\leq
\min\left\{
\frac18,\,
\frac{1}{32(\sqrt2+\sqrt3)}
\right\}$. Suppose that
\begin{align}
\gamma
&\leq
\kappa
\min\left\{
1,\sqrt{\frac{(S-3)A}{K}}
\right\},
\qquad
\gamma_{\src}\sqrt{\frac{n_0}{H}}
\leq
\kappa.
\label{eq:reference-style-testing-condition}
\end{align}

Using Equation~\eqref{eq:reference-style-visits} and
\begin{align}
\min\left\{
1,\sqrt{\frac{(S-3)A}{K}}
\right\}
\sqrt{\frac{1+K/(S-3)}{A}}
&\leq
\sqrt{\frac32},
\nonumber
\end{align}
the last two terms in Equation~\eqref{eq:reference-style-pinsker} are at
most $\E[T_i]/4$.  Since $A\geq2$,
\begin{align}
\frac1A\sum_{a=1}^{A}\E_{\Prb_{i,a}}[T_i(a)]
&\leq
\frac34\E[T_i].
\nonumber
\end{align}

Thus, if $a_i^*$ is uniform on $[A]$, the expected number of early
non-special actions at $x_i$ is at least $\E[T_i]/4$.

Apply the preceding inequality conditionally on the special actions at all
other informative states, and then average
$a_1^*,\ldots,a_{S-3}^*$ independently and uniformly over $[A]$.
Equation~\eqref{eq:reference-style-visits} and
$\sum_i\E[M_i]=K$ give
\begin{align}
\E\left[
\sum_{i=1}^{S-3}
\sum_{k=1}^{K}
\sum_{h=2}^{\lfloor H/2\rfloor}
\mathbf 1\{s_h^k=x_i,\ a_h^k\neq a_i^*\}
\right]
&\geq
\frac{KH}{32}.
\label{eq:reference-style-mistakes}
\end{align}

Consequently, at least one deterministic vector
$(a_1^*,\ldots,a_{S-3}^*)$ satisfies
Equation~\eqref{eq:reference-style-mistakes}.

The performance-difference identity gives
\begin{align}
\E[\Reg(K)]
&=
\E\left[
\sum_{k=1}^{K}\sum_{h=1}^{H}
\left(
V_h^*(s_h^k)-Q_h^*(s_h^k,a_h^k)
\right)
\right].
\nonumber
\end{align}

Combining Equations~\eqref{eq:reference-style-gap} and
\eqref{eq:reference-style-mistakes} yields
\begin{align}
\E[\Reg(K)]
&\geq
\frac{KH\gamma}{64}.
\label{eq:reference-style-regret-gamma}
\end{align}

We set
\begin{align}
\gamma
&=
\frac{\kappa}{2}
\min\left\{
1,\,
\sqrt{\frac{SA}{K}},\,
H\nu_{\max}+\sqrt{\frac{H}{\tau^*}}
\right\},
\nonumber\\
\gamma_{\src}
&=
\begin{cases}
0,
&
H\nu_{\max}\geq 2\gamma,
\\[0.2em]
\displaystyle
\gamma-\frac{H\nu_{\max}}{2},
&
H\nu_{\max}<2\gamma.
\end{cases}
\label{eq:reference-style-choice}
\end{align}

This choice has $\gamma\leq1/4$.  Since $S\geq6$ implies
$S/4\leq S-3$, the first condition in
Equation~\eqref{eq:reference-style-testing-condition} holds.  It remains to
verify the source testing condition and the source-target shift constraint.
We distinguish two cases according to the size of the bias.

  \par\medskip\noindent\emph{Case 1: $H\nu_{\max}\geq2\gamma$ (large-bias regime).}
In this case, set $\gamma_{\src}=0$.  The source testing condition holds
immediately:
\begin{align}
\gamma_{\src}\sqrt{\frac{n_0}{H}}
&=
0
\leq
\kappa.
\nonumber
\end{align}
Moreover, at every special action,
\begin{align}
\left\|
P^{\tar}(\cdot\mid x_i,a_i^*)
-
P^{\src}(\cdot\mid x_i,a_i^*)
\right\|_1
&=
\frac{2\gamma}{H}
\leq
\nu_{\max}.
\nonumber
\end{align}
Thus the source data are neutral at the special actions, while the complete
target perturbation is covered by the allowed bias.

  \par\medskip\noindent\emph{Case 2: $H\nu_{\max}<2\gamma$ (small-bias regime).}
In this case,
\begin{align}
\gamma_{\src}
&=
\gamma-\frac{H\nu_{\max}}2.
\nonumber
\end{align}
By the definition of $\gamma$,
\begin{align}
\gamma_{\src}
&\leq
\frac{\kappa}{2}
\left(
H\nu_{\max}+\sqrt{\frac{H}{\tau^*}}
\right)
-
\frac{H\nu_{\max}}2
=
\frac{\kappa}{2}\sqrt{\frac{H}{\tau^*}}
-
\frac{1-\kappa}{2}H\nu_{\max}
\leq
\frac{\kappa}{2}\sqrt{\frac{H}{\tau^*}}.
\nonumber
\end{align}

Using $n_0\leq\tau^*$, we then  obtain
\begin{align}
\gamma_{\src}\sqrt{\frac{n_0}{H}}
&\leq
\frac{\kappa}{2}\sqrt{\frac{n_0}{\tau^*}}
\leq
\kappa.
\nonumber
\end{align}

The source testing condition therefore holds.  At every special action,
\begin{align}
\left\|
P^{\tar}(\cdot\mid x_i,a_i^*)
-
P^{\src}(\cdot\mid x_i,a_i^*)
\right\|_1
&=
\frac{2(\gamma-\gamma_{\src})}{H}
=
\nu_{\max}.
\nonumber
\end{align}
Hence the source-target shift constraint also holds in the small-bias
regime.

In both cases, the two conditions in
Equation~\eqref{eq:reference-style-testing-condition} hold and the
constructed source-target pair is admissible.  Substituting
Equation~\eqref{eq:reference-style-choice} into
Equation~\eqref{eq:reference-style-regret-gamma} gives
\begin{align}
\E[\Reg(K)]
&\geq
\Omega\left(
KH
\min\left\{
1,\,
\sqrt{\frac{SA}{K}},\,
H\nu_{\max}+\sqrt{\frac{H}{\tau^*}}
\right\}
\right).
\end{align}
\end{proof}

\section{Proof for Instance-Dependent Regret Lower Bound}\label{sec:regret_dependent_lower_proof} 

In this subsection we provide the instance-dependent regret lower bound. We also provide some useful lemmas below.

\begin{lemma}[Offline-to-online KL decomposition]
Let $\mathcal M=(P^{\tar},P^{\src})$ and $\wt{\mathcal M}=(\wt P^{\tar},\wt P^{\src})$ be two environments sharing the same reward, initial distribution, offline sample count, and algorithm.  Suppose they differ only at one first-stage pair $x=(s,a)$.  If $\rho$ and $\wt\rho$ are the joint distribution of the offline data and online histories under $\mathcal M$ and $\wt{\mathcal M}$, then
\begin{equation}
    \KL(\rho,\wt\rho)
    =
    \E_{\mathcal M}[\bar N_K(x)]
      \KL(P_x^{\tar},\wt P_x^{\tar})
    +
    N_x^{\src}
      \KL(P_x^{\src},\wt P_x^{\src}).
    \label{eq:kl-decomposition}
\end{equation}
\end{lemma}

\begin{proof}
The likelihood ratio factors into the product of offline source-transition likelihood ratios and online target-transition likelihood ratios.  The action-selection probabilities are induced by the same algorithm in both environments and cancel in the likelihood ratio.  Since the two environments differ only at $x$, only observations at $x$ contribute.  Taking expectation under $\mathcal M$ gives \eqref{eq:kl-decomposition}.
\end{proof}

\begin{lemma}[Bernoulli KL bound]
There is a universal constant $C_{\mathrm{kl}}>0$ such that for all $p,q\in[1/4,3/4]$,
\[
    \kl(p,q)\le C_{\mathrm{kl}}(p-q)^2.
\]
\end{lemma}

\begin{lemma}~\citep{simchowitz2019non}
For any uniformly good algorithm and every suboptimal pair $x=(s,a)$ of the baseline instance,
\begin{equation}
    \liminf_{K\to\infty}
    \frac{
        \KL(\rho_{\mathcal M},\rho_{\wt{\mathcal M}})
    }{\log K}
    \ge c_0,
    \label{eq:change-measure}
\end{equation}
where $c_0>0$ is a universal constant.
\end{lemma}

\begin{proof} We consider an episodic finite-horizon MDP with horizon $H\ge 2$, state space
\[
    \Sset=[S]\cup\{g,b\},
\]
action space $\Aset=[A]$, and initial state $s_1$ sampled uniformly from $[S]$.  The two states $g$ and $b$ are absorbing.  The reward is deterministic and known:
\[
    r(g,a)=1,\qquad r(b,a)=\frac12,\qquad r(s,a)=0\quad \text{for all }s\in[S].
\]

 For every state-action pair $x=(s,a)$, the offline data contain $N_x^{\src}$ independent source-transition samples.  Since the hard instance below has only one nontrivial decision at the first stage, define
\[
    \bar N_K(s,a):=\sum_{k=1}^K \one\{S_1^k=s,\ A_1^k=a\}.
\]

We define $B_x$ as the one-sided value-scale source-target bias budget:
\begin{equation}
    \max_{h\in[H]}\sup_{0\le V\le H-h}
    \left|(P^{\src}-P^{\tar})V(s,a)\right|
    \le B_x .
    \label{eq:value-bias-budget}
\end{equation}

\par\medskip\noindent\emph{Hard instances.} For each $s\in[S]$, action $1$ is the unique optimal action.  For every $a\ne 1$, fix a gap $\Delta(s,a)>0$ satisfying $\Delta(s,a)\le c_0(H-1)$ for a sufficiently small numerical constant $c_0$, so that all probabilities below lie in $[1/4,3/4]$.

The target transition kernel of the baseline instance $\mathcal M$ is defined by
\begin{align}
    P^{\tar}(g\mid s,1)&=\frac12,\qquad
    P^{\tar}(b\mid s,1)=\frac12,\nonumber\\
    P^{\tar}(g\mid s,a)&=\frac12-\frac{2\Delta(s,a)}{H-1},\qquad
    P^{\tar}(b\mid s,a)=\frac12+\frac{2\Delta(s,a)}{H-1},\qquad a\ne 1.\nonumber
    \label{eq:baseline-target}
\end{align}

For any algorithm,
\begin{equation}
    \E_{\mathcal M}[\Reg(K)]
    \ge
    \sum_{s\in[S]}\sum_{a\ne 1}
    \Delta(s,a)\,\E_{\mathcal M}[\bar N_K(s,a)].
    \label{eq:regret-count-relation}
\end{equation}

\par\medskip\noindent\emph{Alternative instance.} Fix a suboptimal pair $x=(s,a)$ with $a\ne 1$ and write $ \Delta_x:=\Delta(s,a)$. Besides, we know $\tilde{\nu}_x \geq B_{s,a}$

Construct an alternative instance $\wt{\mathcal M}_x$ that coincides with $\mathcal M$ except at $x$.  At $x$, define the alternative target kernel by
\begin{align}
    \wt P^{\tar}(g\mid x)&=\frac12+\frac{2\Delta_x}{H-1},\nonumber\\
    \wt P^{\tar}(b\mid x)&=\frac12-\frac{2\Delta_x}{H-1}.
    \label{eq:alternative-target}
\end{align}

Then, under $\wt{\mathcal M}_x$, action $a$ is better than action $1$ at state $s$ by exactly $\Delta_x$.  The value-scale separation between $P_x^{\tar}$ and $\wt P_x^{\tar}$ is $2\Delta_x$, and hence
\begin{equation}
    \KL(P_x^{\tar},\wt P_x^{\tar})
    \le C\frac{\Delta_x^2}{H^2}.
    \label{eq:target-kl-bound}
\end{equation}

Now choose the source kernels.  Let $ m_x:=\min\{\tilde{\nu}_x,\Delta_x\}$. 

For the baseline instance, set
\begin{align}
    P^{\src}(g\mid x)
    &= P^{\tar}(g\mid x)+\frac{2m_x}{H-1},\nonumber\\
    P^{\src}(b\mid x)
    &=1-P^{\src}(g\mid x).
    \label{eq:baseline-source}
\end{align}

For the alternative instance, set
\begin{align}
    \wt P^{\src}(g\mid x)
    &= \wt P^{\tar}(g\mid x)-\frac{2m_x}{H-1},\nonumber\\
    \wt P^{\src}(b\mid x)
    &=1-\wt P^{\src}(g\mid x).
    \label{eq:alternative-source}
\end{align}

For all other pairs, take source and target kernels to coincide.

This construction satisfies the value-scale shift constraint because moving the probability of $g$ by $2m_x/(H-1)$ changes the value by exactly $m_x\le \tilde{\nu}_x$.  Moreover, the remaining source-source value-scale separation is
$ 2\pos{\Delta_x-\tilde{\nu}_x}$.

Therefore,
\begin{equation}
    \KL(P_x^{\src},\wt P_x^{\src})
    \le C\frac{\pos{\Delta_x-\tilde{\nu}_x}^2}{H^2}.
    \label{eq:source-kl-bound}
\end{equation}

If $\tilde{\nu}_x\ge \Delta_x$, the two source kernels at $x$ are identical, and the offline data have zero information for distinguishing $\mathcal M$ and $\wt{\mathcal M}_x$ at this pair.  If $\tilde{\nu}_x=0$, the source kernels coincide with the target kernels, so offline samples have the same information order as online samples.

Fix $x=(s,a)$ with $a\ne 1$.  By the KL decomposition \eqref{eq:kl-decomposition}, the target KL bound \eqref{eq:target-kl-bound}, and the source KL bound \eqref{eq:source-kl-bound},
\begin{align}
    \KL(\rho_{\mathcal M},\rho_{\wt{\mathcal M}_x})
    &\le
    C\E_{\mathcal M}[\bar N_K(x)]\frac{\Delta_x^2}{H^2}
    +
    C N_x^{\src}\frac{\pos{\Delta_x-\tilde{\nu}_x}^2}{H^2}.
    \label{eq:kl-final-upper}
\end{align}

On the other hand, by the change-of-measure lemma, for all sufficiently large $K$,
\[
    \KL(\rho_{\mathcal M},\rho_{\wt{\mathcal M}_x})
    \ge c_0\log K.
\]
Combining the two displays and multiplying by $H^2/(C\Delta_x^2)$ yields
\begin{align}
    \E_{\mathcal M}[\bar N_K(x)]
    &\ge
    c\frac{H^2\log K}{\Delta_x^2}
    -
    N_x^{\src}
    \left(\frac{\pos{\Delta_x-\tilde{\nu}_x}}{\Delta_x}\right)^2
    \nonumber\\
    &=
    c\frac{H^2\log K}{\Delta_x^2}
    -
    N_x^{\src}
    \left(1-\frac{\tilde{\nu}_x}{\Delta_x}\right)_+^2.
    \label{eq:count-lower-bound}
\end{align}

Since a count is nonnegative, we may take the positive part of the right-hand side.  Multiplying by $\Delta_x$ and summing over $s\in[S]$ and $a\ne 1$, and then using the regret-count relation \eqref{eq:regret-count-relation}, proves
\begin{equation}
    \E_{\mathcal M}[\Reg(K)]
   \gtrsim
    c
\sum_{s,a}
    \left[
        \frac{H^2L}{\Delta(s,a)}
        -
        N_{s,a}^{\src}\Delta(s,a)
        \left(1-\frac{\tilde{\nu}(s,a)}{\Delta(s,a)}\right)_+^2
    \right]_+.
\end{equation}
\end{proof}

\section{Proof for Instance-Independent Sub-optimality Gap Lower Bound}\label{sec:bpi_independent_lower_proof}

In this section, we follow the derivations from~\citet{NEURIPS2021_e61eaa38} to derive the instance-independent sub-optimality gap lower bound.  

\begin{proof}
The construction below is inspired by the bandit-state construction of
\citet{NEURIPS2021_e61eaa38}. We first assume that \(S\geq 8\), \(A\geq 2\), \(H\geq 8\), \(K\geq 1\),
\(N^{\src}\geq 0\), \(C_\nu^*\geq 2\), and \(\nu_{\max}\geq 0\).
 And we set $M:=\min\{\lfloor C_\nu^*\rfloor,A\}$. Let \(S_0:=S-4\).  The state space is defined as $\mathcal S=\{s_0,x_1,\ldots,x_{S_0},g,b,z\}$. Here \(s_0\) is the  initial state, the states \(x_i\) are
informative states, \(g\) is a good absorbing state, \(b\) is a bad absorbing state, and \(z\) is a dead absorbing state.  The action space is
\(\mathcal A=[A]\).  The deterministic reward is
\begin{align}
r(g,a)=1,
\qquad
r(s,a)=0
\quad\text{for every }s\neq g\text{ and }a\in[A].
\nonumber
\end{align}
The states \(g,b,z\) are absorbing under every action.  

\par\medskip\noindent\emph{Target hard family.}
For each vector
\((a_1^*,\ldots,a_{S_0}^*)\in[M]^{S_0}\), construct one target MDP.
At the gateway state, action \(1\) moves uniformly to the informative states,
whereas every other action moves to \(z\):
\begin{align}
P^{\tar}(x_i\mid s_0,1)=\frac{1}{S_0},
\qquad
P^{\tar}(z\mid s_0,a)=1\quad(a\neq 1).
\nonumber
\end{align}

We consider a parameter \(\theta\in(0,1/4]\) fixed near the end of the proof.  For every  \(x_i\) and every action \(a\), we set  $P^{\tar}(x_i\mid x_i,a)=1-\frac1H$. For a non-special action \(a\neq a_i^*\),
\begin{align}
P^{\tar}(g\mid x_i,a)
=P^{\tar}(b\mid x_i,a)
=\frac{1}{2H},
\nonumber
\end{align}
whereas the special action \(a_i^*\) satisfies
\begin{align}
P^{\tar}(g\mid x_i,a_i^*)
=\frac1H\left(\frac12+\theta\right),
\qquad
P^{\tar}(b\mid x_i,a_i^*)
=\frac1H\left(\frac12-\theta\right).
\nonumber
\end{align}
The total probability of leaving \(x_i\) is \(1/H\), independently of the
action.  The same action \(a_i^*\) is special at \(x_i\) at every stage;
consequently, \(P^{\tar}\) is genuinely time-homogeneous.

The deterministic target-optimal policy takes action \(1\) at \(s_0\), takes
\(a_i^*\) at \(x_i\), and takes action \(1\) at the absorbing states.  If a
policy takes a non-special action at \(x_i\) at stage \(h\), its probability of
moving to \(g\), rather than \(b\), is smaller by \(\theta/H\).  Since
\(V_{h+1}^*(g)=H-h\) and \(V_{h+1}^*(b)=0\), the corresponding optimal
action gap is
\begin{align}
Q_h^*(x_i,a_i^*)-Q_h^*(x_i,a)
=\frac{\theta(H-h)}{H},
\qquad a\neq a_i^*.
\nonumber
\end{align}


\par\medskip\noindent\emph{Source MDP and behavior policy.}
We define the source parameter by $\theta_{\src}:=(\theta-H\nu_{\max})_+$. The source MDP equals the target MDP at every state--action pair except
\((x_i,a_i^*)\).  At this pair, we consider
\begin{align}
P^{\src}(g\mid x_i,a_i^*)
=\frac1H\left(\frac12+\theta_{\src}\right),
\qquad
P^{\src}(b\mid x_i,a_i^*)
=\frac1H\left(\frac12-\theta_{\src}\right).
\nonumber
\end{align}

Therefore,
\begin{align}
\TV\!\left(
P^{\tar}(\cdot\mid x_i,a_i^*),
P^{\src}(\cdot\mid x_i,a_i^*)
\right)
=\frac{|\theta-\theta_{\src}|}{H}
\leq \nu_{\max}.
\nonumber
\end{align}
The total-variation distance is zero at every other state--action pair.

Let the source behavior policy take the gateway action with probability $\pi^b(1\mid s_0)=\frac{M}{C_\nu^*}$.
With the remaining probability it takes any action that moves to \(z\).
At every \(x_i\), it is uniform over the first \(M\) actions: $\pi^b(a\mid x_i)=\frac1M\1\{a\in[M]\}$.
At \(g,b,z\), it takes action \(1\).  The factor \(M/C_\nu^*\) at the
gateway is what produces an effective source sample size
\(N^{\src}/C_\nu^*\), rather than \(N^{\src}/M\).

We now verify single-policy concentrability.  At stage \(1\), the occupancy
ratio of the optimal gateway action is \(C_\nu^*/M\leq C_\nu^*\).
For \(h\geq2\), the probability of being at \(x_i\), conditional on having
passed the gateway, is $\frac1{S_0}\left(1-\frac1H\right)^{h-2}$. It is the same in the source and target MDPs because the self-loop probability
does not depend on the action or on the parameter.  Hence, at the optimal
informative state--action pair, we have
\begin{align}
\frac{d_h^{\pi^*,\tar}(x_i,a_i^*)}
     {d_h^{\pi^b,\src}(x_i,a_i^*)}
=
\frac{1}
     {(M/C_\nu^*)(1/M)}
=C_\nu^*.
\nonumber
\end{align}

At pairs not used by \(\pi^*\), the numerator is zero.  At \(g\), the
state-occupancy ratio is at most
\begin{align}
\frac{C_\nu^*}{M}
\frac{1/2+\theta}{1/2+\theta_{\src}/M}
\leq
\frac{3C_\nu^*}{2M}
\leq C_\nu^*,
\nonumber
\end{align}
where \(M\geq2\) and \(\theta\leq1/4\).  At \(b\), it is at most
\begin{align}
\frac{C_\nu^*}{M}
\frac{1/2-\theta}{1/2-\theta_{\src}/M}
\leq
\frac{C_\nu^*}{M}
\leq C_\nu^*.
\nonumber
\end{align}

The target-optimal occupancy of \(z\) is zero.  This proves
\(\sup_{h,s,a}d_h^{\pi^*,\tar}(s,a)/
d_h^{\pi^b,\src}(s,a)\leq C_\nu^*\), and equality is attained at each
\((x_i,a_i^*)\).

Below we put the uniform prior on
\((a_1^*,\ldots,a_{S_0}^*)\in[M]^{S_0}\).
For a fixed \(i\), compare an alternative in which \(a\in[M]\) is special at
\(x_i\) with a reference experiment in which every action at \(x_i\) has zero
parameter.  The special actions at the other informative states are left
unchanged.  Let \(N_i^K(a)\) and \(N_i^{\src}(a)\) be the numbers of observed
target and source transitions, respectively, from \((x_i,a)\).  These counts
include all stages, which is valid because the transition kernel is the same
at every stage.

For \(x\in[0,1/4]\), the transition vectors with zero and nonzero parameter are
\begin{align}
p_0
=\left(1-\frac1H,\frac1{2H},\frac1{2H}\right),
\qquad
p_x
=\left(1-\frac1H,\frac1H\left(\frac12+x\right),
\frac1H\left(\frac12-x\right)\right).
\nonumber
\end{align}

Their one-observation divergence satisfies
\begin{align}
\KL(p_0\|p_x)
=\frac1{2H}\log\frac1{1-4x^2}
\leq\frac{4x^2}{H}.
\nonumber
\end{align}

The chain rule for relative entropy remains valid under adaptive online data
collection.  If \(\Prb_{i,0}\) denotes the law of all source data, target
observations, and algorithmic randomness in the reference experiment, and
\(\Prb_{i,a}\) denotes the corresponding law when \(a\) is special at \(x_i\),
then
\begin{align}
\KL(\Prb_{i,0}\|\Prb_{i,a})
\leq
\frac{4\theta^2}{H}\E_{i,0}[N_i^K(a)]
+
\frac{4\theta_{\src}^2}{H}\E_{i,0}[N_i^{\src}(a)].
\nonumber
\end{align}

In one target episode, the expected number of visits to a fixed \(x_i\) is at
most \(H/S_0\).  This bound holds for an adaptive learner as well: it may
choose not to pass the gateway, but conditional on passing, the gateway selects
\(x_i\) uniformly.  Consequently, we have $\frac1M\sum_{a=1}^M\E_{i,0}[N_i^K(a)]
\leq\frac{KH}{S_0M}$. For a source episode, the behavior policy passes the gateway with probability
\(M/C_\nu^*\) and then selects each of the first \(M\) actions with
probability \(1/M\).  Hence, for every \(a\in[M]\), we have $\E_{i,0}[N_i^{\src}(a)]
\leq
\frac{N^{\src}H}{S_0C_\nu^*}$. Averaging the preceding KL bound over \(a\in[M]\) therefore gives
\begin{align}
\frac1M\sum_{a=1}^M\KL(\Prb_{i,0}\|\Prb_{i,a})
\leq
4\left\{
\frac{K\theta^2}{S_0M}
+
\frac{N^{\src}\theta_{\src}^2}{S_0C_\nu^*}
\right\}.
\nonumber
\end{align}

Fix an output stage \(h\).  Under the reference experiment, the average
probability that the output policy fails to select a uniformly chosen action
\(a\in[M]\) is at least \(1-1/M\geq1/2\).  Pinsker's inequality and Jensen's
inequality thus yield
\begin{align}
\frac1M\sum_{a=1}^M
\E_{i,a}\!\left[1-\widehat\pi_h(a\mid x_i)\right]
&\geq
\frac12
-
\frac1M\sum_{a=1}^M
\sqrt{\frac12\KL(\Prb_{i,0}\|\Prb_{i,a})}
\geq
\frac12
-
\sqrt{
2\left\{
\frac{K\theta^2}{S_0M}
+
\frac{N^{\src}\theta_{\src}^2}{S_0C_\nu^*}
\right\}}.
\nonumber
\end{align}

It follows that, whenever
\begin{align}
\frac{K\theta^2}{S_0M}
+
\frac{N^{\src}(\theta-H\nu_{\max})_+^2}
     {S_0C_\nu^*}
\leq \frac1{32},
\nonumber
\end{align}
the average probability of selecting a non-special action at every
\((h,x_i)\) is at least \(1/4\).  


We then assume that the output policy takes action \(1\) at \(s_0\): replacing
its action at \(s_0\) by this known gateway action can only increase its target
value, so a lower bound for the improved policy also applies to the original
one.  Under every output policy so modified,
\begin{align}
d_h^{\widehat\pi,\tar}(x_i)
=
\frac1{S_0}\left(1-\frac1H\right)^{h-2},
\qquad h\geq2,
\nonumber
\end{align}
since the total exit probability from \(x_i\) is action-independent.

The performance-difference identity and the action-gap calculation give the
exact expression
\begin{align}
V_1^{*,\tar}(s_0)-V_1^{\widehat\pi,\tar}(s_0)
=
\frac1{S_0}
\sum_{i=1}^{S_0}\sum_{h=2}^{H}
\left(1-\frac1H\right)^{h-2}
\frac{\theta(H-h)}{H}
\left[1-\widehat\pi_h(a_i^*\mid x_i)\right].
\nonumber
\end{align}

For \(2\leq h\leq\lfloor H/2\rfloor\),
\((1-1/H)^{h-2}\geq1/2\) and \((H-h)/H\geq1/2\).  Averaging the previous
display over the uniform prior and using the testing-error lower bound yields
\begin{align}
\E\!\left[
V_1^{*,\tar}(s_0)-V_1^{\widehat\pi,\tar}(s_0)
\right]
&\geq
\frac{\theta}{4S_0}
\sum_{i=1}^{S_0}
\sum_{h=2}^{\lfloor H/2\rfloor}
\E\!\left[1-\widehat\pi_h(a_i^*\mid x_i)\right]\geq c_0H\theta.
\nonumber
\end{align}

Since the average over the hard family is at least \(c_0H\theta\), at least
one member of the family has expected suboptimality at least this large.

We then take a sufficiently small universal constant \(c_1>0\) and set
\begin{align}
\theta
=
c_1\min\left\{
1,\,
\sqrt{\frac{S_0M}{K}},\,
H\nu_{\max}
+\sqrt{\frac{S_0}{K/M+N^{\src}/C_\nu^*}}
\right\}.
\nonumber
\end{align}

This choice ensures \(\theta\leq1/4\) and
\begin{align}
\frac{K\theta^2}{S_0M}\leq c_1^2.
\nonumber
\end{align}

Moreover, because \(c_1\leq1\),
\begin{align}
(\theta-H\nu_{\max})_+
\leq
c_1\sqrt{\frac{S_0}{K/M+N^{\src}/C_\nu^*}},
\nonumber
\end{align}

and hence
\begin{align}
\frac{N^{\src}(\theta-H\nu_{\max})_+^2}
     {S_0C_\nu^*}
\leq
c_1^2
\frac{N^{\src}/C_\nu^*}
     {K/M+N^{\src}/C_\nu^*}
\leq c_1^2.
\nonumber
\end{align}

Choosing \(c_1\) small enough makes the required information condition hold.
We have therefore proved
\begin{align}
\E\!\left[
V_1^{*,\tar}(s_0)-V_1^{\widehat\pi,\tar}(s_0)
\right]
\geq
cH\min\left\{
1,\,
\sqrt{\frac{S_0M}{K}},\,
H\nu_{\max}
+\sqrt{\frac{S_0}{K/M+N^{\src}/C_\nu^*}}
\right\}.
\nonumber
\end{align}
Finally, \(S_0=S-4\geq S/2\) for \(S\geq8\), so replacing \(S_0\) by \(S\)
only changes the universal constant.  

\end{proof}

\section{Proof for Instance-Dependent Sub-optimality Gap Lower Bound}\label{sec:bpi_dependent_lower_proof}

In this subsection, we follow the derivations from~\citet{wang2022gap} to provide the  proof for instance-dependent sub-optimality gap lower bound.

\begin{proof}
Let $S\geq 8$, $A\geq2$, $H\geq5$, $C_\nu^*\geq2$, and define $M:=\min\{\lfloor C_\nu^*\rfloor,A\}$.

\par\medskip\noindent\emph{The hard time-homogeneous MDP family.}
The state space contains the fixed initial state $s_1$, the $S-4$ informative
states $x_1,\ldots,x_{S-4},=$, and three absorbing states $g,b,z$.  The action space is $\mathcal{A}=[A]$.  We use action
$1$ as a gateway action at $s_1$.  The transition from $s_1$ under action $1$
is
\begin{align}
P(x_i\mid s_1,1)
=
\frac{8\epsilon}
{(S-4)(H-2)\Delta_{\min}},
\qquad
i\in[S-4],
\nonumber
\end{align}
and the remaining probability is assigned to $z$.  Every other action at
$s_1$ transitions to $z$ with probability one.  The assumed range of
$\epsilon$ guarantees that these probabilities are valid.

Each target MDP is indexed by
\begin{align}
(a_1^*,\ldots,a_{S-4}^*)
\in[M]^{S-4}.
\nonumber
\end{align}

At an informative state $x_i$, the target transition under its special action
$a_i^*$ is
\begin{align}
P^{\tar}(g\mid x_i,a_i^*)
=
\frac12+\Delta_{\min},
\qquad
P^{\tar}(b\mid x_i,a_i^*)
=
\frac12-\Delta_{\min}.
\nonumber
\end{align}
Every non-special action at $x_i$ transitions to $g$ and $b$ with probability
$1/2$ each.  The states $g,b,z$ are absorbing.

The known deterministic reward is defined as
\begin{align}
r(g,a)=1,
\qquad
r(s_1,1)=\Delta_{\min},
\qquad
r(s,a)=0
\quad\text{for every other }(s,a).
\nonumber
\end{align}

 The
known reward $\Delta_{\min}$ at $(s_1,1)$ merely makes the gateway action
unambiguously optimal and reveals no information about
$(a_1^*,\ldots,a_{S-4}^*)$.

If an informative state $x_i$ is visited at stage $h$, choosing $a_i^*$
instead of a non-special action increases the probability of entering $g$ by
$\Delta_{\min}$.  Since entering $g$ produces one unit of reward at each of
the remaining $H-h$ stages,
\begin{align}
Q_h^*(x_i,a_i^*)
-
Q_h^*(x_i,a)
=
(H-h)\Delta_{\min},
\qquad
a\neq a_i^*.
\nonumber
\end{align}
The smallest positive value of the right-hand side is attained at $h=H-1$ and
equals $\Delta_{\min}$.  The gap between the gateway action and every other
action at $s_1$ is at least its known immediate reward
$\Delta_{\min}$.  All actions at the absorbing states have zero gap.
Consequently,
\begin{align}
\min_{\substack{h,s,a:\\
V_h^*(s)-Q_h^*(s,a)>0}}
\left\{
V_h^*(s)-Q_h^*(s,a)
\right\}
=
\Delta_{\min}.
\nonumber
\end{align}

\par\medskip\noindent\emph{The source MDP and the transition shift.}
The source and target MDPs are identical except at the special actions.  For
each $i\in[S-4]$, define the source transition directly by
\begin{align}
P^{\src}(g\mid x_i,a_i^*)
&=
\frac12+
\sqrt{
\Delta_{\min}
(\Delta_{\min}-\nu_{\max})_+
},
\nonumber\\
P^{\src}(b\mid x_i,a_i^*)
&=
\frac12-
\sqrt{
\Delta_{\min}
(\Delta_{\min}-\nu_{\max})_+
}.
\nonumber
\end{align}

If $\nu_{\max}<\Delta_{\min}$, the total-variation distance at a special
action is
\begin{align}
\Delta_{\min}
-
\sqrt{
\Delta_{\min}
(\Delta_{\min}-\nu_{\max})
}
&=
\Delta_{\min}
\left[
1-
\sqrt{
1-\frac{\nu_{\max}}{\Delta_{\min}}
}
\right]\leq
\nu_{\max},
\nonumber
\end{align}
where the inequality follows from
$\sqrt{1-u}\geq1-u$ for $u\in[0,1]$.  If
$\nu_{\max}\geq\Delta_{\min}$, the distance is
$\Delta_{\min}\leq\nu_{\max}$.  Hence every source-target pair satisfies the
required shift constraint.  Moreover, the squared source parameter is exactly
\begin{align}
\Delta_{\min}
(\Delta_{\min}-\nu_{\max})_+
=
\Delta_{\min}^2
\left(
1-\frac{\nu_{\max}}{\Delta_{\min}}
\right)_+.
\label{eq:source-information-discount}
\end{align}

Equation~\eqref{eq:source-information-discount} is why a source observation
receives a linear bias discount rather than placing the bias inside the
leading gap denominator.

\par\medskip\noindent\emph{The source behavior policy and exact concentrability.}
At the initial state, the behavior policy selects the gateway action with
probability $\pi^b(1\mid s_1)
=
\frac{M}{C_\nu^*}$. Its remaining probability is assigned arbitrarily to the other actions, all
of which enter $z$.  At every informative state, the behavior policy is
uniform on the $M$ candidate actions: $\pi^b(a\mid x_i) = \frac1M,  a\in[M]$. It takes an arbitrary fixed action at the absorbing states.  


The target-optimal policy takes action $1$ at $s_1$ and action $a_i^*$ at
$x_i$.  At stage two,
\begin{align}
d_2^{\pi^*,P^{\tar}}(x_i,a_i^*)
&=
\frac{8\epsilon}
{(S-4)(H-2)\Delta_{\min}},
\nonumber\\
d_2^{\pi^b,P^{\src}}(x_i,a_i^*)
&=
\frac{M}{C_\nu^*}
\frac{8\epsilon}
{(S-4)(H-2)\Delta_{\min}}
\frac1M
=
\frac{1}{C_\nu^*}
\frac{8\epsilon}
{(S-4)(H-2)\Delta_{\min}}.
\nonumber
\end{align}
Therefore,
\begin{align}
\frac{
d_2^{\pi^*,P^{\tar}}(x_i,a_i^*)
}{
d_2^{\pi^b,P^{\src}}(x_i,a_i^*)
}
=
C_\nu^*.
\label{eq:exact-concentrability}
\end{align}
At $s_1$, the occupancy ratio of the gateway action is
$C_\nu^*/M\leq C_\nu^*$.  At $z$, the ratio is no larger than one.  At $g$,
the ratio between the target-optimal occupancy and the source-behavior
occupancy is at most
\begin{align}
\frac{C_\nu^*}{M}
\frac{
1/2+\Delta_{\min}
}{
1/2+
\sqrt{
\Delta_{\min}(\Delta_{\min}-\nu_{\max})_+
}/M
}
\leq
\frac{3C_\nu^*}{2M}
\leq
\frac{3C_\nu^*}{4},
\nonumber
\end{align}
because $M\geq2$ and $\Delta_{\min}\leq1/4$.  The corresponding ratio at
$b$ is at most $2C_\nu^*/3$.  Every other action selected with zero
probability by the target-optimal policy has zero numerator.  Combining these
bounds with \eqref{eq:exact-concentrability} shows that the single-policy
concentrability coefficient is exactly $C_\nu^*$.
\par\medskip\noindent\emph{Change of measure.}
Fix $x_i$ and the special actions at all other informative states.  Let
$\Prb_{i,0}$ be the law obtained by setting every action at $x_i$ to transition
to $g$ and $b$ with probability $1/2$ in both MDPs, and let $\Prb_{i,a}$ be the
corresponding law when $a$ is special at $x_i$.  Without loss of generality,
the output takes the known optimal gateway action at $s_1$.  Let
$N_i^{K}(a)$ and $N_i^{\src}(a)$ be the target and source counts of $(x_i,a)$.
Under $\Prb_{i,0}$,
\begin{align}
\frac1M
\sum_{a=1}^M
\E_{i,0}\!\left[N_i^{K}(a)\right]
&\leq
\frac{8K\epsilon}
{(S-4)(H-2)M\Delta_{\min}},
\nonumber\\
\E_{i,0}\!\left[N_i^{\src}(a)\right]
&=
\frac{8N^{\src}\epsilon}
{(S-4)(H-2)C_\nu^*\Delta_{\min}},
\qquad a\in[M],
\nonumber
\end{align}
where the first inequality holds because a target episode reaches $x_i$ with
probability at most
$8\epsilon/((S-4)(H-2)\Delta_{\min})$, independently of the algorithm's
past observations.  The second follows from the behavior policy just defined.

For $u\in[0,1/4]$, we have
\begin{align}
\KL\!\left(
\Ber(1/2)\,\middle\|\,\Ber(1/2+u)
\right)
=
-\frac12\log(1-4u^2)
\leq
4u^2.
\nonumber
\end{align}


The chain rule for KL divergence gives
\begin{align}
\KL\!\left(
\Prb_{i,0}\,\middle\|\,\Prb_{i,a}
\right)
\leq
4\Delta_{\min}^2
\E_{i,0}\!\left[N_i^{K}(a)\right]
&+
4\Delta_{\min}^2
\left(
1-\frac{\nu_{\max}}{\Delta_{\min}}
\right)_+
\E_{i,0}\!\left[N_i^{\src}(a)\right].
\nonumber
\end{align}

Let $\widehat a_i:=\widehat\pi_2(x_i)$. Averaging over the $M$ possible special actions gives
\begin{align}
\frac1M
\sum_{a=1}^M
\Prb_{i,0}(\widehat a_i\neq a)
\geq
1-\frac1M.
\nonumber
\end{align}

Based on Pinsker's inequality, we have 
\begin{align}
\frac1M
\sum_{a=1}^M
\Prb_{i,a}(\widehat a_i\neq a)
&\geq
1-\frac1M
-
\sqrt{
\frac{1}{2M}
\sum_{a=1}^M
\KL\!\left(
\Prb_{i,0}\,\middle\|\,\Prb_{i,a}
\right)
}
\nonumber\\
&\geq
\frac12
-
\sqrt{
\frac{
16\epsilon\Delta_{\min}
}{
(S-4)(H-2)
}
\left[
\frac{K}{M}
+
\frac{N^{\src}}{C_\nu^*}
\left(
1-\frac{\nu_{\max}}{\Delta_{\min}}
\right)_+
\right]
}.
\label{eq:per-state-testing-error}
\end{align}


The right-hand side of \eqref{eq:per-state-testing-error} is at least $1/4$
whenever its square-root term is at most $1/4$.  After squaring, multiplying
by $M$, and moving the source contribution to the right-hand side, this
requirement is exactly
\begin{align}
K
<
\left[
\frac{(S-4)(H-2)M}
{256\,\epsilon\Delta_{\min}}
-
\frac{M N^{\src}}{C_\nu^*}
\left(
1-\frac{\nu_{\max}}{\Delta_{\min}}
\right)_+
\right]_+.
\label{eq:direct-K-threshold}
\end{align}

The inequality ensures that, when the expression inside the positive
part is nonpositive, no nonnegative $K$ is incorrectly included.  Under
\eqref{eq:direct-K-threshold}, the expected number of informative states at
which the output policy selects a non-special action is therefore at least $\frac{S-4}{4}$.

Under the target-optimal policy, each informative state is reached at stage
two with probability $\frac{8\epsilon}
{(S-4)(H-2)\Delta_{\min}}$. An incorrect action at that state loses $(H-2)\Delta_{\min}$ in conditional value.
Consequently, for every realized output policy that takes the gateway action,
\begin{align}
V_1^*(s_1)-V_1^{\widehat\pi}(s_1)
=
\frac{8\epsilon}{S-4}
\sum_{i=1}^{S-4}
\mathbf 1\{\widehat a_i\neq a_i^*\}.
\label{eq:value-error-identity}
\end{align}
Averaging \eqref{eq:value-error-identity} over the uniform prior, the data, and
the algorithmic randomness gives
\begin{align}
\E\!\left[
V_1^*(s_1)-V_1^{\widehat\pi}(s_1)
\right]
\geq
2\epsilon.
\nonumber
\end{align}

The value loss in \eqref{eq:value-error-identity} is at most $8\epsilon$.
Splitting according to whether the value loss is smaller than $\epsilon$
therefore gives
\begin{align}
2\epsilon
&\leq
\E\!\left[
V_1^*(s_1)-V_1^{\widehat\pi}(s_1)
\right]\leq
\epsilon
+
7\epsilon
\Prb\!\left(
V_1^*(s_1)-V_1^{\widehat\pi}(s_1)
\geq\epsilon
\right).
\nonumber
\end{align}
Consequently,
\begin{align}
\Prb\!\left(
V_1^*(s_1)-V_1^{\widehat\pi}(s_1)
\geq\epsilon
\right)
\geq\frac17.
\nonumber
\end{align}
Since this bound holds after averaging over the finite hard family, at least
one member of the family satisfies it.  Combining this conclusion directly
with \eqref{eq:direct-K-threshold}, we have proved that  when $C_\nu^*\geq A$, if 
\begin{align}
K
<
\left[
\frac{(S-4)(H-2)A}
{256\,\epsilon\Delta_{\min}}
-
\frac{A N^{\src}}{C_\nu^*}
\left(
1-\frac{\nu_{\max}}{\Delta_{\min}}
\right)_+
\right]_+,
\nonumber
\end{align}
we have
\begin{align}
\Prb\!\left(
V_1^*(s_1)-V_1^{\widehat\pi}(s_1)
\geq\epsilon
\right)
\geq\frac17. \nonumber
\end{align}

\end{proof}

\vskip 0.2in
\bibliography{Reference}

\end{document}